\documentclass[final,3p,times]{elsarticle}
\usepackage{amsmath, amsfonts, amssymb, amsthm, mathtools, bm}
\usepackage{thmtools}
\usepackage{thm-restate}

\theoremstyle{plain}
\newtheorem{theorem}{Theorem}[section]

\newtheorem{definition}[theorem]{Definition}

\newtheorem{assumption}[theorem]{Assumption}
\usepackage{graphicx}       % Include graphics
\usepackage{booktabs}       % Professional-quality tables
\usepackage{siunitx}

\usepackage{longtable}
\usepackage{multirow} 
\usepackage{caption}       % For caption customization
\usepackage{array}         % For better column definitions
\usepackage{siunitx}       % For decimal points
\usepackage{subcaption}
\usepackage{pgfplots}
\usepackage{tikz}
\usetikzlibrary{arrows.meta,positioning,fit}
\usepgfplotslibrary{groupplots}
\pgfplotsset{compat=1.18}

\usepackage{enumitem}       % Customizable lists

\usepackage[ruled,vlined]{algorithm2e} % Algorithm environment

\usepackage{float}          % Improved float handling

\usepackage{hyperref}       % Hyperlinks

\newcommand{\matr}[1]{\bm{#1}}        % Bold Matrices
\newcommand{\tr}{\operatorname{tr}}

\journal{Computers and Electrical Engineering}

\begin{document}

\begin{frontmatter}

%% Title, authors, and addresses

\title{Data Collaboration Analysis with Orthonormal Basis Selection and Alignment}

\tnotetext[titlelabel]{Earlier versions of this article have been circulated under the titles "Data Collaboration Analysis Over Matrix Manifolds" and "Data Collaboration Analysis with Orthogonal Basis Alignment."}

\author[inst1]{Keiyu Nosaka} %% Author name
\ead{s2430118@u.tsukuba.ac.jp}

\author[inst1]{Yamato Suetake}
\ead{s2620455@u.tsukuba.ac.jp}

\author[inst2]{Yuichi Takano}
\ead{ytakano@sk.tsukuba.ac.jp}

\author[inst2]{Akiko Yoshise \corref{cor1}} %% Author name
\ead{yoshise@sk.tsukuba.ac.jp}

\cortext[cor1]{Corresponding author}

\affiliation[inst1]{organization={Graduate School of Science and Technology, University of Tsukuba},%Department and Organization
            addressline={1-1-1}, 
            city={Tennodai, Tsukuba},
            postcode={305-8573}, 
            state={Ibaraki},
            country={Japan}}

% \affiliation[inst2]{organization={School of Science and Engineering, University of Tsukuba},%Department and Organization
%             addressline={1-1-1}, 
%             city={Tennodai, Tsukuba},
%             postcode={305-8573}, 
%             state={Ibaraki},
%             country={Japan}}

\affiliation[inst2]{organization={Institute of Systems and Information Engineering, and the Center for Artificial Intelligence Research, Tsukuba Institute for Advanced Research (TIAR),  University of Tsukuba},%Department and Organization
            addressline={1-1-1}, 
            city={Tennodai, Tsukuba},
            postcode={305-8573}, 
            state={Ibaraki},
            country={Japan}}

% Abstract
\begin{abstract}

Data Collaboration (DC) enables multiple parties to jointly train a model by sharing only linear projections of their private datasets. The core challenge in DC is to align the bases of these projections without revealing each party's secret basis. While existing theory suggests that any target basis spanning the common subspace should suffice, in practice, the choice of basis can substantially affect both accuracy and numerical stability. We introduce \textbf{Orthonormal Data Collaboration (ODC)}, which enforces orthonormal secret and target bases, thereby reducing alignment to the classical Orthogonal Procrustes problem, which admits a closed-form solution. We prove that the resulting change-of-basis matrices achieve \emph{orthogonal concordance}, aligning all parties' representations up to a shared orthogonal transform and rendering downstream performance invariant to the target basis. Computationally, ODC reduces the alignment complexity from
\(
O\!\left(\min\{a(c\ell)^2,\; a^2 c \ell\}\right)
\)
to
\(
O(a c \ell^2)
\),
and empirical evaluations show up to \(100\times\) speed-ups with equal or better accuracy across benchmarks. ODC preserves DC's one-round communication pattern and privacy assumptions, providing a simple and efficient drop-in improvement to existing DC pipelines.

\end{abstract}

%% Keywords

\end{frontmatter}

%% Add \usepackage{lineno} before \begin{document} and uncomment the following line to enable line numbers


\section{Introduction}
\label{section:introduction}

The effectiveness of machine learning (ML) algorithms depends strongly on the quality, diversity, and comprehensiveness of the training datasets. High-quality datasets enhance predictive performance and improve the generalization capabilities of ML models across diverse real-world applications. To overcome the biases and limitations inherent to datasets from a single source, data aggregation across multiple sources has become a common practice. Nonetheless, this practice introduces significant ethical and privacy concerns, particularly regarding unauthorized access and disclosure of sensitive user information. Recent literature highlights a growing awareness and reports an increasing number of privacy breaches linked to large-scale personal data collection and analysis~\cite{DataBreaching}.

To address these privacy concerns, Privacy-Preserving Machine Learning (PPML) has emerged as a crucial approach, enabling the secure and effective use of sensitive data—such as medical records, financial transactions, and geolocation data—without compromising individual privacy. Among various PPML methodologies, \textit{Federated Learning} (FL)~\cite{FL1} has gained prominence for enabling collaborative model training across decentralized data sources while safeguarding data privacy by limiting direct data exchange.

While FL offers promising solutions, it faces notable challenges in \textit{cross-silo} settings, primarily because it relies on iterative communication among participating entities during training. This communication overhead constitutes a significant obstacle, particularly in privacy-sensitive sectors such as healthcare and finance, where institutions often operate under stringent regulatory constraints and within isolated network infrastructures. Furthermore, FL inherently lacks formal privacy guarantees, necessitating the use of supplementary privacy-enhancing mechanisms to achieve explicit privacy assurances.

A widely adopted strategy to mitigate these shortcomings is \emph{Differential Privacy} (DP)~\cite{DP,NbaFL}, which protects individual records by adding carefully calibrated noise to the learning process. Although DP provides strong formal privacy guarantees, it necessarily trades off model utility for privacy~\cite{PPML}, and the resulting degradation in accuracy can substantially limit its practicality in real-world deployments, especially when data are already scarce or biased because they originate from a small number of institutions.

To overcome the limitations of conventional PPML methods, \textit{Data Collaboration} (DC) has emerged as a promising alternative~\cite{DCframework1, DCframework2}. Unlike FL, DC leverages the central aggregation of \textit{secure intermediate representations} computed locally from raw data. Specifically, each participating entity independently transforms its dataset using a basis matrix that it has privately selected. Subsequently, a central aggregator aligns these transformed datasets within a common representation space by constructing a shared target basis—\emph{without knowledge of the secret bases}. This approach facilitates collaborative model training while ensuring robust privacy preservation, thus eliminating the requirement for iterative communication among semi-honest parties~\cite{DCprivacy}.

Recent advancements in DC have expanded its applicability to scenarios involving more serious adversarial threats. Notably, methods have been proposed to counter re-identification attacks targeting intermediate representations~\cite{NRI-DC}, ensuring that transformed datasets remain unlinkable to their sources. Additionally, integrating differential privacy mechanisms has been explored to further mitigate the risks associated with malicious collusion among participants~\cite{DC-DP}. Moreover, hybrid approaches such as FedDCL~\cite{FedDCL} demonstrate how DC's advantages can be effectively combined with FL paradigms, resulting in scalable, privacy-preserving collaborative ML solutions.

Although DC has empirically demonstrated significant potential in balancing privacy and utility without iterative communication, its theoretical foundations remain underdeveloped. Existing theoretical analyses typically assume that \textit{any} target basis spanning the same subspace as the secret bases is sufficient. However, recent empirical studies indicate that selecting the target basis substantially influences the performance of downstream models~\cite{KawakamiDC, publication1}. Specifically, choosing target bases that disproportionately emphasize certain directions in the feature space can degrade model accuracy and utility. These findings highlight a clear discrepancy between current theoretical guarantees and observed empirical behavior, emphasizing the need for improved basis selection and alignment strategies within the DC framework.

\subsection{Our Contributions}

We propose a novel framework, termed \textit{Orthonormal Data Collaboration (ODC)}, to bridge the gap between DC's theoretical foundations and empirical performance. The central innovation of ODC is the explicit enforcement of \textit{orthonormality} constraints on both the secret and target bases during basis selection and alignment. This design leverages common practices in DC, as conventional dimensionality-reduction methods, such as Principal Component Analysis (PCA) and Singular Value Decomposition (SVD), naturally produce orthonormal bases, thereby imposing minimal additional overhead.

The orthonormality constraint leads to two significant theoretical advantages:

\begin{enumerate}

    \item \textbf{Alignment Efficiency:} The basis alignment simplifies precisely to the classical \textit{Orthogonal Procrustes Problem}~\cite{procrustessolution}, for which a closed-form analytical solution is available. This simplification significantly reduces the computational complexity relative to existing DC approaches.

    \item \textbf{Orthogonal Concordance:} All orthonormal target bases spanning the same subspace as the secret bases result in identical downstream model performance. Thus, the specific choice of an orthonormal target basis becomes inconsequential, effectively resolving previous instabilities in DC.

\end{enumerate}

Empirical evaluations confirm these theoretical results by: (a) demonstrating that ODC achieves substantially faster alignment compared to state-of-the-art DC methods, validating its theoretical complexity advantages; (b) highlighting the practical benefits of orthogonal concordance in stabilizing and enhancing model accuracy; and (c) assessing the robustness of ODC under realistic conditions where theoretical assumptions may not strictly hold.

Notably, the ODC framework extends traditional DC by adding only a single practical assumption—the orthonormality of secret bases. Empirical results demonstrate that relaxing this assumption notably degrades performance, underscoring its necessity. Because orthonormal bases already emerge naturally in existing DC workflows (e.g.,~\cite{NRI-DC}, DP integration~\cite{DC-DP}, and FL integration~\cite{FedDCL}), ODC can seamlessly integrate into current pipelines. Moreover, our empirical comparisons position ODC within the broader context of PPML, highlighting its advantages over mainstream techniques such as DP-based perturbations and FL.

\subsection{Notations and Organization}\label{sec:notation-organization}

\paragraph{Notations}
For a positive integer $c$, we write $[c] := \{1,2,\dots,c\}$.
We denote by $\mathbb{R}^{p\times q}$ the set of real $p\times q$ matrices.
For a matrix $\matr{M}$, $\matr{M}^\top$ denotes the transpose, $\matr{M}^\dagger$ the Moore--Penrose
pseudoinverse~\cite{pseudoinverse} (e.g., $M^\dagger = (M^\top M)^{-1}M^\top$ when $M$ has full column rank), $\|\matr{M}\|_F$ the Frobenius norm, and $\mathrm{tr}(\matr{M})$ the matrix trace.
We write
\begin{equation}
\mathcal{GL}(\ell):=\{\matr{G}\in\mathbb{R}^{\ell\times\ell}:\mathrm{rank}(\matr{G})=\ell\},
\qquad
\mathcal{O}(\ell):=\{\matr{O}\in\mathbb{R}^{\ell\times\ell}:\matr{O}^\top \matr{O}=\matr{O}\matr{O}^\top=\matr{I}_\ell\}
\end{equation}
for the general linear and orthogonal groups, respectively.
For a matrix $\matr{M}$, its (thin) SVD is written as $\matr{M}=\matr{U}\matr{\Sigma}\matr{V}^\top$.

In the DC setting, each user $i\in[c]$ holds a private dataset $\matr{X}_i\in\mathbb{R}^{n_i\times m}$ (with labels $\matr{L}_i$)
and a common anchor dataset $\matr{A}\in\mathbb{R}^{a\times m}$.
User $i$ selects a secret basis $\matr{F}_i\in\mathbb{R}^{m\times \ell}$ and releases only the intermediate representations
\begin{equation}
\tilde{\matr{X}}_i := \matr{X}_i \matr{F}_i\in\mathbb{R}^{n_i\times \ell},
\qquad
\matr{A}_i := \matr{A}\matr{F}_i\in\mathbb{R}^{a\times \ell}.
\end{equation}
The analyst constructs change-of-basis matrices $\matr{G}_i\in\mathbb{R}^{\ell\times\ell}$ and forms aligned data
\begin{equation}
\hat{\matr{X}}_i := \tilde{\matr{X}}_i \matr{G}_i.
\end{equation}
Contemporary DC methods typically allow $\matr{G}_i\in\mathcal{GL}(\ell)$, whereas ODC enforces $\matr{G}_i\in\mathcal{O}(\ell)$.

\paragraph{Organization}
\S~\ref{section: preliminaries} reviews the DC protocol as well as its privacy and communication properties.
\S~\ref{section: related works} discusses related work on Procrustes-based alignment and PPML.
\S~\ref{section: Basis Alignment} summarizes existing DC basis-alignment methods and their limitations.
\S~\ref{section: ODC} presents ODC, derives its closed-form Procrustes alignment, and proves orthogonal concordance.
\S~\ref{sec:efficiency-eval}--\S~\ref{sec:anchor-construction} report empirical studies on time efficiency, robustness under relaxed assumptions, concordance,
and anchor construction. \S~\ref{sec:conclusion} concludes the paper.

\paragraph{Illustrative Summary}
Fig.~\ref{fig: DCoverview} summarizes the ODC workflow. Each user privately selects an orthonormal basis $\matr{F}_i$ and releases only the intermediate representations $\tilde{\matr{X}}_i=\matr{X}_i\matr{F}_i$ and $\matr{A}_i=\matr{A}\matr{F}_i$. Using the anchor representations $\matr{A}_i (i\in[c])$, the analyst computes orthogonal change-of-basis matrices $\matr{G}_i$ that align all users in a common latent space without access to the secret bases. In ODC, these matrices are obtained from the SVD $\matr{A}_i^\top \matr{A}_1 \matr{O}=\matr{U}_i\matr{\Sigma}_i\matr{V}_i^\top$, which yields $\matr{G}_i=\matr{U}_i\matr{V}_i^\top$ for any $\matr{O}\in\mathcal{O}(\ell)$. The aligned representations $\tilde{\matr{X}}_i\matr{G}_i$ are then aggregated for downstream learning.

\begin{figure*}[t]
    \centering
    \includegraphics[width=\textwidth]{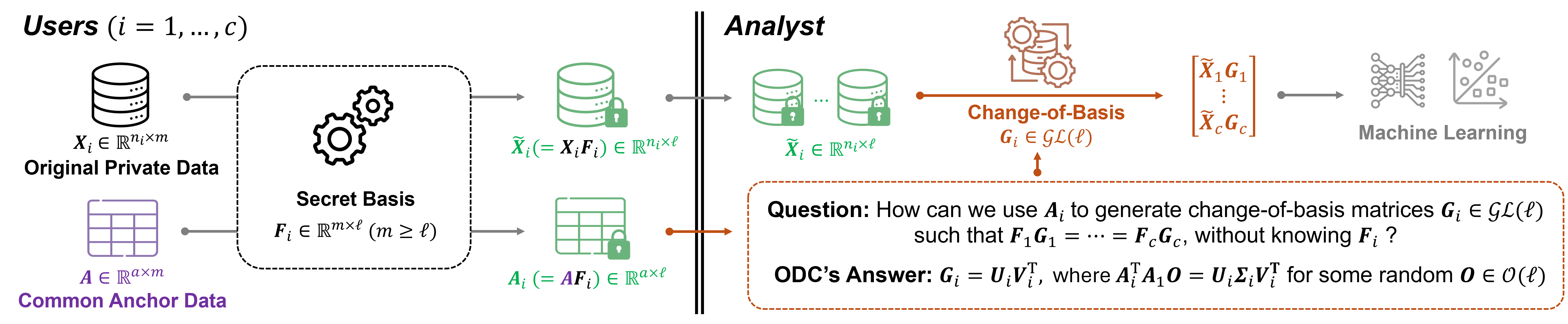}
    \caption{Schematic of the ODC workflow. Users release intermediate
    representations of private data and a common anchor, and the analyst
    computes orthogonal alignment matrices from the anchor representations
    to obtain a shared latent space.}
    \label{fig: DCoverview}
\end{figure*}

\section{Preliminaries}

\label{section: preliminaries}

In this section, we present the necessary preliminaries on DC analysis. Specifically, we begin with the DC protocol, then examine its privacy-preserving mechanisms and communication overhead.

\subsection{The Data Collaboration Algorithm}
\label{subsection: DCalgorithm}

We consider a general DC framework for supervised machine learning \cite{DCframework1, DC-DP}. Let \( \matr{X} \in \mathbb{R}^{n \times m} \) represent a dataset containing \( n \) training samples, each characterized by \( m \) features, and let \( \matr{L} \in \mathbb{R}^{n \times \ell} \) denote the corresponding label set with \( \ell \) labels. For privacy-preserving analysis across multiple entities, we assume the dataset is horizontally partitioned among \( c \) distinct entities, expressed as:
\begin{equation}
\matr{X} = \begin{bmatrix}
\matr{X}_1 \\
\vdots \\
\matr{X}_c
\end{bmatrix}, \quad
\matr{L} = \begin{bmatrix}
\matr{L}_1 \\
\vdots \\ 
\matr{L}_c
\end{bmatrix},
\end{equation}
where each entity \( i \) possesses a subset of the data \( \matr{X}_i \in \mathbb{R}^{n_i \times m} \) and corresponding labels \( \matr{L}_i \in \mathbb{R}^{n_i \times \ell} \). The total number of samples satisfies \( n = \sum_{i \in [c]} n_i \), where \([c] := \{1, 2, \ldots, c\}\). Additionally, each entity holds a test dataset \( \matr{Y}_i \in \mathbb{R}^{s_i \times m} \), for which the goal is to predict the corresponding labels \( \matr{L}_{\matr{Y}_i} \in \mathbb{R}^{s_i \times \ell} \). The DC framework can also be extended to handle more complex scenarios, such as partially shared features \cite{DCfeatures} or data partitioned both horizontally and vertically \cite{DCframework2}.

The framework defines two primary roles: the \textit{user} and the \textit{analyst}. Users possess their private datasets \( \matr{X}_i \) and corresponding labels \( \matr{L}_i \), and aim to improve local model performance by leveraging insights from other users' data without revealing their own. The analyst's role is to facilitate this collaborative process by providing the necessary resources and infrastructure for the machine learning workflow.

At the outset, users collaboratively construct a shared anchor dataset $\matr{A} \in \mathbb{R}^{a \times m}$, where $a > m$. Following the standard DC literature~\cite{DCframework1, DCnoniid}, we assume that $\matr{A}$ is either (i) synthetically generated dummy data or (ii) sampled from an external public database (for example, unlabeled compounds from PubChem for chemical fingerprints~\cite{DCnoniid}). Here, “public” refers only to the \emph{data source}: the concrete subset of records and feature vectors that form $\matr{A}$ is sampled locally by the users and is \emph{not} disclosed to the analyst. The analyst receives only the transformed anchor representations $\matr{A}_i = \matr{A} \matr{F}_i$, not the raw anchor $\matr{A}$ itself, so $\matr{A}$ remains unknown to the analyst unless a malicious colluding user explicitly reveals it.

Each user independently selects an \(m\)-dimensional basis with size \(\ell\), denoted by \(\matr{F}_i \in \mathbb{R}^{m \times \ell}\) (\(m \geq \ell\)), to linearly transform their private dataset \(\matr{X}_i\) and the anchor dataset \(\matr{A}\) into secure intermediate representations. A common basis selection method employs truncated SVD with a random orthogonal mapping \cite{FedDCL, DCnoniid}:

\begin{equation}
\label{eq:typicalbasisselection}
    \matr{F}_i = \matr{V}_i \matr{E}_i,
\end{equation}
where, \(\matr{E}_i \in \mathcal{O}(\ell) \coloneqq \{\matr{O} \in \mathbb{R}^{\ell \times \ell} : \matr{O}^\top \matr{O} = \matr{O}\matr{O}^\top = \matr{I}\}\) and \(\matr{V}_i \in \mathbb{R}^{m \times \ell}\) denotes the top \( \ell \) right singular vectors of \( \matr{X}_i \). Notably, this typical method inherently produces orthonormal bases, i.e., \( \matr{F}_i^\top \matr{F}_i = \matr{I} \).

Once the secret bases \(\matr{F}_i\) are chosen for each user, the secure intermediate representations of the private dataset \(\matr{X}_i\) and the anchor dataset \( \matr{A} \)  are computed as follows.

\begin{equation}
\label{eq: IR naive}
    \tilde{\matr{X}_i} = \matr{X}_i\matr{F}_i, \quad \matr{A}_i = \matr{A} \matr{F}_i.
\end{equation}

Each user shares \(\tilde{\matr{X}_i}\), \(\matr{L}_i\), and \(\matr{A}_{i}\) with the analyst, whose task is to construct a collaborative ML model based on all \(\tilde{\matr{X}_i}\) and \(\matr{L}_i\). However, directly concatenating \(\tilde{\matr{X}_i}\) and building a model from it is futile, as the bases were selected privately and are generally different. Within the DC framework, the analyst aims to align the secret bases using change-of-basis matrices \(\matr{G}_i \in \mathbb{R}^{\ell \times \ell}\), and constructs \(\hat{\matr{X}}\) as follows:
\begin{equation}
\label{eq: xhat}
\hat{\matr{X}} = 
\begin{bmatrix}
    \tilde{\matr{X}_1} \matr{G}_1 \\
    \vdots \\
    \tilde{\matr{X}_c} \matr{G}_c
\end{bmatrix}.
\end{equation}
After successfully creating the change-of-basis matrices from the aggregated \(\matr{A}_i\), as detailed in \S~\ref{section: related works} and ~\ref{section: ODC}, the analyst utilizes \( \hat{\matr{X}} \) and \( \matr{L} \) to construct a supervised classification model \( h \):
\begin{equation}
\label{eq: Lapprox}
\matr{L} \approx h(\hat{\matr{X}}).
\end{equation}

This model \( h \) can simply be distributed to the users along with \( \matr{G}_i \) to predict the labels \( \matr{L}_{\matr{Y}_i} \) of the test dataset \( \matr{Y}_i \):
\begin{equation}
\label{eq: inference}
\matr{L}_{\matr{Y}_i} = h(\matr{Y}_i \matr{F}_i \matr{G}_i),
\end{equation}
or employed in other DC-based applications \cite{NRI-DC, DCapp1, DCapp2, DCapp3, DCapp4, DCapp5, DCrecommender} for enhanced privacy or utility. 

Algorithm~\ref{alg: DCframework} summarizes the DC protocol. Cross-entity communication is limited to three phases: sharing the common anchor among users, sending local intermediate representations to the analyst, and returning the alignment matrices together with the trained model. The key algorithmic question is how to construct the change-of-basis matrices $\matr{G}_i$ from the anchor representations without access to the secret bases $\matr{F}_i$; this question is addressed in \S~\ref{section: Basis Alignment} and \S~\ref{section: ODC}.

\begin{algorithm}[t]
\caption{Data Collaboration Protocol \emph{(adapted from Algorithm 1 in \cite{NRI-DC})}}
\label{alg: DCframework}
\SetAlgoLined
\DontPrintSemicolon
\LinesNumbered
\SetNlSty{textbf}{}{:}
\SetNlSkip{0.5em}

\SetKwInOut{Input}{Input}
\SetKwInOut{Output}{Output}

\Input{$\matr{X}_i \in \mathbb{R}^{n_i \times m}$, $\matr{L}_i \in \mathbb{R}^{n_i \times \ell}$, and $\matr{Y}_i \in \mathbb{R}^{s_i \times m}$ for each user $i \in [c]$}
\Output{$\matr{L}_{\matr{Y}_i} \in \mathbb{R}^{s_i \times \ell}$ for each user $i \in [c]$}

\tcc{Phase 1 (user side): each user converts raw data into privacy-preserving intermediate representations using a private basis.}
Generate the shared anchor dataset $\matr{A} \in \mathbb{R}^{a \times m}$ and distribute it to all users\;
\ForEach{user $i \in [c]$}{
    Select a secret basis $\matr{F}_i \in \mathbb{R}^{m \times \ell}$ \tcp*{e.g., basis selection in Eq.~\eqref{eq:typicalbasisselection}}
    Compute $\tilde{\matr{X}}_i = \matr{X}_i \matr{F}_i$ and $\matr{A}_i = \matr{A}\matr{F}_i$ \tcp*{intermediate representations; cf. Eq.~\eqref{eq: IR naive}}
    Send $(\tilde{\matr{X}}_i,\matr{A}_i,\matr{L}_i)$ to the analyst\;
}

\tcc{Phase 2 (analyst side): the analyst uses only the released anchor representations to align the user-specific latent spaces, without accessing raw data or secret bases.}
Compute the change-of-basis matrices $\matr{G}_i$ from $\matr{A}_i$ for all $ i\in [c]$ \tcp*{see \S~\ref{section: Basis Alignment} and \S~\ref{section: ODC}}
\ForEach{user $i \in [c]$}{
    Form the aligned representation $\hat{\matr{X}}_i = \tilde{\matr{X}}_i \matr{G}_i$ \tcp*{aligned data; cf. Eq.~\eqref{eq: xhat}}
}
Aggregate $\hat{\matr{X}}_i,\matr{L}_i$ for all $i\in[c]$ to construct $\hat{\matr{X}}$ and $\matr{L}$ \tcp*{cf. Eq.~\eqref{eq: xhat}}
Train a collaborative model $h$ such that $\matr{L} \approx h(\hat{\matr{X}})$ \tcp*{cf. Eq.~\eqref{eq: Lapprox}}
Return $\matr{G}_i$ and $h$ to each user $i \in [c]$\;

\tcc{Phase 3 (user side): each user applies the returned alignment and trained model locally to its private test data.}
\ForEach{user $i \in [c]$}{
    Predict $\matr{L}_{\matr{Y}_i} = h(\matr{Y}_i \matr{F}_i \matr{G}_i)$ \tcp*{cf. Eq.~\eqref{eq: inference}}
}
\end{algorithm}

\subsection{Privacy Analysis}
\label{subsection: DCprivacy}

This section provides a brief review of the privacy guarantees and limitations inherent in the DC framework~\cite{DCprivacy}. Throughout this paper, we adopt the standard \emph{semi-honest} threat model for DC~\cite{DCprivacy, NRI-DC, DC-DP}: all parties follow the prescribed protocol but may attempt to infer additional information from the data they are authorized to observe.

\paragraph{Standard semi-honest model}
Each user \(i\) observes only its own raw dataset \((\matr{X}_i,\matr{L}_i)\), the shared anchor \(\matr{A}\), and the learned model \(h\). The analyst observes only intermediate representations \((\tilde{\matr{X}}_i,\matr{A}_i)\) and labels \(\matr{L}_i\); in particular, the analyst never sees the raw anchor \(\matr{A}\) or the secret bases \(\matr{F}_i\), and thus cannot directly invert user-side transformations. 

The shared anchor $\matr{A}\in\mathbb{R}^{a\times m}$ is \emph{common across users}, but its \emph{origin} (public vs.\ synthetic) and its \emph{visibility} (whether the analyst can access the exact $\matr{A}$ used by the users) are distinct design choices. In the default DC threat model, users share $\matr{A}$ among themselves (or share a PRNG seed that deterministically generates $\matr{A}$), while the analyst receives only the transformed anchors $\matr{A}_i := \matr{A}\matr{F}_i$ (Algorithm~\ref{alg: DCframework}, Step~5). When we say that an anchor is ``public'' we mean \emph{non-sensitive / not derived from any user's private records}; the analyst need not be assumed to possess the exact $\matr{A}$ instance unless stated otherwise. If the analyst \emph{does} know $\matr{A}$ (e.g., a truly public anchor dataset), the setting falls outside the semi-honest analyst model and is discussed under the collusion model.

Under the semi-honest model, the privacy guarantees and limitations of this work coincide with those of the underlying DC framework. In particular, as long as all participants remain semi-honest and non-colluding, the subsequent privacy results do not depend on the number of users $ c$; adding more semi-honest users does not, by itself, reveal additional information about any individual dataset.

The following two theorems are direct restatements, in our notation, of the privacy results established for DC in~\cite{DCprivacy}. They are included for completeness and are not claimed as new contributions.

\begin{theorem}
\label{theorem: privacy-semi-honest-users}
\textbf{Privacy Against Semi-Honest Users} \emph{(Adapted from Theorem~1 in \cite{DCprivacy})}.\\
Any semi-honest user \( i \) in the DC framework cannot infer the private dataset \( \matr{X}_j \) of any other user \( j \neq i \).
\end{theorem}

\begin{proof}
A semi-honest user \(i\) has access only to its own private dataset \((\matr{X}_i,\matr{L}_i)\), the shared anchor dataset \(\matr{A}\), and the collaboratively trained model \(h\). Information about another user \(j \neq i\) is available only through the learned model \(h\) and, implicitly, through the transformed representations \(\hat{\matr{X}}_j\) that contributed to training.

The anchor dataset \(\matr{A}\) is either publicly sourced or synthetically generated and does not contain any information derived from \(\matr{X}_j\). Moreover, the transformed data for user \(j\) is given by \(\hat{\matr{X}}_j = \matr{X}_j \matr{F}_j \matr{G}_j\), where both transformation matrices \(\matr{F}_j\) and \(\matr{G}_j\) are unknown to user \(i\). Without access to \(\matr{F}_j\) and \(\matr{G}_j\), user \(i\) cannot invert the transformation or recover meaningful information about \(\matr{X}_j\). Hence, a semi-honest user cannot infer the private dataset of any other user.
\end{proof}

\begin{theorem}
\label{theorem: privacy-semi-honest-analyst}
\textbf{Privacy Against a Semi-Honest Analyst} \emph{(Adapted from Theorem~2 in \cite{DCprivacy})}.\\
A semi-honest analyst in the DC framework cannot infer the private dataset \( \matr{X}_j \) of any user \( j \).
\end{theorem}

\begin{proof}
The semi-honest analyst has access only to the outputs of the user-side linear transformations, namely
\(\tilde{\matr{X}}_j = \matr{X}_j \matr{F}_j\) and
\(\matr{A}_j = \matr{A} \matr{F}_j\) for each user \(j\). However, both the transformation matrices \(\matr{F}_j\) and the anchor \(\matr{A}\) are unknown to the analyst. Thus, the analyst observes only pairs \((\tilde{\matr{X}}_j,\matr{A}_j)\) with an unknown common factor \(\matr{F}_j\), and lacks sufficient information to recover \(\matr{X}_j\) from \(\tilde{\matr{X}}_j\). Therefore, a semi-honest analyst cannot infer any user’s private dataset.
\end{proof}

Contemporary DC techniques typically employ orthonormal bases to transform the private data. Intuitively, enforcing orthonormality on the secret bases preserves geometric properties (e.g., distances and angles) but does not aid reconstruction in the absence of the bases themselves. Consequently, orthonormality does not weaken the privacy guarantees under the semi-honest model.

To visually support this intuition, Fig.~\ref{fig:visual_privacy} presents a simple illustrative example on CelebA~\cite{celeba}. We demonstrate that orthonormal projections significantly degrade visual recognizability, and the degree of obfuscation is comparable to that achieved by general (non-orthogonal) random projections. A more detailed visual privacy analysis appears later in \S~\ref{sec: assumptions-evaluation}.

\begin{figure}[htbp]
    \centering
    \begin{subfigure}[b]{0.9\textwidth}
        \centering
        \includegraphics[width=\textwidth]{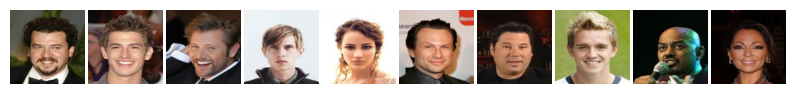}
        \caption{Original images}
        \label{fig:visual_privacy_a}
    \end{subfigure}
    
    \begin{subfigure}[b]{0.9\textwidth}
        \centering
        \includegraphics[width=\textwidth]{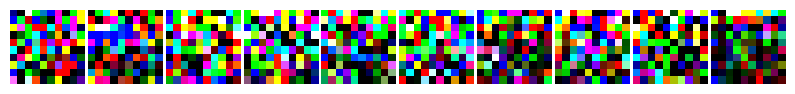}
        \caption{Orthogonally projected images ($\ell=300$)}
        \label{fig:visual_privacy_b}
    \end{subfigure}
    
    \begin{subfigure}[b]{0.9\textwidth}
        \centering
        \includegraphics[width=\textwidth]{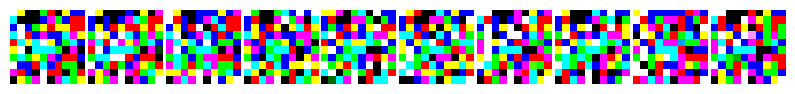}
        \caption{Randomly projected images ($\ell=300$)}
        \label{fig:visual_privacy_c}
    \end{subfigure}
    
    \caption{Visual privacy verification using CelebA~\cite{celeba}. Original images (panel (a)) compared to images after orthonormal projections (panel (b)) and random non-orthogonal projections (panel (c)). Both transformations strongly obfuscate the visual content, illustrating that the orthonormality assumption does not compromise visual privacy relative to general projections.}
    \label{fig:visual_privacy}
\end{figure}

\paragraph{Remarks on collusion model}
In practice, as the number of users increases, it becomes more likely that some participants deviate from the semi-honest assumption. A natural extension of the DC threat model therefore considers the possibility of \emph{malicious collusion} between participants. In particular, when a user \(i\) colludes with the analyst, the colluding parties jointly observe another user \(j\)’s transformed data, i.e., \(\matr{A}_j\) and \(\tilde{\matr{X}}_j\), as well as the raw anchor \(\matr{A}\). Since \(\matr{A}\) has full column rank, this enables reconstruction of the secret basis via the Moore–Penrose pseudoinverse~\cite{pseudoinverse}, \(\matr{F}_j = \matr{A}^\dagger \matr{A}_j\). As argued in~\cite{DCprivacy}, because \(\matr{F}_j \in \mathbb{R}^{m \times \ell}\) with \( m > \ell \), the matrix \(\matr{F}_j\) acts as a projection, and choosing a smaller latent dimension \(\ell\) makes it substantially harder to reconstruct the original data \(\matr{X}_j\) from \(\tilde{\matr{X}}_j = \matr{X}_j \matr{F}_j\), consistent with the guarantees of \(\varepsilon\)-DR privacy~\cite{DRprivacy}.

To mitigate such collusion-based risks and align with stronger privacy standards, differential-privacy-based mechanisms and procedures to eliminate identifiability have been integrated into DC, as demonstrated in~\cite{NRI-DC,DC-DP}. These techniques provide additional protection under threat models that go beyond semi-honest behavior; we refer interested readers to these works for a detailed treatment of stronger adversaries.

We emphasize that this study neither introduces new privacy vulnerabilities nor attempts to strengthen the existing privacy guarantees of DC. Our focus is instead on the \emph{concordance} properties of DC and on improving its computational efficiency, as discussed in \S~\ref{section: ODC}. The ODC methodology is fully compatible with enhanced DC variants that offer stronger privacy protections~\cite{NRI-DC,DC-DP}, but a comprehensive formal analysis and empirical evaluation of such combinations are left for future work.

\subsection{Communication Overhead}
\label{subsec:communication}

In this subsection, we quantify the communication overhead incurred by DC and compare it with standard FL. As illustrated in \textbf{Algorithm}~\ref{alg: DCframework}, DC requires at most three communication phases:
\begin{enumerate}
    \item A one-off distribution of a common anchor dataset to all users (Step~2).
    \item A single uplink from each user to the analyst, containing transformed representations (Step~5).
    \item A final downlink from the analyst to each user, comprising the trained model and the change-of-basis matrix (Step~15).
\end{enumerate}

Let \(q\) denote the number of bits per scalar element (e.g., \(q{=}32\) for FP32 and \(q{=}16\) for FP16), and let \(N\) denote the number of parameters in the downstream model \(h\). We write \(c\) for the number of users, \(n_i\) for the private sample count at user \(i\), \(\bar n := \frac{1}{c}\sum_{i=1}^c n_i\) for the average sample count, \(a\) for the anchor size, \(m\) for the input dimension, and \(\ell\) for the latent dimension.

For each user \(i\in[c]\), the \emph{uplink} payload consists of the transformed representations of both its private
dataset and the anchor dataset:
\begin{equation}
    B^{\mathrm{DC}\uparrow}_{i} = \frac{(n_i+a)\,\ell\,q}{8}
    \quad\text{[bytes]}.
    \label{eq:dc-uplink}
\end{equation}
The corresponding \emph{downlink} from the analyst to user \(i\) contains the trained model \(h\) and the (change-of-basis) alignment matrix \(\matr{G}_i \in \mathbb{R}^{\ell\times \ell}\):
\begin{equation}
    B^{\mathrm{DC}\downarrow}_{i} = \frac{(\ell^{2}+N)\,q}{8}
    \quad\text{[bytes]}.
    \label{eq:dc-downlink}
\end{equation}

Let
\begin{equation}
    B_A := \frac{a\,m\,q}{8}\quad\text{[bytes]}
    \label{eq:anchor-onecopy}
\end{equation}
denote the size of the raw anchor matrix \(\matr{A}\in\mathbb{R}^{a\times m}\) for one copy. The \emph{aggregate} cross-institution traffic required to make \(\matr{A}\) available can be written as
\begin{equation}
    B^{\mathrm{anchor}} = \gamma\,B_A = \gamma\,\frac{a\,m\,q}{8}
    \quad\text{[bytes]},
    \label{eq:anchor}
\end{equation}
where \(\gamma\ge 0\) is a topology-dependent \emph{replication factor} (how many cross-silo copies traverse links).

In cross-silo settings, anchor distribution is typically unicast (so \(\gamma\approx c\)), unless a public/seeded anchor is
used (\(\gamma=0\)). Importantly, this anchor cost is \emph{one-off}.

Summing \eqref{eq:dc-uplink}, \eqref{eq:dc-downlink}, and \eqref{eq:anchor}, the aggregate DC traffic becomes
\begin{align}
    B^{\mathrm{DC}}
    &= \sum_{i=1}^{c}\Bigl(B^{\mathrm{DC}\uparrow}_{i}+B^{\mathrm{DC}\downarrow}_{i}\Bigr)+B^{\mathrm{anchor}}\notag\\[3pt]
    &= \frac{q}{8}\Bigl[c\bigl((\bar{n}+a)\,\ell+(\ell^{2}+N)\bigr)+\gamma\,a\,m\Bigr]
    \quad\text{[bytes]}.
    \label{eq:dc-total}
\end{align}
DC is a \emph{single-round} protocol: the cost \eqref{eq:dc-total} is incurred once (plus the one-off anchor
distribution if \(\gamma>0\)).

Let \(R\) denote the number of FL rounds and \(p\in(0,1]\) the per-round participation fraction (constant across rounds). Each selected participant uploads and downloads the full model once per round, thus transmitting \(2N\) scalars per round. The cumulative FL traffic is
\begin{equation}
    B^{\mathrm{FL}} = 2\,R\,p\,c \;\frac{N\,q}{8}
    \quad\text{[bytes]}.
    \label{eq:fl-total}
\end{equation}

Setting \(B^{\mathrm{DC}}=B^{\mathrm{FL}}\) and solving for \(R\) yields
\begin{equation}
    R^{*}
    = \frac{(\bar n+a)\,\ell}{2pN}
    + \frac{\ell^{2}+N}{2pN}
    + \frac{\gamma\,a\,m}{2p\,c\,N}.
    \label{eq:Rstar}
\end{equation}
Thus, DC is more communication-efficient whenever \(R \ge \lceil R^{*}\rceil\). Note that \(q\) cancels exactly in \eqref{eq:Rstar}; the dependence on \(c\) appears only through the anchor term and disappears in the common unicast case \(\gamma\approx c\) (e.g., cold start or coordinator distribution).

Equation~\eqref{eq:Rstar} implies that \(R^* \propto 1/p\) and decreases with model size \(N\). While \(R^*\) is independent of the bit-width \(q\), the \emph{absolute} traffic scales linearly with \(q\): 
\begin{equation}
B^{\mathrm{DC}}(q)=\frac{q}{q_0}B^{\mathrm{DC}}(q_0),
\qquad
B^{\mathrm{FL}}(q)=\frac{q}{q_0}B^{\mathrm{FL}}(q_0),
\end{equation}
where \(q_0\) is just the baseline quantization bit-width. 

For the numerical setup below, Fig.~\ref{fig:q_sweep} plots \(B^{\mathrm{DC}}\) and \(B^{\mathrm{FL}}\) as functions of \(q\). The heatmaps in Fig.~\ref{fig:Rstar_heatmaps} illustrate how \(R^*\) shifts with respect to \(\bar n\) and \(N\) (for several values of \(p\)), while Fig.~\ref{fig:Rstar_curves} presents one-dimensional slices, showing \(R^*\) versus \(N\) (for multiple \(p\)) and \(R^*\) versus \(p\) (for multiple \(N\)).

\begin{figure}[t]
    \centering
    \includegraphics[width=0.70\textwidth]{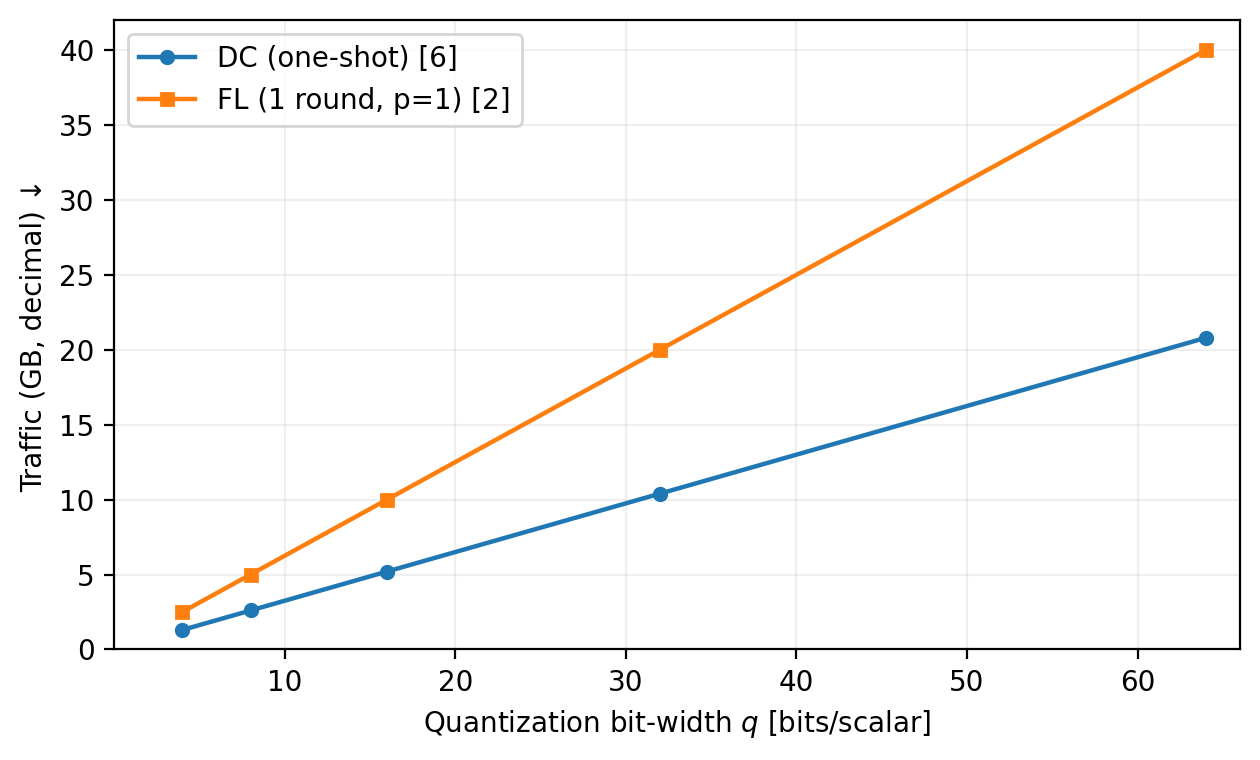}
    \caption{Absolute communication volume versus quantization bit-width \(q\) for the healthcare example (\(\gamma=c\)).}
    \label{fig:q_sweep}
\end{figure}

\begin{figure}[htbp]
    \centering

    \begin{subfigure}[b]{0.48\textwidth}
        \centering
        \includegraphics[width=\textwidth]{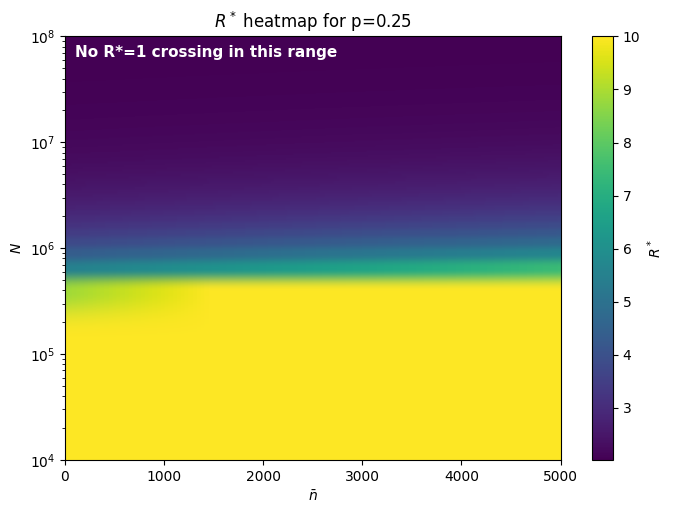}
        \caption{$p = 0.25$}
        \label{fig:p025}
    \end{subfigure}
    \hfill
    \begin{subfigure}[b]{0.48\textwidth}
        \centering
        \includegraphics[width=\textwidth]{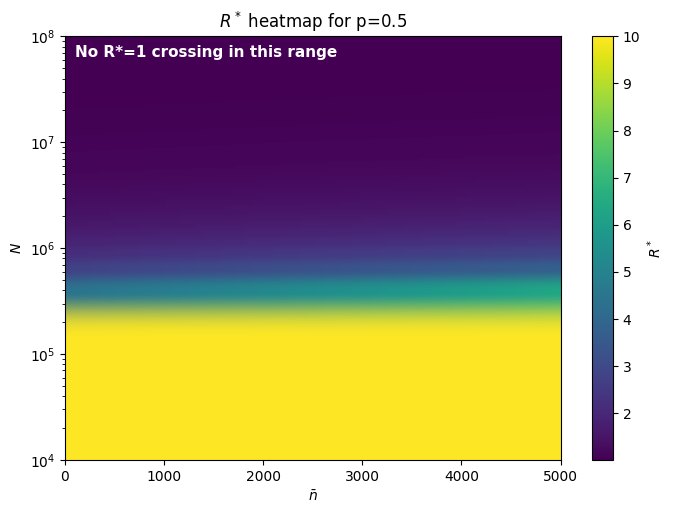}
        \caption{$p = 0.5$}
        \label{fig:p05}
    \end{subfigure}

    \begin{subfigure}[b]{0.48\textwidth}
        \centering
        \includegraphics[width=\textwidth]{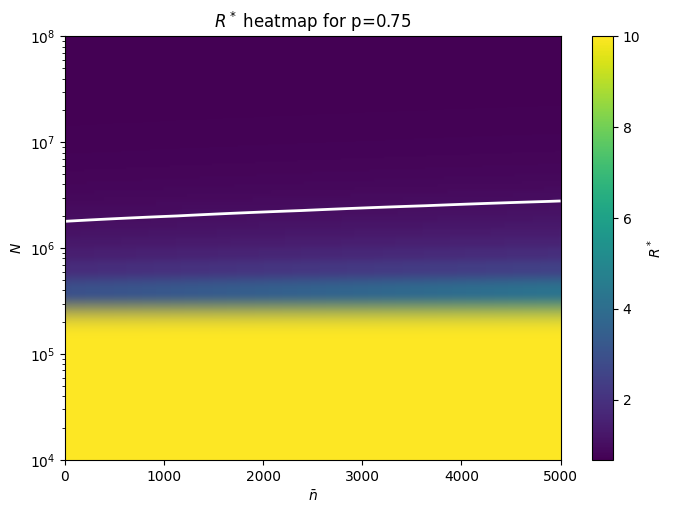}
        \caption{$p = 0.75$}
        \label{fig:p075}
    \end{subfigure}
    \hfill
    \begin{subfigure}[b]{0.48\textwidth}
        \centering
        \includegraphics[width=\textwidth]{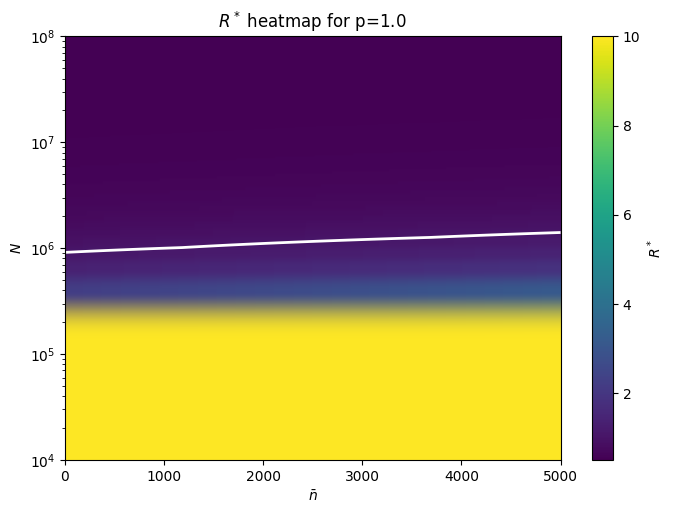}
        \caption{$p = 1.0$}
        \label{fig:p1}
    \end{subfigure}

    \caption{Heatmaps of the threshold number of FL rounds \(R^*\) for different participation rates \(p\). Brighter regions correspond to larger \(R^*\), i.e., more FL rounds are required for DC to be more communication-efficient. The white line indicates the break-point for \(R^* = 1\), i.e., regions above this line mean a single FL round costs more than DC.}
    \label{fig:Rstar_heatmaps}
\end{figure}

\begin{figure}[t]
    \centering
    \begin{subfigure}[b]{0.48\textwidth}
        \centering
        \includegraphics[width=\textwidth]{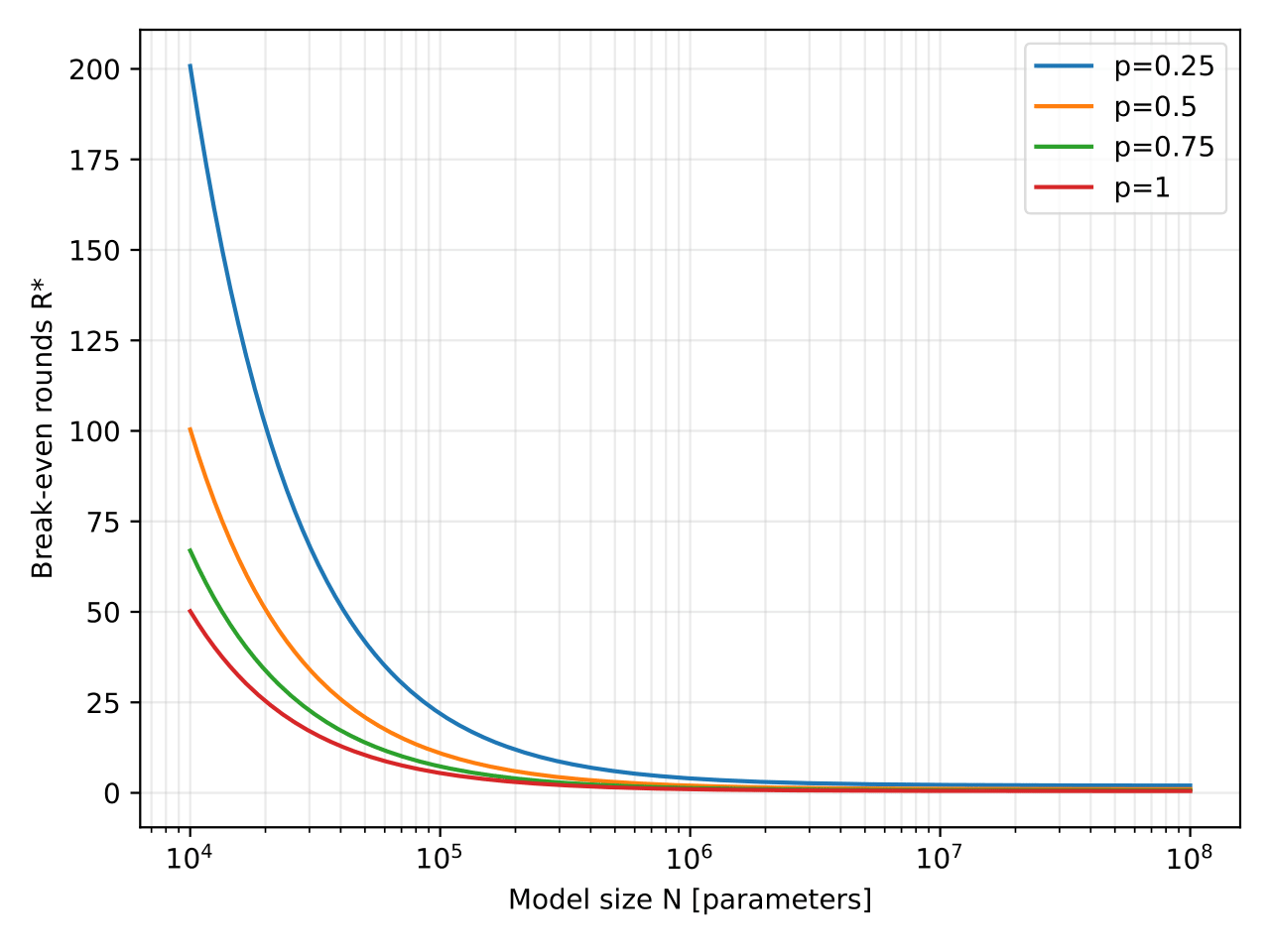}
        \caption{\(R^*\) as a function of model size \(N\) for \(p\in\{0.25,0.5,0.75,1.0\}\) (fixed \(\bar n\)).}
        \label{fig:Rstar_vs_N}
    \end{subfigure}
    \hfill
    \begin{subfigure}[b]{0.48\textwidth}
        \centering
        \includegraphics[width=\textwidth]{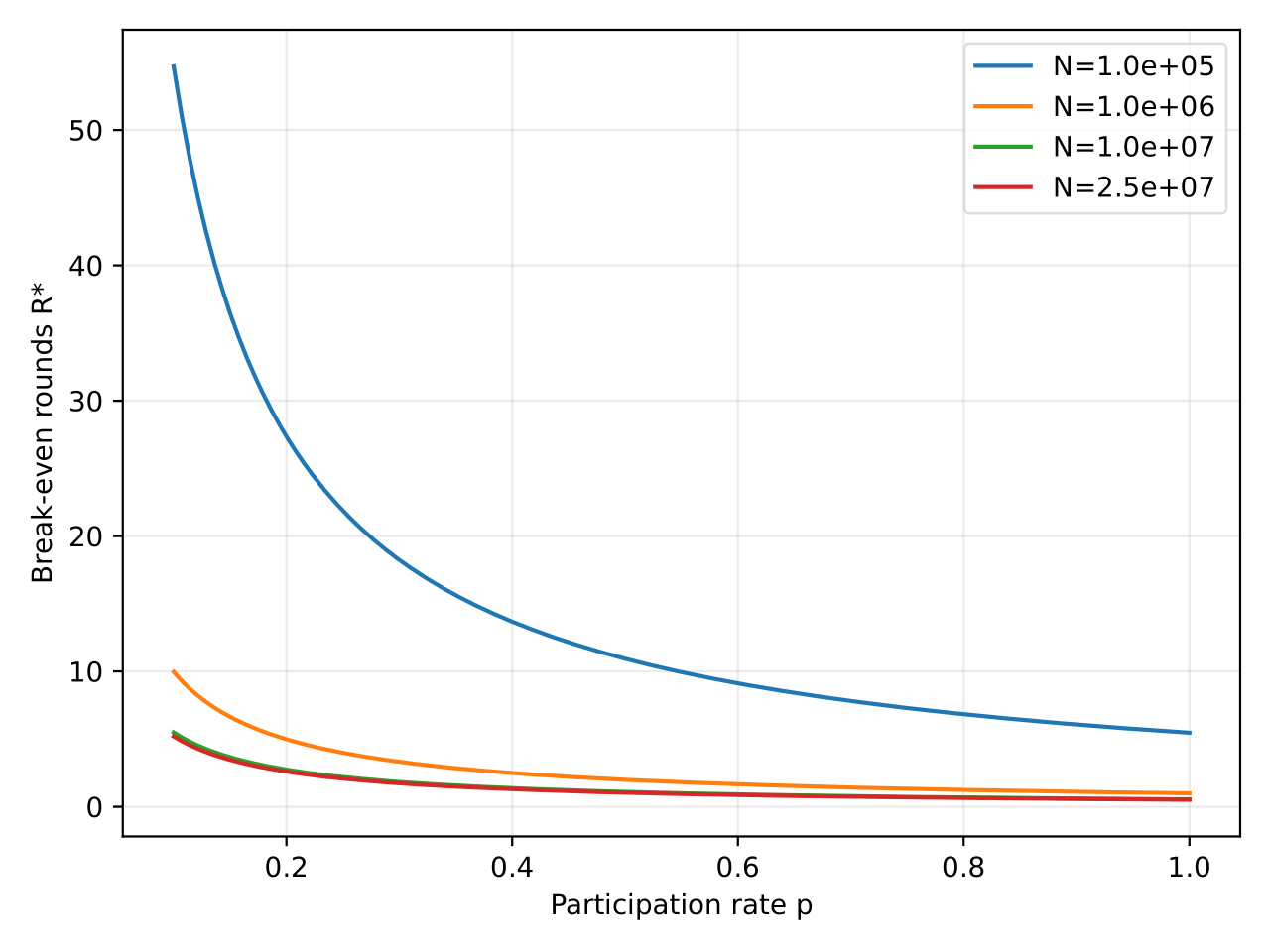}
        \caption{\(R^*\) as a function of participation \(p\) for several \(N\) (fixed \(\bar n\)).}
        \label{fig:Rstar_vs_p}
    \end{subfigure}
    \caption{Sensitivity curves for the break-even FL rounds \(R^*\) (example settings: \(a=10^{3}\), \(m=784\), \(\ell=100\), \(\gamma=c\)).}
    \label{fig:Rstar_curves}
\end{figure}

Let \(\beta\) denote the effective \emph{bottleneck} goodput (bits/s) and \(\tau\) the RTT (seconds). A simple engineering estimate for end-to-end transfer time (ignoring compute and assuming payload dominates) is
\begin{equation}
    T^{\mathrm{DC}} \approx \frac{8\,B^{\mathrm{DC}}}{\beta} + 3\tau,
    \qquad
    T^{\mathrm{FL}} \approx \frac{8\,B^{\mathrm{FL}}}{\beta} + 2R\tau,
    \label{eq:time-model}
\end{equation}
where the additive RTT terms reflect the three DC phases and the two-message-per-round structure of synchronous FL. If clients are the bottleneck rather than the coordinator, one can replace \(B^{\mathrm{DC}}\) (resp.\ \(B^{\mathrm{FL}}\)) by the maximum per-client transmitted bytes in the corresponding phase.

To avoid ambiguity, Table~\ref{tab:units} lists the decimal (network) and binary (storage) unit conventions. We report GB in decimal units (\(10^9\) bytes) when translating communication volume to network time.

\begin{table}[t]
\centering
\caption{Unit conversions. We report GB in decimal units (\(10^9\) bytes) when translating to network time.}
\label{tab:units}
\begin{tabular}{lcc}
\toprule
Unit & Decimal (SI) & Binary (IEC) \\
\midrule
KB / KiB & \(10^{3}\) B & \(2^{10}\) B \\
MB / MiB & \(10^{6}\) B & \(2^{20}\) B \\
GB / GiB & \(10^{9}\) B & \(2^{30}\) B \\
\bottomrule
\end{tabular}
\end{table}

Consider \(c=100\) hospitals, each with \(n_i=10^{3}\) samples; anchor size \(a=10^{3}\); input dimension \(m=784\); latent dimension \(\ell=100\); and a ResNet-50 with \(N\simeq 2.5\times 10^{7}\) parameters. Assuming 32-bit quantization (\(q=32\)) and a cold-start/unicast anchor distribution (\(\gamma=c\)),
\eqref{eq:dc-total} gives
\begin{equation}
    B^{\mathrm{DC}} \approx 1.04\times 10^{10}\ \text{bytes} \approx 10.4\ \text{GB},
\end{equation}
while one round of FL with full participation (\(p=1\)) incurs
\begin{equation}
    B^{\mathrm{FL}} \approx 2.0\times 10^{10}\ \text{bytes} \approx 20\ \text{GB}.
\end{equation}
Thus, DC achieves nearly a \(50\%\) reduction in communication even compared to a \emph{single} full-participation FL round. (If the anchor is public/seeded, \(\gamma=0\) and \(B^{\mathrm{DC}}\) decreases slightly by removing the one-off anchor term.)

Under partial participation, say \(p=0.1\), \eqref{eq:Rstar} yields \(R^{*}\approx 5.2\) (for \(\gamma=c\)): FL must run for more than about five rounds before its cumulative traffic exceeds that of DC. Since practical FL deployments often require tens to hundreds of rounds to converge, DC is expected to yield substantial communication savings in realistic cross-silo settings.

\section{Related Works}
\label{section: related works}

\subsection{Procrustes methods in multi-view alignment}

The orthogonal Procrustes problem (OPP)~\cite{procrustessolution} is a classical tool for aligning two matrices by an orthogonal transformation. In multi-view alignment, OPP and its multi-set generalizations underpin a large family of methods that seek a common latent representation shared across views. In functional neuroimaging, \emph{hyperalignment} aligns subject-specific response matrices to a common template by solving a multi-set orthogonal Procrustes problem, under the assumption that subject-specific representational spaces are related by orthogonal transformations~\cite{Haxby2011CommonModel}. Kernel Hyperalignment extends this framework by solving a regularized multi-set OPP in a reproducing kernel Hilbert space, thereby accommodating nonlinear similarities while retaining the orthogonality constraint on each subject-specific map~\cite{Lorbert2012KernelHyperalignment}. Closely related generalized orthogonal Procrustes problems (GOPP) seek a collection of orthogonal matrices that best align multiple point clouds, and recent work has analyzed semidefinite relaxations and generalized power methods with sharp recovery guarantees under signal-plus-noise models~\cite{GOPP}. 

In multi-view representation learning and clustering, Procrustes-based alignment is also widely used. Many methods learn view-specific embeddings $\matr{X}_i$ together with orthogonal transforms $\matr{R}_i$ and a consensus representation $\matr{Z}$ by solving objectives of the form
\begin{equation}
  \min_{\matr{Z},\matr{R}_i \in \mathcal{O}(\ell)} \sum_{i} \bigl\| \matr{X}_i \matr{R}_i - \matr{Z} \bigr\|_\mathrm{F}^2,
  \label{eq:mv-procrustes-generic}
\end{equation}
with variants that introduce adaptive view weights, Grassmannian constraints, or anchor-based parameterizations. For example, Multiview Clustering via Adaptively Weighted Procrustes (AWP) explicitly formulates clustering as a weighted Procrustes averaging problem over spectral embeddings from different views~\cite{Nie2018AWP}, while Multi-view Clustering with Adaptive Procrustes on Grassmann Manifold (MC-APGM) learns an indicator matrix that approximates multiple orthogonal spectral embeddings on the Grassmann manifold, again through an orthogonal Procrustes-type objective~\cite{Dong2022GrassmannProcrustes}. Similar multi-set Procrustes formulations appear in manifold alignment, where independently constructed embeddings of multiple graphs or feature views are brought into correspondence by an orthogonal map~\cite{Wang2008ManifoldProcrustes}.  

Beyond "classical" multi-view settings, OPP has become a standard primitive for aligning independently trained representation spaces. In natural language processing, Wasserstein--Procrustes objectives are used to align word or sentence embeddings across languages by jointly estimating an optimal transport plan and an orthogonal mapping between embedding spaces~\cite{Grave2019WassersteinProcrustes}. In knowledge-graph representation learning, highly efficient frameworks repeatedly apply closed-form orthogonal Procrustes updates to keep entity and relation embeddings in a shared space while reducing training time and memory footprint by orders of magnitude~\cite{Peng2021ProcrustesKGE}. In these works, orthogonal transformations are primarily valued for preserving inner products and norms, thereby keeping the geometric structure of each representation space intact while making different embeddings comparable.  

Conceptually, our ODC formulation shares the same mathematical backbone as these multi-set Procrustes methods: under our orthonormal-basis assumption on the secret bases in \S~\ref{section: ODC}, basis alignment in ODC can be written as
\begin{equation}
\min_{\matr{Z} \in \mathbb{R}^{a \times \ell},\,\matr{G}_i \in \mathcal{O}(\ell)} \quad \sum_{i=1}^{c} \|\matr{A}_i \matr{G}_i - \matr{Z}\|_\mathrm{F}^2.
\end{equation}
which is a multi-view OPP over the anchor representations $\matr{A}_i$. What fundamentally distinguishes ODC is the \emph{role} it plays and the way we exploit its non-uniqueness. Existing multi-view Procrustes approaches typically fix one view as a template, or implicitly choose a particular $Z$ (e.g., via the SVD of concatenated data) and then interpret the resulting orthogonal transforms as \emph{the} alignment; the fact that the solution set of~\eqref{eq:mv-procrustes-generic} is closed under right-multiplication by a common orthogonal matrix is treated as a benign identifiability issue and is not analyzed in connection with downstream learning objectives~\cite{Haxby2011CommonModel, Lorbert2012KernelHyperalignment, Nie2018AWP, Dong2022GrassmannProcrustes, Grave2019WassersteinProcrustes, Peng2021ProcrustesKGE}.  

In contrast, our notion of \emph{orthogonal concordance}~\ref{def:orth_concordance} elevates this indeterminacy to the main object of study. Assuming orthonormal secret bases $\matr{F}_i$, we show that there exists a common matrix $\matr{F}$ and orthogonal matrices $\matr{O}_i$ such that $\matr{F}_i = \matr{F} \matr{O}_i$, and that any Procrustes solution can be written in the form $\matr{G}_i^\star = \matr{O}_i^\top \matr{O}$ for an \emph{arbitrary} global orthogonal matrix $\matr{O}$. Consequently, all aligned representations $\matr{F}_i \matr{G}_i^\star$ collapse to the same product $\matr{F} \matr{O}$, and every choice of $\matr{O}$ in the Procrustes equivalence class yields \emph{identical} downstream predictions for distance-based models (and, empirically, near-identical performance more broadly). This invariance property — that the entire multi-set Procrustes solution set induces the same collaborative model — is what we call orthogonal concordance, and it is specific to the DC setting with orthonormal secret bases and anchor-only access to the data.  

Finally, most Procrustes-based multi-view methods assume a centralized optimizer that directly observes all raw views $\matr{X}_i$ or their feature embeddings and do not model privacy or communication constraints~\cite{Nie2018AWP, Dong2022GrassmannProcrustes, Grave2019WassersteinProcrustes, Peng2021ProcrustesKGE}. ODC instead operates in a privacy-preserving DC pipeline, where the analyst only sees intermediate anchor representations $\matr{A}_i$ and user-transformed data. The goal is not merely to reduce alignment error, but to guarantee that the entire equivalence class of Procrustes solutions yields the same collaborative model. This theoretical shift---from \emph{finding a good orthogonal alignment} to \emph{characterizing all alignments that are provably equivalent for learning}---is the key difference between classical Procrustes-based multi-view alignment and our concept of orthogonal concordance.

\subsection{Privacy-Preserving Machine Learning}
\label{subsection: PPML}

Privacy-preserving machine learning (PPML) broadly studies how to learn from sensitive data without directly centralizing raw records. Among the most widely deployed paradigms is \emph{federated learning} (FL), where clients repeatedly run local training and a server aggregates model updates over many communication rounds (e.g., FedAvg) \cite{FL1}. Because FL communicates gradients/weights, it does not, by itself, provide a formal privacy guarantee; consequently, FL is often paired with \emph{differential privacy} (DP) mechanisms that add calibrated noise to client updates to obtain an $(\varepsilon,\delta)$-DP guarantee \cite{DP, NbaFL}. DP provides worst-case, distribution-free protection—output distributions remain stable to the inclusion/exclusion of any single record—but can impose a non-trivial privacy--utility tradeoff, especially at strong privacy levels \cite{DP}.

\emph{Data collaboration} (DC) offers a complementary PPML design point: instead of iteratively sharing model updates, each party shares \emph{one-shot intermediate representations} obtained by applying a private (secret) transform to its data (and to a shared anchor), after which an analyst aligns representations into a common space for downstream learning \cite{FL1, DC-DP}. This yields a different theoretical separation from FL/DP:
\begin{itemize}
  \item \textbf{Optimization object:} FL performs distributed optimization of model parameters via iterative aggregation \cite{FL1}, whereas DC centralizes \emph{aligned representations} and can then train arbitrary downstream models in that shared space \cite{DCframework1}.
  \item \textbf{Communication pattern:} FL typically requires many bidirectional rounds until convergence; DC uses a single uplink of representations (plus a one-time anchor broadcast), which can be advantageous in cross-silo settings with large models \cite{FL1, DCframework1}.
  \item \textbf{Privacy semantics:} DP provides a formal indistinguishability guarantee via noise \cite{DP}, while DC primarily relies on secrecy of the local transform (often analyzed under semi-honest assumptions) rather than an $(\varepsilon,\delta)$ guarantee \cite{DCprivacy}.
\end{itemize}

Our proposed ODC remains within the DC family, retaining DC's one-shot protocol, while constraining alignment to orthogonal transformations, thereby preserving geometric structure and simplifying the analytical alignment process. Finally, DP and DC are not mutually exclusive: DP can be layered on top of released representations to strengthen protection under collusion, and DP-enhanced DC has been explicitly explored \cite{DC-DP}.

\section{Existing Basis Alignment}
\label{section: Basis Alignment}

Focusing on Step~10 of \textbf{Algorithm}~\ref{alg: DCframework}, we now address the central question of DC:
\begin{equation}
\textit{How can we construct the change-of-basis matrices } \matr{G}_i \textit{ without access to the secret bases } \matr{F}_i \textit{?}
\end{equation}
In this section, we review the two prevailing basis-alignment strategies employed in existing DC frameworks—Imakura-DC and Kawakami-DC—and analyze their concordance properties and computational efficiency, measured by FLOP count and peak memory usage.

\subsection{Imakura's Basis Alignment (Imakura-DC)}
\label{subsection: Imakura}

Here, we review the basis-alignment method proposed by Imakura \textit{et al.}~\cite{DCprivacy}. To begin with, we introduce \textbf{Assumption}~\ref{assumptions: DC}.

\begin{assumption}
    \label{assumptions: DC}
    We impose the following assumptions on the anchor dataset \( \matr{A} \in \mathbb{R}^{a \times m} \) and the secret bases \( \matr{F}_i \in \mathbb{R}^{m \times \ell} \; (m \geq \ell) \), for all \(i \in [c]\):
    \begin{enumerate}[label=\arabic*)]
        \item \( \mathrm{rank}(\matr{A}) = m \);
        \item There exists \( \matr{E}_i \in \mathcal{GL}(\ell) \coloneqq \{ \matr{R} \in \mathbb{R}^{\ell \times \ell} : \mathrm{rank}(\matr{R}) = \ell \}  \) such that \( \matr{F}_i = \matr{F}_1\matr{E}_i \).
    \end{enumerate}
\end{assumption}

Condition 1) is a common requirement in the DC literature. For instance, selecting \( \matr{A} \)  to be a uniformly random matrix (with appropriate dimensions) ensures \(\mathrm{rank}(\matr{A}) = m\), which is standard practice. Condition 2) requires all intermediate representations to span an identical \(\ell\)-dimensional subspace. 

The central task for the analyst is to construct suitable change-of-basis matrices \( \matr{G}_i \) that \emph{align} the secret bases \( \matr{F}_i \), even though these bases are never directly observed. Formally, the objective is to identify matrices \( \matr{G}_i \in \mathcal{GL}(\ell) \) for each \( i \in [c] \) such that
\begin{equation}
    \label{eq: basis equation}
    \matr{F}_1 \matr{G}_1 \;=\; \matr{F}_2 \matr{G}_2 \;=\; \dots \;=\; \matr{F}_c \matr{G}_c.
\end{equation}
Note that such $\matr{G}_i$ exist if and only if condition 2) of \textbf{Assumption}~\ref{assumptions: DC} is met.

Since the analyst has access to the intermediate representations of a common anchor dataset, defined as \(\matr{A}_i = \matr{A}\matr{F}_i \in \mathbb{R}^{a\times \ell}\), where \(\matr{A}\) remains consistent across all users, it follows that any set of matrices \(\matr{G}_i\) satisfying~\eqref{eq: basis equation} necessarily satisfy the following condition:

\begin{equation}
    \label{eq:anchor_equation}
    \matr{A}_1\matr{G}_1 = \matr{A}_2\matr{G}_2 = \dots = \matr{A}_c\matr{G}_c.
\end{equation}

Consequently, the matrices \(\matr{G}_i\) that satisfy equation~\eqref{eq:anchor_equation} are necessarily those that minimize the following optimization problem:

\begin{equation}
    \label{problem:non-compact}
    \min_{\matr{G}_i, \matr{G}_j \in \mathcal{GL}(\ell)}\quad
    \sum_{i=1}^{c}\sum_{j=1}^{c}
    \bigl\|
        \matr{A}_{i}\matr{G}_{i}-\matr{A}_{j}\matr{G}_{j}
    \bigr\|_{\mathrm{F}}^{2}.
\end{equation}

A crucial observation is that \( \mathcal{GL}(\ell) \) is inherently non-compact. This non-compactness poses significant difficulties, specifically that one can construct a sequence of invertible matrices whose singular values decrease, causing the sequence to converge to the zero matrix. Such convergence undermines the stability of the optimization and complicates the identification of meaningful solutions. This ill-posedness can, however, be alleviated by a judicious, \textit{a priori} selection of some target matrix \( \matr{Z} \in  \mathbb{R}^{a \times \ell} \). Specifically, define \( \matr{Z} = \matr{U}\matr{R} \), where \( \matr{U} \) comprises the top \(\ell\) left singular vectors of the concatenated matrix \(\begin{bmatrix}\matr{A}_1 & \cdots & \matr{A}_c\end{bmatrix}\), and set an arbitrary invertible matrix \(\matr{R} \in \mathcal{GL}(\ell)\). Problem~\eqref{problem:non-compact} reduces to:
\begin{equation}
    \label{eq:DC_Gi}
    \min_{\matr{G}_i \in \mathcal{GL}(\ell)} \|\matr{A}_i\matr{G}_i - \matr{Z}\|_\mathrm{F}^2.
\end{equation}
The optimization problem in~\eqref{eq:DC_Gi} admits closed-form analytical solutions given by \(\matr{G}_i^* = \matr{A}_i^\dagger \matr{Z}\) for each \(i \in [c]\). Imakura-DC's basis alignment procedure is summarized in \textbf{Algorithm}~\ref{alg: Imakura-DC}. 

\begin{algorithm}[t]
\caption{Imakura-DC's Basis Alignment Procedure~\cite{DCprivacy}}
\label{alg: Imakura-DC}
\SetAlgoLined
\DontPrintSemicolon
\LinesNumbered
\SetNlSty{textbf}{}{:}    
\SetNlSkip{0.5em}        

\SetKwInOut{Input}{Input}
\SetKwInOut{Output}{Output}

\Input{\( \matr{A}_i \in \mathbb{R}^{a \times \ell} \) for each user \( i \in [c] \)}
\Output{\(\matr{G}_i^* \in \mathcal{GL}(\ell) \) for each user \( i \in [c] \)}
\BlankLine
\Begin{
    Form the concatenated anchor $\big[\,\matr{A}_1~\cdots~\matr{A}_c\,\big]$ and compute its top $\ell$ left singular vectors $\matr{U}$\;
    Set the target $\matr{Z} = \matr{U}\matr{R}$ with any $\matr{R}\in\mathcal{GL}(\ell)$ (often $\matr{R} = \matr{I}$)\;
    Set $\matr{G}_i^* = \matr{A}_i^\dagger\,\matr{Z}$ for all $i$\;
}

\end{algorithm}

\subsubsection{Weak Concordance of Imakura-DC}
\label{subsubsection: Imakura concordance}

To theoretically assess the downstream performance of the resulting change-of-basis matrices $\matr{G}_i^*$, we analyze their sufficiency with respect to Eq.~\eqref{eq: basis equation}. We formalize \emph{weak concordance} as follows:

\begin{definition} \label{def:weak_concordance} (Weak Concordance)\\ 
The change-of-basis matrices \(\matr{G}_i \in \mathcal{GL}(\ell)\), for \( i \in [c] \), satisfy weak concordance if 
\begin{equation} 
\matr{F}_1 \matr{G}_1 = \matr{F}_2 \matr{G}_2 = \dots = \matr{F}_c \matr{G}_c. 
\end{equation} 
\end{definition} 

\begin{theorem}

    \label{theorem: Imakura concordance}
    (Weak Concordance of Imakura-DC (cf.~\cite{DCprivacy}))\\
    Suppose that we observe matrices \(\matr{A}_i = \matr{A}\matr{F}_i\), \( i \in [c] \), with \(\matr{A}\in \mathbb{R}^{a\times m}\) and \(\matr{F}_i\in\mathbb{R}^{m\times \ell}\). Under \textbf{Assumption}~\ref{assumptions: DC}, let \(\matr{U}\) denote the matrix formed by the top \(\ell\) left singular vectors of the concatenation $\big[\,\matr{A}_1~\cdots~\matr{A}_c\,\big]$.  
    Then, for each \(i \in [c]\), the solution \(\matr{G}_i^* = \matr{A}_i^\dagger \matr{Z}\) to the optimization problem:
    \begin{equation}
        \label{eq:DC_Gi_repeated}
        \min_{\matr{G}_i \in \mathbb{R}^{\ell\times \ell}} \|\matr{A}_i\matr{G}_i - \matr{Z}\|_\mathrm{F}^2,
    \end{equation}
    where \(\matr{Z}=\matr{U}\matr{R}\) for an arbitrary invertible matrix \(\matr{R}\in\mathcal{GL}({\ell})\), is weakly concordant~(\textbf{Definition}~\ref{def:weak_concordance}).
    
\end{theorem}

\begin{proof}

Since all \(\matr{A}_i\) share the same \(\ell\)-dimensional column space, the top \( \ell \) left singular vectors \( \matr{U} \) of $\big[\,\matr{A}_1~\cdots~\matr{A}_c\,\big]$ also lie in the same column space. Therefore, there exists some invertible matrix \( \matr{Q} \in \mathcal{GL}(\ell)\) such that:
\begin{equation}
\matr{Z} = \matr{A}_1 \matr{Q}.
\end{equation}
Then we can write:
\begin{equation}
\matr{Z} = \matr{A}_1 \matr{Q} = \matr{A} \matr{F}_1 \matr{E}_i (\matr{E}_i^{-1} \matr{Q}) = \matr{A}_i (\matr{E}_i^{-1} \matr{Q}).
\end{equation}
Hence, 
\begin{equation}
\matr{G}_i^* = \matr{A}_i^\dagger \matr{Z} = \matr{A}_i^\dagger \matr{A}_i (\matr{E}_i^{-1} \matr{Q}) = \matr{E}_i^{-1} \matr{Q},
\end{equation}
and therefore, we have
\begin{equation}
\matr{F}_1 \matr{G}_1^* = \cdots = \matr{F}_c \matr{G}_c^*,
\end{equation}
which completes the proof.

\end{proof}

\textbf{Theorem~\ref{theorem: Imakura concordance}} establishes that the invertible right factor \(\matr{R}\in\mathbb{R}^{\ell\times \ell}\) of the target matrix \(\matr{Z}=\matr{U}\matr{R}\) can be chosen arbitrarily while preserving weak concordance. This flexibility naturally prompts the question of whether such arbitrary choices could negatively impact downstream model performance. Unfortunately, the answer is affirmative. Empirical evidence examining this issue is presented in \S\ref{sec: concordance evaluation}. Intuitively, selecting a target matrix that disproportionately emphasizes certain directions within the feature space may adversely affect model accuracy and utility. Although recent studies empirically indicate improved performance when choosing \(\matr{R} = \matr{I}\)~\cite{DCframework1}, this particular choice remains heuristic and lacks rigorous theoretical justification, suggesting potential suboptimality.

Consequently, the practical utility of \textbf{Definition}~\ref{def:weak_concordance} and \textbf{Theorem}~\ref{theorem: Imakura concordance} is inherently limited. Indeed, Imakura's basis alignment method exhibits a notable discrepancy between its theoretical guarantees and empirical performance.

\subsubsection{FLOPs of Imakura-DC}
\label{subsubsec: Imakura FLOPs}

We adopt standard BLAS/LAPACK floating-point operation (FLOP) counts~\cite{GEMM, QR, SVD, trinv}, summarized in Table~\ref{tab:flopdef}.

\begin{table}[h]
\centering
\caption{Standard matrix operations and their approximate floating-point operation (flop) counts.}
\begin{tabular}{lll}
\hline
Operation & Matrix & Approximate FLOPs \\ \hline
Matrix product ($\matr{C} = \matr{A}\matr{B}$)
  & $\matr{A} \in \mathbb{R}^{m \times k},\ \matr{B} \in \mathbb{R}^{k \times n}$
  & $\approx 2 m k n$ \\[4pt]

QR factorization (Householder, $n = \min\{m, n\}$)
  & $\matr{A} \in \mathbb{R}^{m \times n}$
  & $\approx 2 m n^{2} - \tfrac{2}{3} n^{3}$ \\[4pt]

SVD ($n = \min\{m, n\}$)
  & $\matr{A} \in \mathbb{R}^{m \times n}$
  & $\approx 4 m n^{2} + 8 n^{3}$ \\[4pt]

Triangular inverse
  & $\matr{R} \in \mathbb{R}^{n \times n}$
  & $\approx \tfrac{1}{3} n^{3}$ \\ \hline
\end{tabular}
\label{tab:flopdef}
\end{table}

In \textbf{Algorithm}~\ref{alg: Imakura-DC}, the concatenated anchor matrix
\begin{equation}
\bigl[\,\matr{A}_1~\cdots~\matr{A}_c\,\bigr] \in \mathbb{R}^{a \times c\ell}
\end{equation}
has size $a \times c\ell$. Let
\begin{equation}
p = \min\{a, c\ell\},\qquad q = \max\{a, c\ell\}.
\end{equation}
Using the FLOP counts in Table~\ref{tab:flopdef}, we estimate the cost of
Imakura-DC as follows:
\begin{itemize}
  \item Computing the SVD of the $a \times c\ell$ matrix
        $\bigl[\,\matr{A}_1~\cdots~\matr{A}_c\,\bigr]$ costs
        \begin{equation}
        4 q p^{2} + 8 p^{3} \ \text{FLOPs}.
        \end{equation}
  \item The matrix $\matr{U} \in \mathbb{R}^{a \times \ell}$ contains the top
        $\ell$ left singular vectors. Forming
        $\matr{Z} = \matr{U}\matr{R}$ multiplies an $a \times \ell$ matrix by
        an $\ell \times \ell$ matrix and costs
        \begin{equation}
        2 a \ell^{2} \ \text{FLOPs}.
        \end{equation}
  \item For each user $i \in [c]$, forming the Moore--Penrose pseudoinverse
        of $\matr{A}_i \in \mathbb{R}^{a \times \ell}$ via SVD and multiplying
        it by $\matr{Z}$ costs
        \begin{equation}
        4 a \ell^{2} + 8 \ell^{3} \;+\; 2 a \ell^{2}
        \;=\; 6 a \ell^{2} + 8 \ell^{3} \ \text{FLOPs}.
        \end{equation}
\end{itemize}
Thus, the total approximate FLOP count for Imakura-DC is
\begin{equation}
\label{eq:imakura-flops}
4 q p^{2} + 8 p^{3}
   + 2 a \ell^{2}
   + c\bigl(6 a \ell^{2} + 8 \ell^{3}\bigr) 
= 4 q p^{2} + 8 p^{3}
   + (6c + 2)\,a \ell^{2}
   + 8 c \ell^{3} \ \text{FLOPs}.
\end{equation}
Under the assumption $a > \ell$, this can be summarized in big-$O$ notation as
\begin{equation}
O\bigl(\min\{a(c\ell)^{2},\, a^{2}c\ell\}\bigr).
\end{equation}

\subsubsection{Peak memory of Imakura-DC}
\label{subsubsec: Imakura memory}
We measure memory in units of real scalars. The $c$ matrices $\matr{A}_i \in \mathbb{R}^{a \times \ell}$ can be viewed as a single concatenated anchor
\begin{equation}
  \bar{\matr{A}}
  = \bigl[\,\matr{A}_1~\cdots~\matr{A}_c\,\bigr]
  \in \mathbb{R}^{a \times c\ell},
\end{equation}
which already requires $a c \ell$ scalars to store.  Let
\begin{equation}
  p = \min\{a, c\ell\}.
\end{equation}
We now bound the peak memory footprint of \textbf{Algorithm}~\ref{alg: Imakura-DC} by inspecting each step. We assume $a > \ell$ throughout.

\begin{itemize}
  \item \textbf{SVD of the concatenated anchor.}
        During the computation of the SVD of $\bar{\matr{A}}$ we store $\bar{\matr{A}}$ itself ($a c \ell$ scalars), the top $\ell$ left singular vectors $\matr{U} \in \mathbb{R}^{a \times \ell}$ ($a\ell$ scalars), and an SVD workspace of size $O(p^{2})$ scalars. Since $p = \min\{a, c\ell\}$, we have $p^{2} \le a c \ell$, and clearly $a\ell \le a c \ell$, so this step uses
        \begin{equation}
          a c \ell + O(a\ell + p^{2})
          = O(a c \ell).
        \end{equation}

  \item \textbf{Forming the target $\matr{Z} = \matr{U}\matr{R}$.} Next, we form the target
        $\matr{Z} = \matr{U}\matr{R}$ with
        $\matr{Z} \in \mathbb{R}^{a \times \ell}$ and
        $\matr{R} \in \mathcal{GL}(\ell)$, which require $a\ell$ and
        $\ell^{2}$ scalars, respectively.  At this point, we still store
        $\bar{\matr{A}}$ and $\matr{U}$, so the memory usage is
        \begin{equation}
          a c \ell + 2 a \ell + \ell^{2}
          = O(a c \ell).
        \end{equation}
        (The SVD workspace from the previous step can be released before
        forming $\matr{Z}$.)

  \item \textbf{Per-user pseudoinverse and alignment.}
        For each $i \in [c]$, we form $\matr{A}_i^\dagger$ via the SVD
        of $\matr{A}_i \in \mathbb{R}^{a \times \ell}$ and multiply by
        $\matr{Z}$.  This per-user computation requires
        $O(a\ell + \ell^{2})$ scalars of workspace, in
        addition to storing $\bar{\matr{A}}$ and $\matr{Z}$.  Storing
        all outputs $\matr{G}_i^*$, where
        $\matr{G}_i^* \in \mathbb{R}^{\ell \times \ell}$, costs
        $c \ell^{2}$ scalars.  Thus, near the end of the loop, we have
        \begin{equation}
          a c \ell
            + O(a\ell + c \ell^{2}).
        \end{equation}
\end{itemize}

Since $a c \ell \ge a\ell$ for all $c \ge 1$ and, under the natural
regime $a \ge \ell$, also $a c \ell \ge c \ell^{2}$, the peak memory
requirement is dominated by storing the concatenated anchor:
\begin{equation}
  a c \ell + O(a\ell + c \ell^{2})
  = \Theta(a c \ell).
\end{equation}

\subsection{Kawakami's Basis Alignment (Kawakami-DC)}
\label{subsection: Kamakami}

Imakura--DC formulates basis alignment by introducing an explicit target matrix that each user must match. Kawakami \textit{et al.}~\cite{KawakamiDC} instead propose a target-free formulation that directly aligns the user-specific anchor representations. In the setting of Problem~\eqref{problem:non-compact}, the main difficulty is that the feasible set is non-compact: without additional constraints, one can rescale the change-of-basis matrices so that their singular values decay and the objective is minimized by sequences converging to the zero matrix.

To avoid this degeneracy, Kawakami \textit{et al.}~\cite{KawakamiDC} introduce column-wise normalization constraints. Kawakami--DC seeks matrices $\matr{G}_i$ that solve
\begin{equation}
    \label{prob: Kawakami}
    \begin{aligned}
        \min_{\matr{G}_{i}\in \mathbb{R}^{\ell \times \ell}}\quad
        & \sum_{i=1}^{c}\sum_{j=1}^{c}
        \bigl\|
            \matr{A}_{i}\matr{G}_{i}-\matr{A}_{j}\matr{G}_{j}
        \bigr\|_{\mathrm{F}}^{2} \\[4pt]
        \text{s.t.}\quad
        & \sum_{i=1}^{c}\bigl\|\matr{A}_{i} g_{i,k}\bigr\|_{2}^{2}=1,
        \qquad k\in\{1, 2, \ldots, \ell\},
    \end{aligned}
\end{equation}
where $g_{i,k}$ denotes the $k$-th column of $\matr{G}_i$.

The constraint fixes the global scale of the $k$-th aligned feature across all users, as measured on the anchor, while leaving its direction free. Under the standard DC assumption that each $\matr{A}_i$ has full column rank, the constraint does not restrict which directions $g_{i,k}$ are admissible: for any nonzero collection of $g_{i,k}$, we can always rescale them to satisfy:
\begin{equation}
    \sum_{i=1}^{c}\bigl\|\matr{A}_{i} g_{i,k}\bigr\|_{2}^{2}=1.
\end{equation}
For Problem~\eqref{prob: Kawakami}, Kawakami \textit{et al.}~\cite{KawakamiDC} show that the matrices $\matr{G}_i$ can be computed efficiently via a QR-SVD-based algorithm (\textbf{Algorithm}~\ref{alg: Kawakami-DC}).

Importantly, the primary role of this formulation is to exclude the trivial all-zero solution; it does \emph{not} guarantee that the resulting matrices \(\matr{G}_i\) are invertible (and thus they need not constitute genuine change-of-basis matrices). Moreover, the concordance properties of Kawakami-DC remain unknown and constitute an open problem for future work.

\begin{algorithm}[t]
    \caption{Kawakami-DC's Basis Alignment Procedure~\cite{KawakamiDC}}
    \label{alg: Kawakami-DC}
    \SetAlgoLined \DontPrintSemicolon \LinesNumbered
    \SetNlSty{textbf}{}{:}
    \SetNlSkip{0.5em}
    \SetKwInOut{Input}{Input}
    \SetKwInOut{Output}{Output}
    \Input{$\matr{A}_{i}\in \mathbb{R}^{a \times \ell}$ for each user $i \in [c]$}
    \Output{$\matr{G}_{i}^{*}\in \mathbb{R}^{\ell \times \ell}$ for each user $i \in [c]$}
    \BlankLine
    \Begin{
        \For{$i \in [c]$}{
            Compute a thin QR factorization $\matr{A}_{i} = \matr{Q}_{i}\matr{R}_{i}$\;
        }
        Form the concatenated matrix
        $\matr{W}_{Q}\coloneqq [\,\matr{Q}_{1}~\cdots~\matr{Q}_{c}\,]\in\mathbb{R}^{a\times c\ell}$\;
        Compute the SVD $\matr{W}_{Q}= \matr{U}\matr{\Sigma}\matr{V}^{\top}$\;
        Let $\matr{V}_{\ell}$ contain the $\ell$ right singular vectors of $\matr{W}_{Q}$ associated with the largest singular values\;
        \For{$k \in \{1,2,\ldots,\ell\}$}{
            Set $v'_{k}$ to the $k$-th column of $\matr{V}_{\ell}$ and partition
            $v'_{k}$ into blocks $\hat{g}_{i,k}\in\mathbb{R}^{\ell}$ for all $i\in[c]$\;
        }
        \For{$i \in [c]$}{
            Set
            $\matr{G}_{i}^{*}\coloneqq [\,\matr{R}_{i}^{-1}\hat{g}_{i,1}~\cdots~\matr{R}_{i}^{-1}\hat{g}_{i,\ell}\,]$\;
        }
    }
\end{algorithm}

\subsubsection{FLOPs of Kawakami-DC}
\label{subsubsec: Kawakami FLOPs}

In Kawakami-DC, each anchor matrix $\matr{A}_{i}\in\mathbb{R}^{a\times \ell}$ is factorized as $\matr{A}_{i}=\matr{Q}_{i}\matr{R}_{i}$, and the orthonormal factors are concatenated to form
\begin{equation}
    \matr{W}_{Q}
    = \bigl[\,\matr{Q}_{1}~\cdots~\matr{Q}_{c}\,\bigr]
    \in \mathbb{R}^{a\times c\ell}.
\end{equation}
Let
\begin{equation}
    p = \min\{a, c\ell\},\qquad q = \max\{a, c\ell\}.
\end{equation}
Using the FLOP counts in Table~\ref{tab:flopdef}, we estimate the cost of Kawakami-DC (\textbf{Algorithm}~\ref{alg: Kawakami-DC}) as follows:

\begin{itemize}
    \item For each $\matr{A}_{i}\in\mathbb{R}^{a\times \ell}$, the thin QR
        factorization $\matr{A}_{i}=\matr{Q}_{i}\matr{R}_{i}$ costs
        \begin{equation}
            2 a\ell^{2} - \frac{2}{3}\ell^{3}\ \text{FLOPs}.
        \end{equation}
        Summed over all $c$ users, the total becomes
        \begin{equation}
            c\Bigl(2 a\ell^{2} - \frac{2}{3}\ell^{3}\Bigr).
        \end{equation}

    \item Computing the SVD of the concatenated matrix
        $\matr{W}_{Q}\in\mathbb{R}^{a\times c\ell}$ costs
        \begin{equation}
            4 q p^{2} + 8 p^{3}\ \text{FLOPs}.
        \end{equation}

    \item For each user $i\in[c]$, we explicitly compute the inverse
        $\matr{R}_{i}^{-1}$ of the upper triangular matrix $\matr{R}_{i}$. This costs
        \begin{equation}
            \frac{1}{3}\ell^{3}\ \text{FLOPs}.
        \end{equation}
        Then, for each $k\in[\ell]$, we recover $g_{i,k}$ via the matrix--vector multiplication $\matr{R}_{i}^{-1}\hat g_{i,k}$, which costs $\ell^{2}$ FLOPs per column. Since there are $\ell$ columns, the total per user is
        \begin{equation}
            \frac{1}{3}\ell^{3} + \ell^{3}
            = \frac{4}{3}\ell^{3}.
        \end{equation}
        Summed over all users, the recovery step costs
        \begin{equation}
            \frac{4}{3} c\,\ell^{3}.
        \end{equation}
\end{itemize}

Thus, the total approximate FLOP count for Kawakami-DC is
\begin{equation}
    \label{eq:kawakami-flops-inverse}
    4 q p^{2} + 8 p^{3}
    + c\Bigl(2 a\ell^{2} - \frac{2}{3}\ell^{3}\Bigr)
    + \frac{4}{3}c\,\ell^{3}
    =
    4 q p^{2} + 8 p^{3}
    + c\!\left(2 a\ell^{2}
        + \frac{2}{3}\ell^{3}
      \right)
    \ \text{FLOPs}.
\end{equation}

Under the assumption $a > \ell$, this can be summarized in big-$O$ notation as
\begin{equation}
    O\bigl(\min\{a(c\ell)^{2},\, a^{2}c\ell\}\bigr).
\end{equation}

\subsubsection{Peak memory of Kawakami-DC}
\label{subsubsec: Kawakami memory}

We measure memory in units of real scalars. The $c$ matrices $\matr{A}_{i}\in \mathbb{R}^{a\times \ell}$ are again stored explicitly, requiring
\begin{equation}
  a c \ell
\end{equation}
scalars in total. Let
\begin{equation}
  p = \min\{a,\, c\ell\}.
\end{equation}
We now inspect each step of the Kawakami-DC procedure and bound its peak memory footprint.

\begin{itemize}

  \item \textbf{QR factorizations of $\matr{A}_{i}$.}
        Each $\matr{A}_{i}$ is factorized as $\matr{A}_{i} = \matr{Q}_{i}\matr{R}_{i}$, where $\matr{Q}_{i}\in\mathbb{R}^{a\times \ell}$ and $\matr{R}_{i}\in\mathbb{R}^{\ell\times \ell}$. While computing the QR factorization, we store $\matr{A}_{i}$ itself ($a\ell$ scalars), along with the Householder vectors and a workspace of size $O(a\ell)$ scalars. Since $a c \ell$ dominates $a\ell$, this step fits within
        \begin{equation}
          a c \ell + O(a\ell) = O(a c \ell).
        \end{equation}

  \item \textbf{Storing all $\matr{Q}_{i}$ and $\matr{R}_{i}$.}
        After QR has been completed for all users, we overwrite each $\matr{A}_{i}$ with $\matr{Q}_{i}$ and store
        \begin{equation}
          \matr{Q}_{1},\dots,\matr{Q}_{c}\in\mathbb{R}^{a\times \ell},
          \qquad
          \matr{R}_{1},\dots,\matr{R}_{c}\in\mathbb{R}^{\ell\times \ell}.
        \end{equation}
        This costs
        \begin{equation}
          c(a\ell) + c\ell^{2} = a c \ell + c\ell^{2}.
        \end{equation}
        Since $a > \ell$, we have $a c \ell > c\ell^{2}$, hence this stage also requires $O(a c \ell)$ scalars.

  \item \textbf{SVD of the concatenated matrix $\matr{W}_{Q} = [\,\matr{Q}_{1}~\cdots~\matr{Q}_{c}\,]$.} The concatenated matrix $\matr{W}_{Q}\in\mathbb{R}^{a\times c\ell}$ requires $a c \ell$ scalars to store. The SVD workspace uses $O(p^{2})$ scalars with $p = \min\{a,\, c\ell\} \le a c \ell$. Thus, during SVD, we store
        \begin{equation}
          a c \ell + O(p^{2}) = O(a c \ell).
        \end{equation}

  \item \textbf{Recovery of the vectors $g_{i,k}$.}
        From the SVD, we obtain the blocks $\hat g_{i,k}\in\mathbb{R}^{\ell}$. To recover the original vectors $g_{i,k}$ we solve $\matr{R}_{i} g_{i,k} = \hat g_{i,k}$ by back substitution. Each solve requires $O(\ell^{2})$ scalars of temporary workspace. Since we already store $\matr{W}_{Q}$ (or the $\matr{Q}_{i}$) and $\matr{R}_{i}$, the memory during this step is
        \begin{equation}
          a c \ell + O(\ell^{2}) = O(a c \ell).
        \end{equation}
        Storing all $\matr{G}_{i}^{*} \in \mathbb{R}^{\ell\times \ell}$ costs $c\ell^{2}$ scalars, which is dominated by $a c \ell$.
\end{itemize}

Combining all contributions, the peak memory usage is dominated by storing the concatenated $\matr{Q}$ matrices (equivalently, the anchor data), giving:
\begin{equation}
  a c \ell + O(a\ell + c\ell^{2}) = \Theta(a c \ell).
\end{equation}
The term $a c \ell$ dominates all others, so the peak memory of Kawakami-DC (QR--SVD) matches that of Imakura-DC.

\section{Proposed Basis Alignment}
\label{section: ODC}

Here, we present our proposed basis-alignment procedure, \emph{Orthonormal Data Collaboration (ODC)}. Our theoretical findings, along with a comparison to existing DC alignment methods, are summarized in Table~\ref{tab: DCcomparison}. In this section, we proceed as follows:
\begin{enumerate}
    \item We show that, under \textbf{Assumption}~\ref{assumptions: ODC}, the basis-alignment problem naturally reduces to the classical Orthogonal Procrustes Problem, which admits a closed-form solution.
    \item We address the instability inherent in the notion of weak concordance and introduce a refined notion, \emph{orthogonal concordance}.
    \item We prove that the ODC basis-alignment procedure satisfies orthogonal concordance.
    \item We analyze the computational efficiency of ODC and show that it reduces the alignment time complexity from \(O\!\left(\min\{a(c\ell)^2,\; a^2 c \ell\}\right)\) for contemporary DC methods to \(O(a c \ell^2)\).
\end{enumerate}

\begin{table*}[t]
    \centering
    \caption{Theoretical comparison between contemporary DC methods and the proposed ODC framework.}
    \resizebox{\textwidth}{!}{
    \begin{tabular}{@{}llll@{}}
        \toprule
        & \textbf{Imakura-DC} & \textbf{Kawakami-DC} & \textbf{ODC} (this work)\\
        \midrule
        \textbf{Objective}
        &
        $\displaystyle \min_{\matr{G}_i \in \mathbb{R}^{\ell \times \ell}} \|\matr{A}_i\matr{G}_i - \matr{Z}\|_\mathrm{F}^2$
        &
        $\displaystyle \min_{\matr{G}_{i}\in \mathbb{R}^{\ell \times \ell}} \sum_{i, j \in [c]}
        \bigl\|
            \matr{A}_{i}\matr{G}_{i}-\matr{A}_{j}\matr{G}_{j}
        \bigr\|_{\mathrm{F}}^{2}$
        &
        $\displaystyle \min_{\matr{G}_{i}\in \mathbb{R}^{\ell \times \ell}} \sum_{i, j \in [c]}
        \bigl\|
            \matr{A}_{i}\matr{G}_{i}-\matr{A}_{j}\matr{G}_{j}
        \bigr\|_{\mathrm{F}}^{2}$
        \\
        \textbf{Constraint}
        &
        $\matr{G}_i \in \mathcal{GL}(\ell)$
        &
        $\displaystyle \sum_{i=1}^{c}\bigl\|\matr{A}_{i} g_{i,k}\bigr\|_{2}^{2}=1$
        &
        $\matr{G}_i \in \mathcal{O}(\ell)$
        \\
        \textbf{Solution Property}
        &
        Ambiguous up to a common right invertible transform
        &
        Not necessarily invertible
        &
        Ambiguous up to a common right orthogonal transform
        \\
        \textbf{Concordance}
        &
        Weak concordance
        &
        Unknown
        &
        Orthogonal concordance
        \\
        \textbf{FLOPs}
        &
        \( O\bigl(\min\{a(c\ell)^2,\, a^2c\ell\}\bigr) \)
        &
        \( O\bigl(\min\{a(c\ell)^2,\, a^2c\ell\}\bigr) \)
        &
        \( O(a c \ell^2) \)
        \\
        \bottomrule
    \end{tabular}
    }
    \label{tab: DCcomparison}
\end{table*}

We begin by introducing \textbf{Assumption}~\ref{assumptions: ODC}.

\begin{restatable}{assumption}{orthogonalassumptions}
    \label{assumptions: ODC}
    We impose the following conditions on the anchor dataset \( \matr{A} \in \mathbb{R}^{a \times m} \) and the secret bases \( \matr{F}_i \in \mathbb{R}^{m \times \ell} \; (m \geq \ell) \), for all \(i \in [c]\):
    \begin{enumerate}[label=\arabic*)]
        \item \( \mathrm{rank}(\matr{A}) = m \);
        \item There exists \( \matr{E}_i \in \mathcal{GL}(\ell) \) such that \( \matr{F}_i = \matr{F}_1\matr{E}_i \);
        \item \(\matr{F}_i^\top \matr{F}_i = \matr{I}\).
    \end{enumerate}
\end{restatable}

Conditions 1) and 2) coincide with those in \textbf{Assumption}~\ref{assumptions: DC}. We additionally impose condition 3), which requires the secret bases to be orthonormal within the ODC framework. It is noteworthy that standard basis-selection methods such as PCA and SVD naturally yield orthonormal bases. Furthermore, to the best of our knowledge, all existing DC applications can readily accommodate this additional orthonormality constraint.

\subsection{Basis Alignment as the Orthogonal Procrustes Problem}

Similarly to Imakura-DC, our goal is to identify \( \matr{G}_i \in \mathcal{GL}(\ell) \) for each \( i \in [c] \) that satisfy Eq.~\eqref{eq: basis equation} without knowledge of the secret bases \(\matr{F}_i\).

Now, from conditions 2) and 3), it immediately follows that \( \matr{E}_i \in \mathcal{GL}(\ell) \) are orthogonal (i.e., \(\matr{E}_i \in \mathcal{O}(\ell)\)), because:
\begin{equation}
\matr{F}_i^\top \matr{F}_i = (\matr{F}_1\matr{E}_i)^\top (\matr{F}_1\matr{E}_i) 
= \matr{E}_i^\top \matr{F}_1^\top \matr{F}_1 \matr{E}_i 
= \matr{E}_i^\top \matr{E}_i 
= \matr{I}_\ell.
\end{equation}

Therefore, we can equivalently write our goal as to identify \( \matr{G}_i \in \mathcal{O}(\ell) \) for each \( i \in [c] \) that satisfy Eq.~\eqref{eq: basis equation} without knowledge of the secret bases \(\matr{F}_i\).

Given the analyst has access to \(\matr{A}_i = \matr{A}\matr{F}_i \in \mathbb{R}^{a\times \ell}\), \(\matr{G}_i \in \mathcal{O(\ell)}\) necessarily satisfy:

\begin{equation}
    \matr{A}_1\matr{G}_1 = \matr{A}_2\matr{G}_2 = \dots = \matr{A}_c\matr{G}_c.
\end{equation}

Consequently, \(\matr{G}_i\) necessarily minimize the following optimization problem:

\begin{equation}
\label{prob: GOPP}
\min_{\matr{Z} \in \mathbb{R}^{a \times \ell},\,\matr{G}_i \in \mathcal{O}(\ell)} \quad \sum_{i=1}^{c} \|\matr{A}_i \matr{G}_i - \matr{Z}\|_\mathrm{F}^2.
\end{equation}

Importantly, Problem~\eqref{prob: GOPP} involves a matrix norm minimization objective, which inherently exhibits invariance under arbitrary common orthogonal transformations from the right. Specifically, let \(\matr{O} \in \mathcal{O}(\ell)\) be an arbitrary orthogonal matrix. Then, choosing \(\matr{Z}^* = \matr{A}_1\matr{O}\) and \(\matr{G}_i^* = \matr{E}_i^\top \matr{O}\) yields global minimizers for Problem~\eqref{prob: GOPP}, because:
\begin{align*}
    \sum_{i=1}^{c} \|\matr{A}_i \matr{E}_{i}^\top \matr{O} - \matr{A}_1 \matr{O}\|_\mathrm{F}^2
    &= \sum_{i=1}^{c}\|\matr{A}\matr{F}_i \matr{E}_i^\top \matr{O} - \matr{A}\matr{F}_1 \matr{O}\|_\mathrm{F}^2 \\
    &= \sum_{i=1}^{c}\|\matr{A}\matr{F}_1 \matr{O} - \matr{A}\matr{F}_1 \matr{O}\|_\mathrm{F}^2 \\
    &= 0.
\end{align*}

Given that the analyst has access to \(\matr{A}_i\) for all \(i \in [c]\), we may fix \(\matr{Z} = \matr{A}_1 \matr{O}\) in Problem~\eqref{prob: GOPP}, leading directly to the classical \textit{Orthogonal Procrustes Problem (OPP)}:

\begin{equation}
\label{OPP}
\min_{\matr{G}_i \in \mathcal{O}(\ell)} \quad \sum_{i=1}^{c} \| \matr{A}_i \matr{G}_i - \matr{A}_1 \matr{O} \|_\mathrm{F}^2.
\tag{OPP}
\end{equation}

The closed-form solutions are given by:

\begin{equation}
    \label{eq: ODC_Gi}
    \matr{G}_i^* = \matr{U}_i \matr{V}_i^\top,
\end{equation}
where:
\begin{equation}
    \matr{A}_i^\top \matr{A}_1 \matr{O} = \matr{U}_i \matr{\Sigma}_i \matr{V}_i^\top,
\end{equation}
as established in~\cite{procrustessolution}.

\begin{algorithm}[t]
\caption{ODC's Basis Alignment Procedure}
\label{alg: ODC}
\SetAlgoLined
\DontPrintSemicolon
\LinesNumbered
\SetNlSty{textbf}{}{:}   
\SetNlSkip{0.5em}         

\SetKwInOut{Input}{Input}
\SetKwInOut{Output}{Output}

\Input{\( \matr{A}_i \in \mathbb{R}^{a \times \ell} \) for each user \( i \in [c] \)}
\Output{\(\matr{G}_i^* \in \mathcal{O}(\ell) \) for each user \( i \in [c] \)}
\BlankLine
\Begin{
    Randomly sample $\matr{O}\in\mathcal{O}(\ell)$ and compute \(\matr{A}_i^\top \matr{A}_1 \matr{O} \;=\; \matr{U}_i \matr{\Sigma}_i \matr{V}_i^\top\)\;
    Set $\matr{G}_i^* \coloneqq \matr{U}_i \matr{V}_i^\top$\;
}

\end{algorithm}

ODC's basis alignment procedure is summarized in \textbf{Algorithm}~\ref{alg: ODC}. 

\subsection{Orthogonal Concordance of ODC}
\label{subsec: ODC concordance}

A fundamental limitation of weak concordance is its invariance under arbitrary right-multiplication by an invertible matrix. Specifically, if  $\matr{G}_i^*$ satisfy weak concordance, then so do the matrices $\matr{G}_i^* \matr{R}$ for any $\matr{R} \in \mathcal{GL}(\ell)$. In practical settings, however, an ill-conditioned choice of $\matr{R}$ can severely impair downstream model performance, even though the resulting matrices remain weakly concordant. To eliminate this degree of freedom, we introduce a stricter requirement---\emph{orthogonal concordance}.

\begin{definition} \label{def:orth_concordance} (Orthogonal Concordance)\\ 
The orthogonal change-of-basis matrices \(\matr{G}_i \in \mathcal{O}(\ell)\), for \( i \in [c] \), satisfy orthogonal concordance if 
\begin{equation} 
\matr{F}_1 \matr{G}_1 = \matr{F}_2 \matr{G}_2 = \dots = \matr{F}_c \matr{G}_c. 
\end{equation} 
\end{definition}

Orthogonal concordance enforces the same alignment condition as weak concordance but additionally constrains each $\matr{G}_i$ to be orthogonal. Consequently, concordance is preserved only under right-multiplication by orthogonal matrices, not by arbitrary invertible transformations. Because orthogonal transformations are Euclidean isometries, distance-based models such as SVMs are theoretically invariant to such transformations~\cite{GeometricPerturbation}. Moreover, even models that are not strictly distance-based (e.g., MLPs) are empirically observed to exhibit only limited sensitivity, as demonstrated in our experiments (see \S~\ref{sec: concordance evaluation}).

\begin{theorem}
    \label{theorem: orthogonal concordance}
    (Orthogonal Concordance of ODC)\\
    Suppose that we observe matrices \(\matr{A}_i = \matr{A}\matr{F}_i\), \( i \in [c] \), with \(\matr{A}\in \mathbb{R}^{a\times m}\) and \(\matr{F}_i\in\mathbb{R}^{m\times \ell}\). For any orthogonal matrix \(\matr{O} \in \mathcal{O}(\ell)\), let \(\matr{A}_i^\top \matr{A}_1 \matr{O} = \matr{U}_i \matr{\Sigma}_i \matr{V}_i^\top \) denote the SVD. Under \textbf{Assumption}~\ref{assumptions: ODC}, for each \( i \in [c] \), the solution \( \matr{G}_i^* = \matr{U}_i \matr{V}_i^\top\) to the Orthogonal Procrustes Problem:
    \begin{equation}
    \min_{\matr{G}_i \in \mathcal{O}(\ell)} \quad \sum_{i=1}^{c} \| \matr{A}_i \matr{G}_i - \matr{A}_1 \matr{O} \|_\mathrm{F}^2,
    \end{equation}
    satisfies
    \begin{equation}
    \matr{G}_i^* = \matr{E}_i^\top \matr{O},
    \end{equation}
    thereby guaranteeing orthogonal concordance~(\textbf{Definition}~\ref{def:orth_concordance}).
\end{theorem}

\begin{proof}
We prove for all \(i \in [c]\). Given \eqref{OPP}, we can write:
\begin{align*}
    \| \matr{A}_i \matr{G}_i - \matr{A}_1 \matr{O} \|_\mathrm{F}^2 
    &= \tr\left((\matr{A}_i \matr{G}_i - \matr{A}_1 \matr{O})^\top (\matr{A}_i \matr{G}_i - \matr{A}_1 \matr{O})\right) \\
    &= \| \matr{A}_i \|_\mathrm{F}^2 + \| \matr{A}_1 \|_\mathrm{F}^2 - 2\tr(\matr{G}_i^\top \matr{A}_i^\top \matr{A}_1 \matr{O}),
\end{align*}
where \(\tr(\cdot)\) denotes the matrix trace. Minimizing \(\| \matr{A}_i \matr{G}_i - \matr{A}_1 \matr{O} \|_\mathrm{F}^2\) for each \(\matr{G}_i\) individually is equivalent to minimizing \(\sum_{i=1}^c \| \matr{A}_i \matr{G}_i - \matr{A}_1 \matr{O} \|_\mathrm{F}^2\) for all \(\matr{G}_i\). Therefore, solving \eqref{OPP} is equivalent to solving:
\begin{equation}
\label{eq: GOPP_trace}
\begin{aligned}
\max_{\matr{G}_i \in \mathcal{O}(\ell)} \quad & \tr(\matr{G}_i^\top \matr{A}_i^\top \matr{A}_1 \matr{O}),
\end{aligned}
\end{equation}
for each \(i \in [c]\). 

From condition 3) of \textbf{Assumption}~\ref{assumptions: ODC}, we have:
\begin{equation}
\matr{A}_1 \matr{O} = \matr{A} \matr{F}_1 \matr{O} = \matr{A} \matr{F}_i \matr{E}_{i}^\top \matr{O} = \matr{A}_i \matr{E}_{i}^\top \matr{O}.
\end{equation}
Substitute this into Problem \eqref{eq: GOPP_trace}, and let \(\matr{A}_i^\top \matr{A}_i = \matr{Q}_i \matr{\Lambda}_i \matr{Q}_i^\top\) be the eigenvalue decomposition. We have:
\begin{align*}
    \tr(\matr{G}_i^\top \matr{A}_i^\top \matr{A}_1 \matr{O}) &= \tr(\matr{G}_i^\top \matr{A}_i^\top \matr{A}_i \matr{E}_{i}^\top \matr{O} ) \\
                                             &= \tr(\matr{G}_i^\top \matr{Q}_i \matr{\Lambda}_i \matr{Q}_i^\top \matr{E}_{i}^\top \matr{O} ) \\
                                             &= \tr(\matr{Q}_i^\top \matr{E}_{i}^\top \matr{O} \matr{G}_i^\top \matr{Q}_i \matr{\Lambda}_i ) \\
                                             &= \tr(\matr{W}'_i \matr{\Lambda}_i) \\
                                             &= \sum_{s=1}^{\ell} w'_{i,(s,s)} \lambda_{i,(s,s)},
                                                 \stepcounter{equation}\tag{\theequation}\label{eq: procrustes_ED}
\end{align*}
where \(\matr{W}'_i = \matr{Q}_i^\top \matr{E}_{i}^\top \matr{O} \matr{G}_i^\top \matr{Q}_i\), and \(w'_{i,(s,t)}, \lambda_{i,(s,t)}\) denote the \((s,t)\)-th elements of matrices \(\matr{W}'_i\) and \(\matr{\Lambda}_i\), respectively. Since \(\matr{W}'_i \in \mathcal{O}(\ell)\), \(w'_{i,(s,t)} \leq 1\) for all \(s, t\). Thus, the sum in \eqref{eq: procrustes_ED} is maximized when \(\matr{W}'_i = \matr{I}\), which gives:
\begin{align*}
    \matr{G}_i^* &= \matr{Q}_i \matr{Q}_i^\top \matr{E}_{i}^\top \matr{O} \\
    \matr{G}_i^* &= \matr{E}_{i}^\top \matr{O},
\end{align*}
and therefore, we have
\begin{equation}
\matr{F}_1 \matr{G}_1^* = \cdots = \matr{F}_c \matr{G}_c^*,
\end{equation}
which completes the proof.
\end{proof}

\textbf{Theorem~\ref{theorem: orthogonal concordance}} shows that the orthogonal right factor $\matr{O} \in \mathcal{O}(\ell)$ can be chosen arbitrarily while preserving orthogonal concordance. This flexibility naturally raises the question of whether such arbitrary choices might adversely affect downstream model performance. We argue that when assumptions are satisfied, any resulting performance variation is negligible, even for non-distance-based models, and we empirically validate this claim in \S\ref{sec: concordance evaluation}.

\subsection{FLOPs of ODC}
\label{subsec: ODC FLOPs}

In \textbf{Algorithm~\ref{alg: ODC}}, ODC aligns the user-specific bases by first sampling a random orthogonal matrix $\matr{O} \in \mathcal{O}(\ell)$ and then, for each user $i \in [c]$, forming the product
\(
\matr{A}_i^\top \matr{A}_1 \matr{O}
\in \mathbb{R}^{\ell \times \ell}
\)
and computing its singular value decomposition. Using the FLOP counts from Table~\ref{tab:flopdef}, we can estimate the cost as follows.

For each user $i$:
\begin{itemize}
  \item Forming $\matr{A}_i^\top \matr{A}_1$ multiplies an $\ell \times a$ matrix by an $a \times \ell$ matrix and thus costs $2 a \ell^{2}$ FLOPs.
  \item Multiplying by $\matr{O}$, i.e.\ computing $\matr{A}_i^\top \matr{A}_1 \matr{O}$, multiplies two $\ell \times \ell$ matrices and costs $2 \ell^{3}$ FLOPs.
  \item Computing the SVD
        $\matr{A}_i^\top \matr{A}_1 \matr{O} = \matr{U}_i \matr{\Sigma}_i \matr{V}_i^\top$ of an $\ell \times \ell$ matrix costs approximately $4 \ell^{3} + 8 \ell^{3} = 12 \ell^{3}$ FLOPs.
  \item Forming the alignment matrix $\matr{G}_i^* = \matr{U}_i \matr{V}_i^\top$ is an $\ell \times \ell$ matrix product and costs $2 \ell^{3}$ FLOPs.
\end{itemize}
Thus, the per-user cost is approximately.
\begin{equation}
2 a \ell^{2} + 2 \ell^{3} + 12 \ell^{3} + 2 \ell^{3} \;=\; 2 a \ell^{2} + 16 \ell^{3} \quad \text{FLOPs}.
\end{equation}
Summing over all $c$ users and ignoring the one-time cost of sampling $\matr{O}$, the total FLOP count for ODC is
\begin{equation}
\label{eq:odc-flops}
c \bigl(2 a \ell^{2} + 16 \ell^{3}\bigr) \;=\; 2 c a \ell^{2} + 16 c \ell^{3} \ \text{FLOPs}.
\end{equation}
In big-$O$ notation, this can be summarized as
\begin{equation}
O\bigl(c a \ell^{2} + c \ell^{3}\bigr) = O\bigl(a c \ell^{2}\bigr).
\end{equation}

\subsection{Peak memory of ODC}
\label{subsec: ODC memory}
We again measure memory in units of real scalars. The $c$ matrices $\matr{A}_i$ are stored explicitly, requiring
\begin{equation}
\sum_{i=1}^c a \ell = a c \ell
\end{equation}
scalars in total. During ODC, we also store the random orthogonal matrix $\matr{O} \in \mathcal{O}(\ell)$, which costs $\ell^{2}$ scalars, and the alignment matrices $\matr{G}_i^* \subset \mathbb{R}^{\ell \times \ell}$, which cost $c \ell^{2}$ scalars once all users have been processed.

For a fixed user $i \in [c]$, the computation of
\(
\matr{A}_i^\top \matr{A}_1 \matr{O}
\)
and its SVD can be organized so that only $O(\ell^{2})$ additional working storage is needed:
\begin{itemize}
  \item We form an intermediate product $\matr{B}_i = \matr{A}_i^\top \matr{A}_1 \in \mathbb{R}^{\ell \times \ell}$ (or equivalently $\matr{B}_i = \matr{A}_i^\top (\matr{A}_1 \matr{O})$) and overwrite it with $\matr{B}_i \matr{O}$. This requires a single $\ell \times \ell$ buffer, i.e.\ $\ell^{2}$ scalars.
  \item Computing the SVD $\matr{B}_i = \matr{U}_i \matr{\Sigma}_i \matr{V}_i^\top$ for an $\ell \times \ell$ matrix requires storing $\matr{U}_i, \matr{V}_i \in \mathbb{R}^{\ell \times \ell}$, the diagonal $\matr{\Sigma}_i$, and SVD workspace, all together costing $O(\ell^{2})$ scalars.
  \item Forming the alignment matrix $\matr{G}_i^* = \matr{U}_i \matr{V}_i^\top \in \mathbb{R}^{\ell \times \ell}$ requires another $\ell^{2}$ scalars, after which $\matr{U}_i$ and $\matr{V}_i$ can be discarded. Thus, at any given time, the per-user working storage is $O(\ell^{2})$.
\end{itemize}
Near the end of the loop over users, we therefore store all inputs $\matr{A}_i$ ($a c \ell$ scalars), all outputs $\matr{G}_i^*$ ($c \ell^{2}$ scalars), the random orthogonal matrix $\matr{O}$ ($\ell^{2}$ scalars), and $O(\ell^{2})$ scalars of temporary workspace. Hence, the peak memory requirement of ODC is
\begin{equation}
a c \ell + c \ell^{2} + O(\ell^{2}).
\end{equation}
Under the natural regime $a > \ell$, this is dominated by the storage of the transformed anchor data, and we also have
\begin{equation}
\Theta(a c \ell).
\end{equation}

\section{Empirics on Time Efficiency}
\label{sec:efficiency-eval}

\begin{table*}[t]
    \centering
    \caption{Computational Time Complexity Comparison of Basis-Alignment}
    \begin{tabular}{@{}llll@{}}
        \toprule
        & \textbf{Imakura-DC} & \textbf{Kawakami-DC} & \textbf{ODC}\\
        \midrule
        \textbf{Computational Time Complexity}  & \( O\left(\min\{a(c\ell)^2,\, a^2c\ell\}\right) \) & \( O\left(\min\{a(c\ell)^2,\, a^2c\ell\}\right) \) & \( O(ac\ell^2) \) \\
        \bottomrule
    \end{tabular}
    \label{tab: bigo}
\end{table*}

The computational time complexities (in big-\(O\) notation) of all basis-alignment methods are summarized in Table~\ref{tab: bigo}. As shown in Table~\ref{tab: bigo}, under the natural dimensional ordering \(\ell \le m < a\), the computational time complexity of ODC satisfies
\begin{equation}
    a c \ell^{2} \;<\; \min\{a(c\ell)^2,\, a^2 c \ell\}.
\end{equation}
This implies that ODC inherently incurs a lower computational cost than both Imakura-DC and Kawakami-DC. To empirically corroborate this theoretical advantage, we measured wall-clock execution times under controlled conditions.

All three methods admit closed-form expressions for the change-of-basis matrices $\matr{G}_i$. Consequently, runtime depends only on the dimensions of the local anchor representations $\matr{A}_i\in\mathbb{R}^{a\times \ell}$ for $i\in[c]$, rather than on $\matr{X}_i$ or $n_i$. Because the global anchor $\matr{A}\in\mathbb{R}^{a\times m}$ is sampled i.i.d.\ from the uniform distribution over $[0,1)$ independently of the private datasets $\matr{X}_i\in\mathbb{R}^{n_i\times m}$, the alignment cost is isolated from user-specific data characteristics.

For the empirical evaluation, we generated uniformly random matrices \(\matr{A}_i \in \mathbb{R}^{a \times \ell}\), \( i \in [c]\), and varied each of the three primary parameters (anchor size~\(a\), latent dimension~\(\ell\), and the number of users~\(c\)), while keeping the other two fixed. The specific experimental settings are summarized in Table~\ref{tab: exp-design-eff}. As prior literature~\cite{KawakamiDC} suggests, randomized SVD improves computational efficiency for contemporary DC methods; thus, we employed it for both Imakura-DC and Kawakami-DC. Conversely, ODC remained unchanged, as it inherently requires only a single full SVD of an \(\ell\times \ell\) matrix. Each experiment was repeated 100 times, and we report median (instead of mean) runtimes to mitigate transient system effects.

All CPU experiments were run on Google Colab Pro+ (CPU, high-memory runtime). The virtual machine reported Linux 6.6.105+ (x86\_64, glibc 2.35) as the operating system, an Intel(R) Xeon(R) CPU @ 2.20 GHz with 8 logical cores (2 threads per core), and approximately 51 GiB of RAM. We used Python~3.12.12, NumPy~2.0.2, and SciPy~1.16.3. Linear algebra kernels were provided by the \texttt{scipy-openblas} build of OpenBLAS~0.3.27 (64-bit integer interface, \texttt{DYNAMIC\_ARCH}, pthreads threading layer, Haswell configuration). We pinned threads via environment variables \texttt{OPENBLAS\_NUM\_THREADS=4}, \texttt{OMP\_NUM\_THREADS=4}, and \texttt{MKL\_NUM\_THREADS=4} (where applicable), and used double-precision (float64) C-contiguous arrays throughout.

Figure~\ref{fig:time_scaling} plots the measured median running times. Across all parameter variations, ODC consistently outperforms contemporary DC approaches, with empirical speed-ups ranging from roughly \(6\times\) to over \(100\times\), validating the theoretical complexity hierarchy detailed in Table~\ref{tab: bigo}.

\begin{figure}[t]
    \centering

    \begin{subfigure}[b]{0.45\textwidth}
        \centering
        \includegraphics[width=\textwidth]{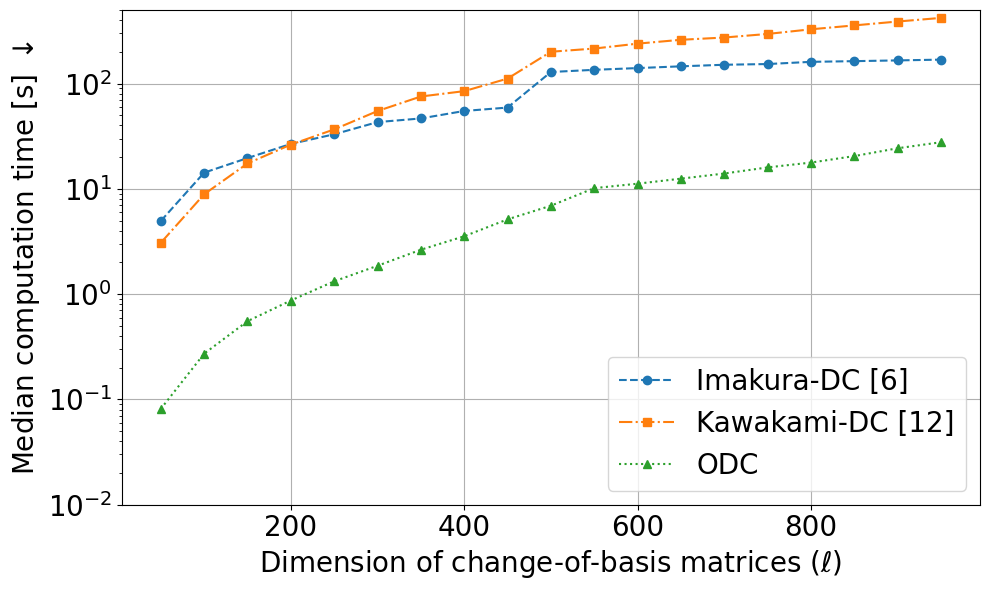}
        \caption{Wall-clock time versus latent dimension~\(\ell\)\,\((a,c) = (1000,50)\).}
        \label{fig:vary_l_time}
    \end{subfigure}
    \hfill
    \begin{subfigure}[b]{0.45\textwidth}
        \centering
        \includegraphics[width=\textwidth]{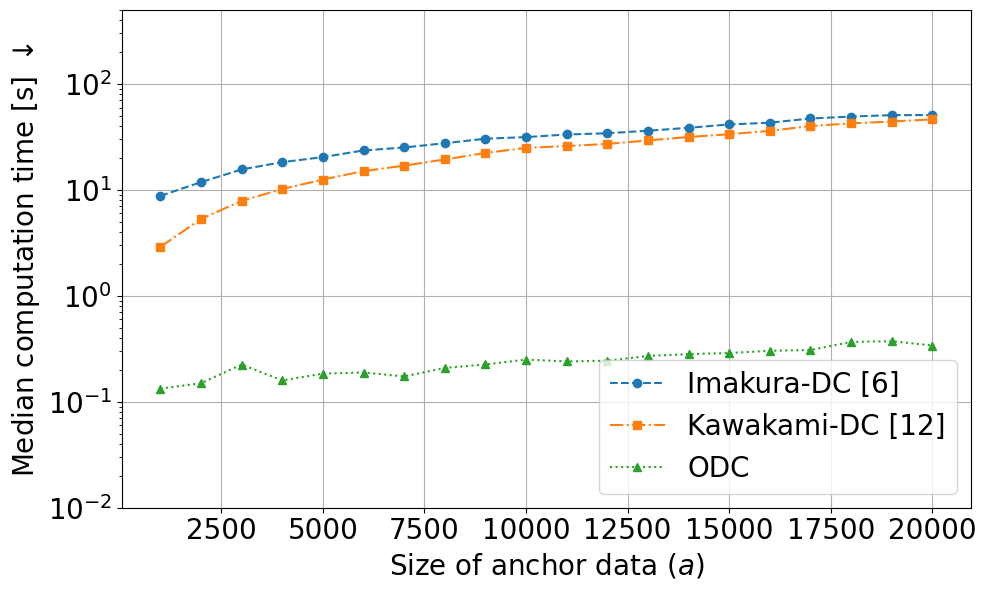}
        \caption{Wall-clock time versus anchor size~\(a\)\,\((\ell,c) = (50,50)\).}
        \label{fig:vary_a_time}
    \end{subfigure}

    \begin{subfigure}[b]{0.45\textwidth}
        \centering
        \includegraphics[width=\textwidth]{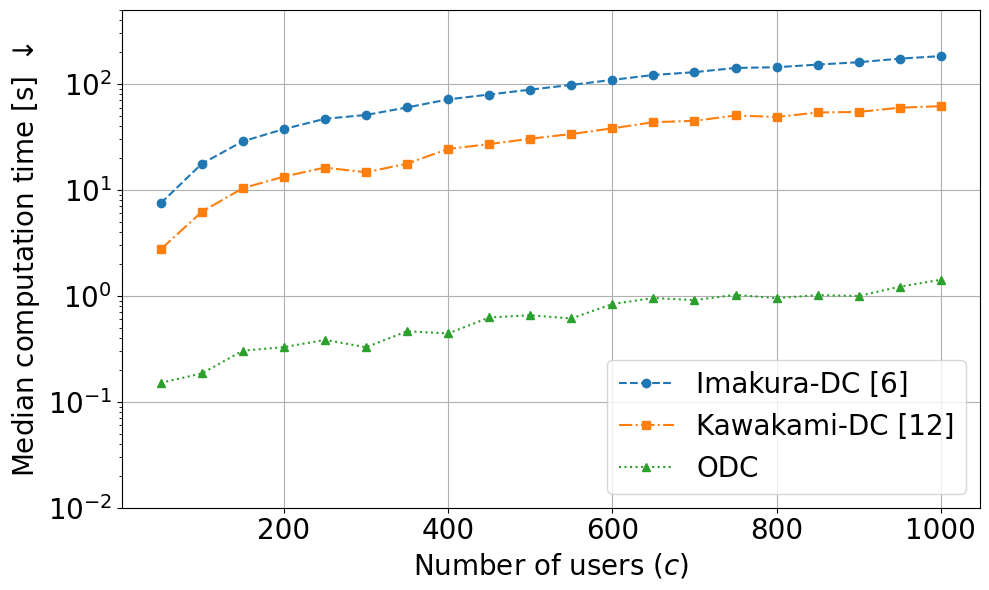}
        \caption{Wall-clock time versus number of users~\(c\)\,\((a,\ell) = (1000,50)\).}
        \label{fig:vary_c_time}
    \end{subfigure}

    \caption{Wall-clock time with varying parameters \((\ell, a, c)\).}
    \label{fig:time_scaling}
\end{figure}

\begin{table}[h!]
    \centering
    \caption{Design of the efficiency experiment. Two parameters are held fixed while the third is swept over the indicated range.}
    \label{tab: exp-design-eff}
    \begin{tabular}{@{}lcc@{}}
        \toprule
        \textbf{Figure} &
        \textbf{Free parameter (fixed pair)} &
        \textbf{Range}\\
        \midrule
        Fig.~\ref{fig:time_scaling}(a) &
        dimension \(\ell\;\;(a,c)=(1000,50)\) &
        \(50\!:\!50\!:\!950\) \\[2pt]
        Fig.~\ref{fig:time_scaling}(b) &
        anchor size \(a\;\;(\ell,c)=(50,50)\) &
        \(1000\!:\!1000\!:\!20000\) \\[2pt]
        Fig.~\ref{fig:time_scaling}(c) &
        users \(c\;\;(a,\ell)=(1000,50)\) &
        \(50\!:\!50\!:\!1000\) \\
        \bottomrule
    \end{tabular}
\end{table}

% ---------------------------------------------------------------
\subsection{Scaling with Latent Dimension \texorpdfstring{\(\ell\)}{l}}
\label{subsec:scaling-l}

Fig.~\ref{fig:time_scaling}(a) indicates an approximate power-law scaling of computation time with respect to the dimension \(\ell\). We applied ordinary least squares (OLS) regression to the linear model:
\begin{equation}
\log_{10}(\mathrm{time}) = \kappa + \alpha\,\log_{10}(\ell),
\end{equation}
with the constant \(\kappa\) to empirically estimate the exponent \(\alpha\). The resulting empirical estimates \(\widehat{\alpha}_{\ell}\) are:
\begin{equation}
\widehat{\alpha}_{\ell} =
\begin{cases}
1.26, & \text{Imakura--DC}\\[2pt]
1.75, & \text{Kawakami--DC}\\[2pt]
2.05, & \text{ODC}.
\end{cases}
\end{equation}

For these log--log fits, the coefficients of determination are \(R^2 \approx 0.97\), \(0.99\), and \(0.99\) for Imakura--DC, Kawakami--DC, and ODC, respectively. The residuals in \(\log_{10}(\mathrm{time})\) are small, show no discernible trend with \(\log_{10}(\ell)\), and are approximately symmetric about zero, indicating that a simple power-law model in \(\ell\) is adequate over the explored range.

The observed exponent \(\alpha \approx 2\) for ODC aligns closely with its theoretically derived complexity \(O(a c \ell^{2})\). Conversely, contemporary DC methods exhibit somewhat smaller empirical exponents, reflecting complexity consistent with their theoretical predictions of \(O\bigl(\min\{a(c \ell)^{2},\,a^{2} c \ell\}\bigr)\). Although ODC has a slightly higher exponent with respect to \(\ell\), its absolute computational cost remains substantially lower across the tested range of dimensions. This practical efficiency arises primarily from ODC avoiding explicit construction and manipulation of the large, dense \(a\times c \ell\) concatenated matrix, thereby circumventing costly large-scale singular value decompositions. Moreover, typical practical scenarios involve \(a \gg \ell\); thus, the substantial scaling advantage in \(a\) (discussed further in \S~\ref{subsec:scaling-a}) renders the modestly larger exponent in \(\ell\) negligible in realistic deployment.

% ------------------------------------------------------------------
\subsection{Scaling with Anchor Size \texorpdfstring{\(a\)}{a}}
\label{subsec:scaling-a}

In Fig.~\ref{fig:time_scaling}(b), the relationship between computational wall-clock time and anchor size \(a\) again exhibits an approximate power-law scaling. Applying OLS regression to the log-log relationship \(\log_{10}(\text{time}) = \kappa + \alpha \log_{10}(a)\) yields the following empirical estimates \(\widehat{\alpha}_{a}\):
\begin{equation}
\widehat{\alpha}_{a} =
\begin{cases}
0.79 & \text{Imakura--DC},\\[2pt]
0.92 & \text{Kawakami--DC},\\[2pt]
0.56 & \text{ODC}.
\end{cases}
\end{equation}

The power-law model in \(a\) again provides a good description of the data, with \(R^2 \approx 0.99\), \(0.998\), and \(0.87\) for Imakura--DC, Kawakami--DC, and ODC, respectively. Residuals in \(\log_{10}(\text{time})\) show no clear heteroscedasticity or trend with \(\log_{10}(a)\); for the two baselines they are tightly concentrated around zero, while for ODC they are somewhat more dispersed but remain roughly symmetric, indicating that deviations from an exact power law are modest.

Notably, all empirically observed slopes are lower than their corresponding theoretical predictions (\(\alpha=1\) for ODC and \(\alpha\in\{1,2\}\) for baseline DC methods). The significantly lower slope of \(0.56\) observed for ODC is particularly noteworthy. This empirical observation can be attributed to ODC's dominant computational workload comprising matrix multiplications---specifically, performing \(c\) independent multiplications of relatively small \(\ell\times a\) and \(a\times \ell\) matrices---which are highly optimized and efficiently executed in modern computational environments, thus performing substantially better in practice than theoretically anticipated.

At the maximum tested anchor size \(a = 20\,000\), median wall-clock runtimes are approximately \(48.5\,\mathrm{s}\) for Imakura--DC, \(50.4\,\mathrm{s}\) for Kawakami--DC, and \textbf{\(0.47\,\mathrm{s}\)} for ODC, corresponding to empirical speed-up factors of about \(104\times\) and \(108\times\) in favor of the proposed ODC method—more than two orders of magnitude in this regime.

% ------------------------------------------------------------------
\subsection{Scaling with the Number of Users \texorpdfstring{\(c\)}{c}}
\label{subsec:scaling-c}

Fig.~\ref{fig:time_scaling}(c) examines computational scaling with respect to the number of users \(c\), holding the dimensions \((a, \ell) = (1000, 50)\) fixed. Applying OLS regression to the log-log relationship
\(
\log_{10}(\text{time}) = \kappa + \alpha \log_{10}(c),
\)
we obtain the following empirical estimates \(\widehat{\alpha}_{c}\):
\begin{equation}
\widehat{\alpha}_{c} =
\begin{cases}
0.95 & \text{Imakura--DC},\\[2pt]
0.92 & \text{Kawakami--DC},\\[2pt]
0.84 & \text{ODC}.
\end{cases}
\end{equation}

These log--log regressions also exhibit high goodness-of-fit, with \(R^2 \approx 0.99\), \(0.98\), and \(0.94\) for Imakura--DC, Kawakami--DC, and ODC, respectively. The residuals in \(\log_{10}(\text{time})\) are small in magnitude, roughly homoscedastic in \(\log_{10}(c)\), and display only mild asymmetry, suggesting that a linear power-law dependence on \(c\) is an adequate description over the explored range.

For ODC, the theoretical complexity is \(O(a c \ell^2)\), predicting a slope \(\alpha \approx 1\). The slightly lower empirical slope of \(0.84\) is primarily due to efficient parallelization and vectorization in modern CPU architectures. Conversely, the theoretical complexity for Imakura--DC and Kawakami--DC methods is \(O(\min\{a(c\ell)^2,\, a^2c\ell\})\). Under the tested dimensions, the \(a^2 c \ell\) term dominates for \(c > 20\), predicting linear scaling in \(c\). The empirical slopes of \(0.95\) and \(0.92\) closely align with these theoretical predictions.

To connect this scaling to the user-facing notion of \emph{incremental user latency}, we further analyzed the same experiment on a linear scale by regressing wall-clock time \(T(c)\) against \(c\) via
\begin{equation}
T(c) = \beta_0 + \beta_1 c,
\end{equation}
where \(\beta_1\) represents the average increase in computation time per additional user (seconds/user) on our hardware. For \((a,\ell)=(1000,50)\), this yields:
\begin{equation}
\beta_1 \approx
\begin{cases}
8.99\times 10^{-2}~\text{s/user} & \text{Imakura--DC},\\[2pt]
4.76\times 10^{-2}~\text{s/user} & \text{Kawakami--DC},\\[2pt]
8.67\times 10^{-4}~\text{s/user} & \text{ODC},
\end{cases}
\end{equation}
with 95\% confidence intervals \([0.084, 0.096]\), \([0.041, 0.054]\), and \([0.00066, 0.00107]\)~s/user, respectively, and coefficients of determination \(R^2 \approx 0.98\), \(0.94\), and \(0.83\). Residuals from these linear fits are unsystematic in \(c\) and small relative to the fitted trend, so over the investigated range \(T(c)\) is well-approximated by an affine function of \(c\).

In other words, in this setting, ODC incurs only about \(0.00087\) seconds (0.87~ms) of additional computation per extra user, compared to \(0.090\) and \(0.048\) seconds for Imakura--DC and Kawakami--DC. Thus, ODC achieves roughly \(100\times\) lower incremental user latency than Imakura--DC and \(55\times\) lower than Kawakami--DC. Over blocks of 50 additional users, these slopes correspond to an extra \(4.5\)~s (Imakura--DC), \(2.4\)~s (Kawakami--DC), and only \(0.043\)~s (ODC).

Although asymptotic scaling behaviors are similar, ODC exhibits significantly improved absolute performance. Median runtimes at \(c=1000\) report that ODC completes computations in only \textbf{\(1.0\,\mathrm{s}\)}, compared to approximately \(98\,\mathrm{s}\) and \(53\,\mathrm{s}\) required by Imakura--DC and Kawakami--DC, respectively. Thus, ODC achieves empirical speed-ups of at least \(50\times\).

This substantial performance gap arises because ODC employs optimized batched matrix multiplications and small-scale (\(\ell\times \ell\)) singular value decompositions, while baseline methods incrementally extend randomized SVD computations with each additional user. In practice, ODC maintains nearly constant incremental computational overhead per user, rendering the basis-alignment phase computationally negligible in realistic deployments.

\section{Empirics Under Relaxed Assumptions}
\label{sec: assumptions-evaluation}

Our theoretical analysis in \S~\ref{section: ODC} relies on two key assumptions regarding the secret bases \(\matr{F}_i\):

\begin{enumerate}
    \item \textbf{Identical Span}: all $\matr{F}_i$ share an identical column space (\textbf{Assumption}~\ref{assumptions: ODC}-2);
    \item \textbf{Orthonormality}: each $\matr{F}_i$ has orthonormal columns (\textbf{Assumption}~\ref{assumptions: ODC}-3).
\end{enumerate}

In practical scenarios, these ideal conditions may not strictly hold, particularly when bases are generated through computationally inexpensive random projections or derived from heterogeneous local datasets. We thus empirically quantify the impact of deviations from these assumptions on performance.

\subsection{Experimental Design}  
We evaluate four controlled scenarios with varying adherence to these assumptions (Table~\ref{tab:assumptions-description}). Let $\matr{V}_i\in\mathbb{R}^{m\times\ell}$ denote the top $\ell$ right singular vectors of $\matr{X}_i$.

For each client $i$, we instantiate the secret basis $\matr{F}_i$ as
\begin{equation}
\matr{F}_i \;=\;
\begin{cases}
\matr{V}_1\,\matr{E}_i, & \textbf{SameSpan-Orth}:~\matr{E}_i \in \mathcal{O}(\ell)\ \text{(Haar)};\\[2pt]
\matr{V}_1\,\matr{E}_i, & \textbf{SameSpan}:~\matr{E}_i \sim \mathrm{Unif}([0,1))^{\ell\times \ell};\\[2pt]
\matr{V}_i\,\matr{E}_i, & \textbf{DiffSpan-Orth}:~\matr{E}_i \in \mathcal{O}(\ell)\ \text{(Haar)};\\[2pt]
\matr{V}_i\,\matr{E}_i, & \textbf{DiffSpan}:~\matr{E}_i \sim \mathrm{Unif}([0,1))^{\ell\times \ell},
\end{cases}
\end{equation}
where $\mathrm{Unif}([0,1))^{\ell\times\ell}$ denotes an $\ell\times\ell$ matrix with i.i.d.\ entries drawn from $[0,1)$. 
Thus, \textbf{SameSpan} fixes $\operatorname{span}(\matr{F}_i)=\operatorname{span}(\matr{V}_1)$ for all users, whereas \textbf{DiffSpan} allows client-specific spans $\operatorname{span}(\matr{F}_i)=\operatorname{span}(\matr{V}_i)$; \textbf{Orthnormality} is controlled by sampling $\matr{E}_i$ either Haar-uniformly from $\mathcal{O}(\ell)$ or entrywise from $\mathrm{Unif}([0,1))$.

\begin{table}[t]
  \centering
  \caption{Summary of secret-basis conditions evaluated.}
  \begin{tabular}{lcc}
    \toprule
    \textbf{Condition} & \textbf{Identical Span} & \textbf{Orthonormality} \\
    \midrule
    \textbf{SameSpan-Orth} & \checkmark & \checkmark \\
    \textbf{SameSpan}      & \checkmark & $\times$ \\
    \textbf{DiffSpan-Orth} & $\times$   & \checkmark \\
    \textbf{DiffSpan}      & $\times$   & $\times$ \\
    \bottomrule
  \end{tabular}
  \label{tab:assumptions-description}
\end{table}

We group tasks into five domains:

\begin{itemize}[leftmargin=*,itemsep=3pt]

    \item \textbf{Image classification (MNIST, Fashion-MNIST):}
    Images are flattened into $28\times 28$ vectors and normalized to $[0,1]$. We randomly select $10{,}000$ samples for the test set. We report mean classification accuracy for SVM and MLP classifiers, averaged over 100 independent runs. The anchor matrix is drawn uniformly at random with $a=1000$.

    \item \textbf{Biomedical compound classification (TDC~\cite{TDC}):}
    Molecules are featurized as 2048-bit Morgan fingerprints (radius $=2$). Data are split across four users, with each user holding samples from exactly one class (see Fig.~\ref{fig:TDC-splitting}). We use the dataset-provided train/test partitions and evaluate performance using ROC-AUC and PR-AUC (MLP), averaged over 100 runs. The anchor matrix is drawn uniformly at random with $a=3000$.

    \item \textbf{Income classification (Adult~\cite{adult}):}
    We use the Adult income dataset and partition it across 200 users using both horizontal and \emph{vertical} splits. We randomly select $10{,}000$ samples for the test set. Under the vertical split, half of the users possess one subset of features, while the remaining users possess the complementary subset. We report the mean classification accuracy of an MLP classifier over 100 runs. The anchor matrix is drawn uniformly at random with $a=1000$.

    \item \textbf{Facial attribute classification (CelebA~\cite{celeba}):}
    RGB images ($128\times 128$) are flattened and normalized to $[0,1]$. In each iteration, we randomly sample 4000 images from the full dataset and select 1000 for testing. We measure binary gender prediction accuracy using an MLP classifier, averaged over 100 runs, and additionally compare against $(\varepsilon,\delta)$-DP baselines. The anchor matrix is drawn uniformly at random with $a=10000$.

    \item \textbf{Clinical regression (eICU-CRD~\cite{eicu}):}
    We use data from the 50 largest hospitals to predict patient length of stay (days) from 25 routinely collected clinical features. We randomly select 20\% of the data as the test set. Performance is reported as RMSE (lower is better), averaged across hospitals over 100 runs using an MLP model. As a federated learning baseline, we run federated averaging (FedAvg) for 40 rounds with full participation. The anchor matrix is drawn uniformly at random with $a=3000$.

\end{itemize}

We compare a centralized oracle (\textbf{Central}); per‑user local models (\textbf{Local}); two contemporary DC baselines \textbf{Imakura‑DC} (\textbf{Algorithm}~\ref{alg: Imakura-DC}), \textbf{Kawakami‑DC} (\textbf{Algorithm}~\ref{alg: Kawakami-DC}); and the proposed \textbf{ODC} (\textbf{Algorithm}~\ref{alg: ODC}). For CelebA, we additionally include additive Gaussian \textbf{DP}~\cite{analyticalgaussianmechanism}, and for eICU, we include \textbf{FedAvg}~\cite{FL1}, to place ODC in the wider PPML context.

The downstream models are configured as follows. For the SVM baseline, we use an RBF kernel with \(C = 1.0\) and \(\gamma = \text{"scale"}\). The MLP baseline consists of a single hidden layer with 256 ReLU units, trained with Adam (\texttt{solver = "adam"}) using a mini-batch size of 32 and a maximum of 1000 training iterations, with early stopping enabled.

All other experimental settings (e.g., hardware, software, and implementation details) are identical to those described in \S~\ref{sec:efficiency-eval}, unless otherwise stated.

All numerical results reported in this paper, including all baselines, were produced by the authors by executing each method on the \emph{same} dataset partitions described above. We do not copy performance numbers from prior publications. To ensure fair comparisons, we keep the latent dimension $\ell$ and anchor size $a$ fixed across DC/ODC variants, and we use the same downstream-model hyperparameters described above for all compared methods.

\subsection{Results and Discussion}
\label{subsec:results-discussion}

\begin{table*}[htbp]
\centering
\caption{Performance comparison of ODC and baseline DC methods across multiple datasets (MNIST, Fashion-MNIST, and biomedical TDC datasets) using various classifiers and performance metrics. The reported results represent the mean scores ($\pm$ margin of error at 95\% confidence, computed over 100 independent trials) under the four distinct secret-base conditions described in Table~\ref{tab:assumptions-description}.}
\label{tab: assumptions-all}
\resizebox{\textwidth}{!}{
    \begin{tabular}{llcccccccccc}
    \toprule
    \multirow{2}{*}{Dataset} 
      & \multirow{2}{*}{Secret Bases}
      & \multicolumn{5}{c}{SVM Accuracy [\%] $\uparrow$} 
      & \multicolumn{5}{c}{MLP Accuracy [\%] $\uparrow$} \\
    \cmidrule(lr){3-7}\cmidrule(lr){8-12}
    & & \textbf{Central} & \textbf{Local} & \textbf{Imakura‑DC~\cite{DCframework1}} 
        & \textbf{Kawakami‑DC~\cite{KawakamiDC}} & \textbf{ODC}
      & \textbf{Central} & \textbf{Local} & \textbf{Imakura‑DC~\cite{DCframework1}} 
        & \textbf{Kawakami‑DC~\cite{KawakamiDC}} & \textbf{ODC} \\
    \midrule

    \multirow{4}{*}{MNIST~\cite{MNIST}}
      & \textbf{SameSpan‑Orth}
        & \multirow{4}{*}{96.0$\pm$0.1} & \multirow{4}{*}{66.0$\pm$0.9}
        & \textbf{\boldmath 96.1$\pm$0.1} & \textbf{\boldmath 96.1$\pm$0.1} & 95.6$\pm$0.1
        & \multirow{4}{*}{95.1$\pm$0.1} & \multirow{4}{*}{59.7$\pm$1.4}
        & \textbf{\boldmath 96.0$\pm$0.1} & 95.9$\pm$0.1 & 95.9$\pm$0.1 \\

      & \textbf{SameSpan}
        &  & 
        & \textbf{\boldmath 96.1$\pm$0.1} & \textbf{\boldmath 96.1$\pm$0.1} & 88.8$\pm$0.2
        &  & 
        & \textbf{\boldmath 95.9$\pm$0.1} & 95.8$\pm$0.1 & 92.6$\pm$0.2 \\

      & \textbf{DiffSpan‑Orth}
        &  & 
        & \textbf{\boldmath 95.1$\pm$0.1} & 94.7$\pm$0.1 & 94.9$\pm$0.1
        &  & 
        & 94.4$\pm$0.2 & 93.2$\pm$0.2 & \textbf{\boldmath 95.3$\pm$0.1} \\

      & \textbf{DiffSpan}
        &  & 
        & \textbf{\boldmath 94.9$\pm$0.1} & 94.7$\pm$0.1 & 86.5$\pm$0.3
        &  & 
        & \textbf{\boldmath 94.9$\pm$0.1} & 93.2$\pm$0.2 & 91.2$\pm$0.2 \\\midrule

    \multirow{4}{*}{Fashion‑MNIST~\cite{fashionmnist}}
      & \textbf{SameSpan‑Orth}
        & \multirow{4}{*}{86.1$\pm$0.2} & \multirow{4}{*}{62.9$\pm$0.8}
        & \textbf{\boldmath 85.4$\pm$0.2} & \textbf{\boldmath 85.4$\pm$0.2} & 83.7$\pm$0.3
        & \multirow{4}{*}{86.1$\pm$0.2} & \multirow{4}{*}{61.1$\pm$1.0}
        & \textbf{\boldmath 86.1$\pm$0.2} & 86.0$\pm$0.2 & 86.0$\pm$0.2 \\

      & \textbf{SameSpan}
        &  & 
        & \textbf{\boldmath 85.4$\pm$0.2} & \textbf{\boldmath 85.4$\pm$0.2} & 77.7$\pm$0.3
        &  & 
        & \textbf{\boldmath 86.1$\pm$0.2} & 85.9$\pm$0.2 & 80.5$\pm$0.3 \\

      & \textbf{DiffSpan‑Orth}
        &  & 
        & \textbf{\boldmath 83.4$\pm$0.2} & 83.0$\pm$0.3 & 82.5$\pm$0.2
        &  & 
        & 81.8$\pm$0.3 & 81.5$\pm$0.3 & \textbf{\boldmath 84.3$\pm$0.2} \\

      & \textbf{DiffSpan}
        &  & 
        & 82.9$\pm$0.2 & \textbf{\boldmath 83.0$\pm$0.3} & 76.3$\pm$0.3
        &  & 
        & \textbf{\boldmath 82.5$\pm$0.3} & 81.5$\pm$0.3 & 78.8$\pm$0.4 \\

    \bottomrule
  \end{tabular}
  }
  
  \bigskip
  
\resizebox{\textwidth}{!}{
  \begin{tabular}{llcccccccc}
    \toprule
    \multirow{2}{*}{Dataset} & \multirow{2}{*}{Secret Bases}
      & \multicolumn{4}{c}{ROC‑AUC [\%] $\uparrow$} 
      & \multicolumn{4}{c}{PR‑AUC [\%] $\uparrow$} \\
    \cmidrule(lr){3-6}\cmidrule(lr){7-10}
    & & \textbf{Central} & \textbf{Imakura‑DC~\cite{DCframework1}} & \textbf{Kawakami‑DC~\cite{KawakamiDC}} & \textbf{ODC} & \textbf{Central}              & \textbf{Imakura‑DC~\cite{DCframework1}} & \textbf{Kawakami‑DC~\cite{KawakamiDC}} & \textbf{ODC} \\
    \midrule
    \multirow{4}{*}{AMES~\cite{AMES}}
      & \textbf{SameSpan‑Orth}
        & \multirow{4}{*}{87.1$\pm$0.0} & 86.2$\pm$0.4 & 82.4$\pm$0.3 & \textbf{\boldmath 88.6$\pm$0.1} & \multirow{4}{*}{88.7$\pm$0.0} & 87.9$\pm$0.4        & 83.8$\pm$0.4         & \textbf{\boldmath 89.9$\pm$0.1} \\
      & \textbf{SameSpan}
        &                               & \textbf{\boldmath 86.6$\pm$0.3} & 82.2$\pm$0.3 & 59.9$\pm$0.4 &                               & \textbf{\boldmath 88.2$\pm$0.3}        & 83.7$\pm$0.3         & 62.5$\pm$0.4 \\
      & \textbf{DiffSpan‑Orth}
        &                               & 63.1$\pm$0.4 & 65.4$\pm$0.4 & \textbf{\boldmath 67.9$\pm$0.3} &                               & 66.9$\pm$0.4        & 69.1$\pm$0.4         & \textbf{\boldmath 71.5$\pm$0.3} \\
      & \textbf{DiffSpan}
        &                               & 61.9$\pm$0.3 & \textbf{\boldmath 65.2$\pm$0.4} & 58.1$\pm$0.4 &                               & 65.6$\pm$0.3        & \textbf{\boldmath 69.0$\pm$0.4}         & 62.0$\pm$0.4 \\
    \midrule
    \multirow{4}{*}{Tox21\_SR-ARE~\cite{Tox21}}
      & \textbf{SameSpan-Orth}
        & \multirow{4}{*}{75.0$\pm$0.0}
        & 69.4$\pm$1.0   & 64.5$\pm$1.3   & \textbf{\boldmath 75.4$\pm$0.3}
        & \multirow{4}{*}{42.2$\pm$0.0}
        & 30.8$\pm$1.1   & 25.2$\pm$1.3   & \textbf{\boldmath 38.4$\pm$0.6} \\
      & \textbf{SameSpan}
        & 
        & \textbf{\boldmath 69.0$\pm$1.0}   & 64.5$\pm$1.3   & 54.0$\pm$0.3
        & 
        & \textbf{\boldmath 30.4$\pm$1.1}   & 25.2$\pm$1.2   & 16.7$\pm$0.2 \\
      & \textbf{DiffSpan-Orth}
        & 
        & 57.4$\pm$0.3   & 58.3$\pm$0.3   & \textbf{\boldmath 61.0$\pm$0.3}
        & 
        & 19.4$\pm$0.2   & 20.2$\pm$0.2   & \textbf{\boldmath 22.6$\pm$0.3} \\
      & \textbf{DiffSpan}
        & 
        & 56.6$\pm$0.3   & \textbf{\boldmath 58.5$\pm$0.3}   & 53.9$\pm$0.3
        & 
        & 18.9$\pm$0.2   & \textbf{\boldmath 20.3$\pm$0.2}   & 17.3$\pm$0.2 \\

        \midrule
    \multirow{4}{*}{HIV~\cite{HIV}}
      & \textbf{SameSpan-Orth}
        & \multirow{4}{*}{78.8$\pm$0.0}
        & 79.6$\pm$0.4   & 78.1$\pm$0.5   & \textbf{\boldmath 80.8$\pm$0.2}
        & \multirow{4}{*}{42.0$\pm$0.0}
        & 41.6$\pm$0.5   & 39.1$\pm$0.7   & \textbf{\boldmath 41.8$\pm$0.4} \\
      & \textbf{SameSpan}
        & 
        & \textbf{\boldmath 80.1$\pm$0.3}   & 78.0$\pm$0.5   & 58.8$\pm$0.4
        & 
        & \textbf{\boldmath 42.2$\pm$0.4}   & 38.9$\pm$0.8   &  5.3$\pm$0.1 \\
      & \textbf{DiffSpan-Orth}
        & 
        & 61.1$\pm$0.4   & 60.9$\pm$0.4   & \textbf{\boldmath 67.4$\pm$0.3}
        & 
        &  7.8$\pm$0.2   &  8.7$\pm$0.2   & \textbf{\boldmath 13.4$\pm$0.3} \\
      & \textbf{DiffSpan}
        & 
        & 58.1$\pm$0.4   & \textbf{\boldmath 60.8$\pm$0.4}   & 56.3$\pm$0.4
        & 
        &  6.6$\pm$0.2   & \textbf{\boldmath  8.5$\pm$0.3}   &  4.6$\pm$0.1 \\

        \midrule
    \multirow{4}{*}{CYP3A4~\cite{CYP2D6}}
      & \textbf{SameSpan-Orth}
        & \multirow{4}{*}{87.2$\pm$0.0}
        & 84.8$\pm$0.2   & 82.9$\pm$0.2   & \textbf{\boldmath 86.2$\pm$0.1}
        & \multirow{4}{*}{83.5$\pm$0.0}
        & 80.6$\pm$0.3   & 78.1$\pm$0.3   & \textbf{\boldmath 82.1$\pm$0.1} \\
      & \textbf{SameSpan}
        & 
        & \textbf{\boldmath 84.7$\pm$0.2}   & 82.9$\pm$0.2   & 59.2$\pm$0.4
        & 
        & \textbf{\boldmath 80.5$\pm$0.3}   & 78.1$\pm$0.3   & 49.8$\pm$0.5 \\
      & \textbf{DiffSpan-Orth}
        & 
        & 57.8$\pm$0.3   & 59.4$\pm$0.4   & \textbf{\boldmath 65.5$\pm$0.3}
        & 
        & 49.8$\pm$0.3   & 51.5$\pm$0.5   & \textbf{\boldmath 57.2$\pm$0.3} \\
      & \textbf{DiffSpan}
        & 
        & 58.3$\pm$0.3   & \textbf{\boldmath 59.1$\pm$0.4}   & 58.2$\pm$0.3
        & 
        & 50.4$\pm$0.3   & \textbf{\boldmath 51.1$\pm$0.5}   & 48.9$\pm$0.4 \\
    \midrule
    \multirow{4}{*}{CYP2D6~\cite{CYP2D6}}
      & \textbf{SameSpan‑Orth}
        & \multirow{4}{*}{83.5$\pm$0.0}
        & 81.4$\pm$0.3 & 80.4$\pm$0.2 & \textbf{\boldmath 83.3$\pm$0.1}
        & \multirow{4}{*}{63.5$\pm$0.0}
        & 60.2$\pm$0.4 & 58.5$\pm$0.3 & \textbf{\boldmath 62.3$\pm$0.2} \\
      & \textbf{SameSpan}
        & 
        & \textbf{\boldmath 81.8$\pm$0.2} & 80.3$\pm$0.2 & 59.5$\pm$0.4
        & 
        & \textbf{\boldmath 60.6$\pm$0.4} & 58.3$\pm$0.3 & 24.2$\pm$0.3 \\
      & \textbf{DiffSpan‑Orth}
        & 
        & 62.4$\pm$0.4 & 64.6$\pm$0.4 & \textbf{\boldmath 65.8$\pm$0.2}
        & 
        & 27.5$\pm$0.5 & 30.3$\pm$0.5 & \textbf{\boldmath 31.2$\pm$0.3} \\
      & \textbf{DiffSpan}
        & 
        & 62.6$\pm$0.3 & \textbf{\boldmath 64.7$\pm$0.4} & 57.1$\pm$0.4
        & 
        & 27.9$\pm$0.4 & \textbf{\boldmath 30.4$\pm$0.5} & 23.0$\pm$0.3 \\
    \midrule
    \multirow{4}{*}{CYP1A2~\cite{CYP2D6}}
      & \textbf{SameSpan‑Orth}
        & \multirow{4}{*}{91.2$\pm$0.0}
        & 89.7$\pm$0.1 & 89.5$\pm$0.2 & \textbf{\boldmath 90.4$\pm$0.2}
        & \multirow{4}{*}{90.1$\pm$0.0}
        & 88.8$\pm$0.1 & 88.4$\pm$0.3 & \textbf{\boldmath 89.8$\pm$0.2} \\
      & \textbf{SameSpan}
        & 
        & \textbf{\boldmath 90.3$\pm$0.1} & 89.2$\pm$0.2 & 65.1$\pm$0.5
        & 
        & \textbf{\boldmath 89.4$\pm$0.1} & 88.1$\pm$0.3 & 59.4$\pm$0.4 \\
      & \textbf{DiffSpan‑Orth}
        & 
        & 67.3$\pm$0.4 & \textbf{\boldmath 68.6$\pm$0.4} & 67.9$\pm$0.3
        & 
        & 62.8$\pm$0.5 & \textbf{\boldmath 64.9$\pm$0.5} & 63.9$\pm$0.3 \\
      & \textbf{DiffSpan}
        & 
        & 67.8$\pm$0.4 & \textbf{\boldmath 68.7$\pm$0.4} & 60.2$\pm$0.5
        & 
        & 63.3$\pm$0.4 & \textbf{\boldmath 64.9$\pm$0.5} & 55.9$\pm$0.5 \\
    \bottomrule
\end{tabular}
}
\end{table*}

\begin{table*}[htbp]
    \centering
    \caption{Performance comparison of ODC and baseline DC methods on the Adult dataset using an MLP classifier. The reported results represent the mean accuracy ($\pm$ margin of error at 95\% confidence, computed over 100 independent trials, higher is better) under the four distinct secret-base conditions described in Table~\ref{tab:assumptions-description}. }
    \label{tab: adult acc}
    \begin{tabular}{lccccc}
\toprule
\textbf{Secret Bases} & 
\textbf{Central} & 
\textbf{Local} & 
\textbf{Imakura‑DC~\cite{DCframework1}} & 
\textbf{Kawakami‑DC~\cite{KawakamiDC}} & 
\textbf{ODC} \\
\midrule
\textbf{SameSpan-Orth}
  & \multirow{4}{*}{84.9$\pm$0.2}
  & \multirow{4}{*}{74.2$\pm$1.0}
  & 84.9$\pm$0.3
  & 84.8$\pm$0.3
  & \textbf{85.0$\pm$0.2} \\
\textbf{SameSpan}
  &
  &
  & \textbf{84.9$\pm$0.3}
  & 84.8$\pm$0.2
  & 83.0$\pm$0.3 \\
\textbf{DiffSpan-Orth}
  &
  &
  & \textbf{84.5$\pm$0.3}
  & 84.4$\pm$0.3
  & \textbf{84.5$\pm$0.3} \\
\textbf{DiffSpan}
  &
  &
  & 84.3$\pm$0.3
  & \textbf{84.4$\pm$0.3}
  & 82.4$\pm$0.3 \\
\bottomrule
\end{tabular}
\end{table*}

\begin{table*}[htbp]
\centering
\caption{Comparison of classification accuracy on the CelebA dataset using an MLP classifier across different PPML approaches. The reported results represent the mean accuracy ($\pm$ margin of error at 95\% confidence, computed over 100 independent trials) under the four distinct secret-base conditions described in Table~\ref{tab:assumptions-description}. Compared methods include models trained using \textbf{DP}-based additive Gaussian noise at three privacy budget levels ($\varepsilon = 0.5, 2, 8$ and $\delta = 10^{-3}$ fixed).}
\label{tab: CelebA accuracy}
\resizebox{\textwidth}{!}{
\begin{tabular}{lcccccccc}
\toprule
Secret Bases 
  & \textbf{Central} 
  & \textbf{Local} 
  & \textbf{DP~\cite{analyticalgaussianmechanism} ($\varepsilon=0.5$)} 
  & \textbf{DP~\cite{analyticalgaussianmechanism} ($\varepsilon=2$)} 
  & \textbf{DP~\cite{analyticalgaussianmechanism} ($\varepsilon=8$)} 
  & \textbf{Imakura-DC~\cite{DCframework1}} 
  & \textbf{Kawakami-DC~\cite{KawakamiDC}} 
  & \textbf{ODC} \\
\midrule
\textbf{SameSpan-Orth}
  & \multirow{4}{*}{88.4$\pm$0.2} 
  & \multirow{4}{*}{75.7$\pm$0.6}  
  & 69.1$\pm$0.5  
  & 79.5$\pm$0.3  
  & 83.8$\pm$0.3  
  & \textbf{86.5$\pm$0.2} 
  & 86.2$\pm$0.2 
  & 86.2$\pm$0.2 \\

\textbf{SameSpan}
  &                            
  &                             
  & 69.1$\pm$0.5  
  & 79.5$\pm$0.3  
  & 83.8$\pm$0.3  
  & \textbf{86.3$\pm$0.2} 
  & 85.8$\pm$0.3 
  & 79.9$\pm$0.3 \\

\textbf{DiffSpan-Orth}
  &                            
  &                             
  & 69.1$\pm$0.5  
  & 79.5$\pm$0.3  
  & \textbf{83.8$\pm$0.3}  
  & 82.7$\pm$0.2 
  & 82.8$\pm$0.3 
  & 83.6$\pm$0.3 \\

\textbf{DiffSpan}
  &                            
  &                             
  & 69.1$\pm$0.5  
  & 79.5$\pm$0.3  
  & \textbf{83.8$\pm$0.3}  
  & 82.6$\pm$0.2 
  & 82.8$\pm$0.3 
  & 78.6$\pm$0.3 \\
\bottomrule
\end{tabular}
}
\end{table*}

\begin{figure}[htbp]
    \centering
    \includegraphics{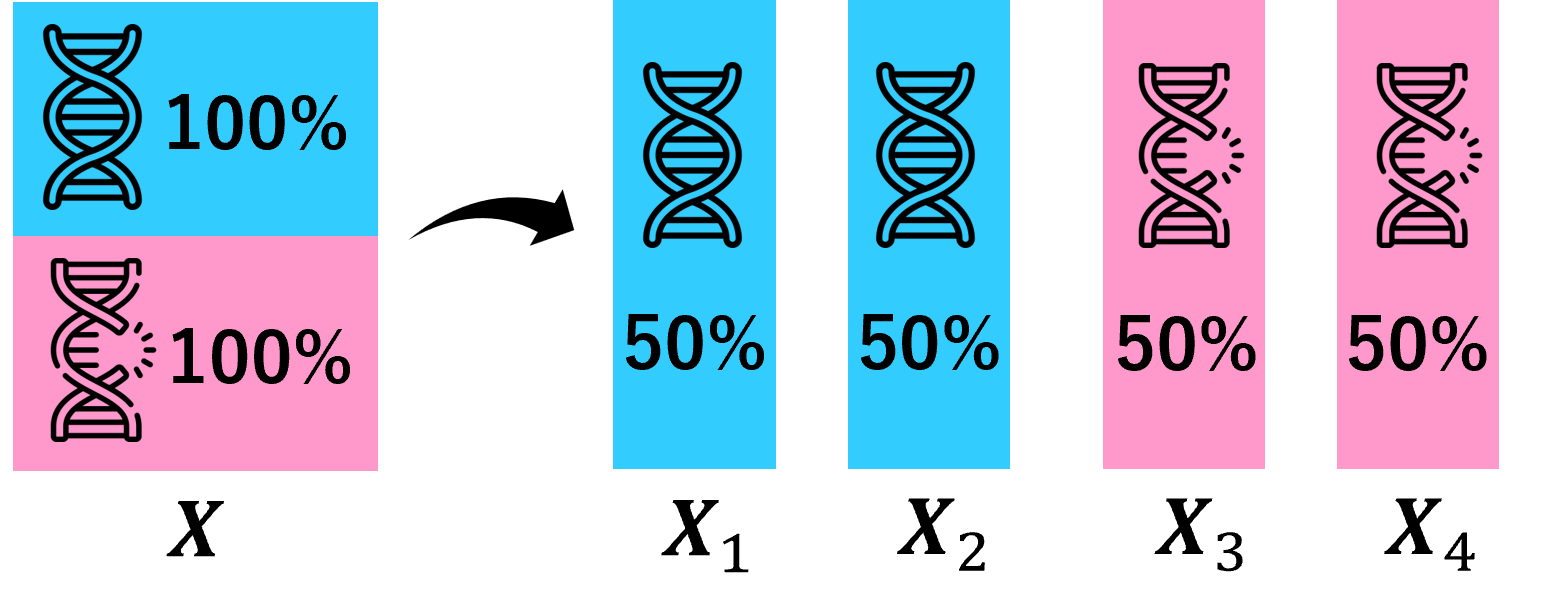}
    \caption{Illustration of extremely heterogeneous splitting applied to TDC datasets. Binary-labeled data are partitioned across four users, each of whom exclusively holds samples of a single label.}
    \label{fig:TDC-splitting}
\end{figure}

\begin{table}[t]
\centering
\caption{Face identifiability on CelebA measured by face-verification ROC-AUC using pretrained FaceNet embeddings (higher $\Rightarrow$ more identifiable; $0.5$ = chance).}
\label{tab: celeba-identifiability}
\begin{tabular}{lc}
\toprule
Representation & Verification AUC $\downarrow$\\
\midrule
Original & 0.977 \\
Noise-added images with \((\varepsilon,\delta)\)-DP~\cite{analyticalgaussianmechanism}, \((\varepsilon,\delta)=(8,10^{-3})\) & 0.671 \\
Noise-added images with \((\varepsilon,\delta)\)-DP~\cite{analyticalgaussianmechanism}, \((\varepsilon,\delta)=(2,10^{-3})\) & 0.618 \\
Noise-added images with \((\varepsilon,\delta)\)-DP~\cite{analyticalgaussianmechanism}, \((\varepsilon,\delta)=(0.5,10^{-3})\)  & 0.522 \\
Orthogonally projected images (\(\ell=300\))   & 0.463 \\
Randomly projected images (\(\ell=300\))       & 0.485 \\
\bottomrule
\end{tabular}
\end{table}

\begin{figure}
    \centering

    \begin{subfigure}[b]{0.9\textwidth}
        \centering
        \includegraphics[width=\textwidth]{Original.png}
        \caption{Original images}
        \label{fig:orig}
    \end{subfigure}

    \begin{subfigure}[b]{0.9\textwidth}
        \centering
        \includegraphics[width=\textwidth]{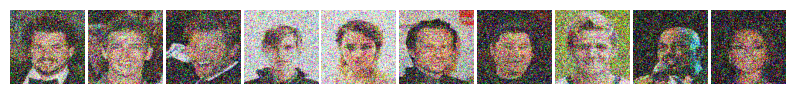}
        \caption{Noise-added images with \((\varepsilon,\delta)\)-DP, \((\varepsilon,\delta)=(8,10^{-3})\)}
        \label{fig:dp8}
    \end{subfigure}

    \begin{subfigure}[b]{0.9\textwidth}
        \centering
        \includegraphics[width=\textwidth]{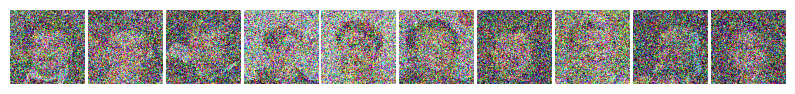}
        \caption{Noise-added images with \((\varepsilon,\delta)\)-DP, \((\varepsilon,\delta)=(2,10^{-3})\)}
        \label{fig:dp2}
    \end{subfigure}

    \begin{subfigure}[b]{0.9\textwidth}
        \centering
        \includegraphics[width=\textwidth]{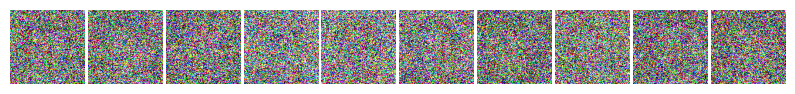}
        \caption{Noise-added images with \((\varepsilon,\delta)\)-DP, \((\varepsilon,\delta)=(0.5,10^{-3})\)}
        \label{fig:dp05}
    \end{subfigure}

    \begin{subfigure}[b]{0.9\textwidth}
        \centering
        \includegraphics[width=\textwidth]{SVD_reduced_dim.png}
        \caption{Orthogonally projected images (\(\ell = 300\))}
        \label{fig:svdproj}
    \end{subfigure}

    \begin{subfigure}[b]{0.9\textwidth}
        \centering
        \includegraphics[width=\textwidth]{Random_reduced_dim.png}
        \caption{Randomly projected images (\(\ell = 300\))}
        \label{fig:randproj}
    \end{subfigure}

    \caption{Visual privacy analysis using CelebA.}
    \label{fig:CelebA_privacy}
\end{figure}

\begin{table*}[t]
\centering
\caption{Comparison of RMSE on the eICU-CRD dataset using an MLP classifier across different PPML approaches. The reported results represent the mean RMSE ($\pm$ margin of error at 95\% confidence, computed over 100 independent, lower is better) under the four distinct secret-base conditions described in Table~\ref{tab:assumptions-description}. Compared methods include: a \textbf{FedAvg} model with full participation.}
\resizebox{\textwidth}{!}{
\begin{tabular}{lcccccc}
\toprule
Secret Bases 
  & \textbf{Central} 
  & \textbf{Local} 
  & \textbf{FedAvg~\cite{FL1}} 
  & \textbf{Imakura-DC~\cite{DCframework1}} 
  & \textbf{Kawakami-DC~\cite{KawakamiDC}} 
  & \textbf{ODC} \\
\midrule
\textbf{SameSpan-Orth}
  & \multirow{4}{*}{3.996$\pm$0.046}
  & \multirow{4}{*}{4.174$\pm$0.048}
  & \multirow{4}{*}{\textbf{4.064$\pm$0.049}}
  & 4.089$\pm$0.049
  & 4.096$\pm$0.049
  & 4.083$\pm$0.048 \\

\textbf{SameSpan}
  &
  &
  &
  & 4.088$\pm$0.049
  & 4.096$\pm$0.049
  & 4.154$\pm$0.049 \\

\textbf{DiffSpan-Orth}
  &
  &
  &
  & 4.085$\pm$0.049
  & 4.090$\pm$0.049
  & 4.070$\pm$0.048 \\

\textbf{DiffSpan}
  &
  &
  &
  & 4.085$\pm$0.049
  & 4.090$\pm$0.049
  & 4.150$\pm$0.049 \\
\bottomrule
\end{tabular}
}
\label{tab: eiCU rmse}
\end{table*}

Tables~\ref{tab: assumptions-all}--\ref{tab: eiCU rmse} comprehensively summarize model performance under four distinct secret-basis scenarios, while Figure~\ref{fig:CelebA_privacy} provides complementary visual evidence of privacy preservation. The following discussion is organized around three comparative perspectives:

\begin{enumerate}[label=\textbf{\Roman*.}, leftmargin=*]

\item \textbf{ODC \emph{vs.} Existing DC Methods Under Various Secret-Basis Conditions (Tables~\ref{tab: assumptions-all} and \ref{tab: adult acc})}

\begin{itemize}[leftmargin=*, itemsep=3pt]

    \item \textbf{SameSpan-Orth (identical span and orthonormality satisfied):} 
    Under ideal conditions, ODC achieves performance equal to or marginally better than the centralized oracle across multiple datasets (e.g., MNIST MLP accuracy \SI{95.9}{\%}, matching the oracle; see Table~\ref{tab: assumptions-all}). Similar performance from Imakura-DC and Kawakami-DC confirms effective utilization of this ideal scenario by all methods.

    \item \textbf{SameSpan (orthonormality violated):} 
    Dropping the orthonormality constraint significantly impacts ODC's performance—MNIST SVM accuracy drops from \SI{95.6}{\%} to \SI{88.8}{\%} and AMES ROC-AUC from \SI{88.6}{\%} to \SI{59.9}{\%}. Imakura-DC and Kawakami-DC remain largely unaffected, underscoring the theoretical dependence of ODC on orthonormal secret bases.

    \item \textbf{DiffSpan-Orth (orthonormality maintained, identical span violated):} 
    When orthonormality is enforced, ODC demonstrates robustness against subspace misalignment, maintaining accuracy in the range \SI{94.9}{\%}--\SI{95.3}{\%} on MNIST, and often surpassing DC baselines in biomedical tasks (e.g., HIV ROC-AUC \SI{67.4}{\%} compared to Imakura-DC \SI{61.1}{\%}; Table~\ref{tab: assumptions-all}). Thus, exact span alignment is advantageous but not critical, provided orthonormality holds.

    \item \textbf{DiffSpan (both conditions violated):} 
    Violating both conditions negatively affects all methods, with ODC showing the greatest degradation (e.g., MNIST SVM accuracy drops to \SI{86.5}{\%}). This observation highlights the practical necessity of orthonormal bases, particularly in heterogeneous, high-variance settings.

\end{itemize}

Regarding Table~\ref{tab: adult acc}, we observe performance trends mirror those observed under purely horizontal splits, despite the additional vertical data splitting.

\item \textbf{ODC \emph{vs.} DP-based Perturbation (Table~\ref{tab: CelebA accuracy})}

Figure~\ref{fig:CelebA_privacy} provides a qualitative visualization, but visual inspection alone can be subjective. We therefore add a quantitative identifiability metric based on an off-the-shelf face-recognition model: we compute 512D FaceNet embeddings~\cite{facenet} (InceptionResnetV1~\cite{resnet} trained on VGGFace2~\cite{vggface}) and evaluate face verification using ROC-AUC from cosine similarity on same-identity vs.\ different-identity pairs. Here, AUC $\approx 1$ indicates high identifiability, while AUC $\approx 0.5$ corresponds to chance-level identifiability. Table~\ref{tab: celeba-identifiability} reports this metric for the same privatized representations shown in Fig.~\ref{fig:CelebA_privacy}.

Table~\ref{tab: CelebA accuracy} summarizes the \emph{utility} side. Under \textbf{DiffSpan-Orth} (heterogeneous users with orthonormal secret bases but without an identical-span constraint), ODC reaches $83.6\%$ accuracy, substantially exceeding the stricter DP settings ($\varepsilon=0.5$ and $\varepsilon=2$) and closely approaching DP at $\varepsilon=8$. In other words, when orthonormality holds, ODC preserves strong predictive performance even outside the idealized identical-span regime.

Taken together, Fig.~\ref{fig:CelebA_privacy} and Table~\ref{tab: celeba-identifiability} clarify \emph{why} the DP baselines can look visually weak at larger privacy budgets: across the tested $\varepsilon$ values, DP perturbation reduces identifiability but still leaves AUC well above chance ($\approx0.52$–$0.67$), whereas DC‑style projections reduce identifiability toward chance (AUC $\approx0.46$–$0.49$). At the same time, Table~\ref{tab: CelebA accuracy} shows the familiar trade-off: increasing noise for stronger DP-style obfuscation is accompanied by a considerable utility drop, while ODC attains strong obfuscation without requiring a calibrated privacy budget.

Finally, the four secret-base conditions in Table~\ref{tab: CelebA accuracy} expose an important practical constraint: ODC is visibly more sensitive to violations of orthonormality (\textbf{SameSpan} and \textbf{DiffSpan}) than the two DC baselines. This aligns with ODC's design goal (orthogonal concordance) and motivates the choice of orthonormal secret bases in practice (e.g., via PCA/SVD/QR), which is typically feasible in DC pipelines.

Table~\ref{tab: CelebA accuracy} shows that under realistic conditions (DiffSpan--Orth), ODC attains \SI{83.6}{\%} accuracy, surpassing DP at lower privacy budgets by substantial margins (\(+14.5\) pp for \(\varepsilon=0.5\) and \(+4.1\) pp for \(\varepsilon=2\)), while nearly matching DP at \(\varepsilon=8\) (only \(0.2\) pp below). Figure~\ref{fig:CelebA_privacy} further demonstrates that DP at high privacy budgets preserves visually identifiable features, whereas ODC fully obfuscates visual identity. Thus, ODC provides superior privacy--utility trade-offs without explicitly calibrating privacy budgets. 

However, the formal privacy mechanisms underpinning ODC and contemporary DC currently rely on a semi-honest assumption; as a result, these mechanisms remain underdeveloped for scenarios involving stronger or malicious adversaries. Suppose the semi-honest privacy assumption cannot be guaranteed or tolerated. In that case, DP-based additive perturbation may still be preferred, despite the advantageous privacy-performance trade-offs offered by DC methods. This preference arises from the rigorous privacy guarantees inherent in DP.

\item \textbf{ODC \emph{vs.} Federated Learning (Table~\ref{tab: eiCU rmse})}

Table~\ref{tab: eiCU rmse} shows FedAvg achieving the lowest RMSE (\SI{4.060}{}). ODC, under DiffSpan--Orth, is statistically indistinguishable (\SI{4.065}{}, \(0.005\) difference), clearly outperforming Imakura-DC and Kawakami-DC (\(\sim\)\SI{4.086}{}). Given FedAvg’s requirement for iterative communication, ODC offers a strong alternative, achieving near-FedAvg performance with significantly reduced communication overhead. Interestingly, enforcing identical span conditions negatively affects performance for eICU-CRD, indicating that arbitrary subspace selection is detrimental for realistic heterogeneous tasks.

\end{enumerate}

ODC achieves \emph{optimal} performance under ideal conditions (\textbf{SameSpan-Orth}), shows \emph{sensitivity} to orthonormality violations, yet remains \emph{robust} against subspace misalignment provided orthonormality is maintained. Across diverse tasks, ODC consistently matches or exceeds DC baselines and offers privacy–utility trade-offs competitive with those of DP and FL approaches. Table~\ref{tab:ppml_compare} provides a qualitative summary comparing these representative PPML frameworks. The observations discussed above strongly advocate enforcing orthonormal secret bases (via PCA, SVD, or QR) while tolerating moderate span misalignment to effectively handle the heterogeneity of realistic data.

\begin{table*}[ht]
    \centering
    \caption{Qualitative comparison of representative PPML frameworks.}
    \label{tab:ppml_compare}
    \resizebox{\textwidth}{!}{
    \begin{tabular}{@{}llll@{}}
        \toprule
        \textbf{Framework} &
        \textbf{Communication rounds for training} &
        \textbf{Typical accuracy\textsuperscript{*}} &
        \textbf{Privacy guarantee} \\
        \midrule
        DP (additive perturbation) &
        One user $\rightarrow$ analyst &
        Low &
        Rigorous $(\varepsilon,\delta)$–DP \\
        Federated Learning &
        Iterative bidirectional rounds until convergence &
        High &
        No intrinsic formalism\textsuperscript{†} \\
        ODC (this work) &
        One user $\rightarrow$ analyst \newline
        $+$ one user-wide anchor broadcast &
        Medium\textsuperscript{\$} &
        Rigorous under the semi-honest model \\
        \bottomrule
    \end{tabular}
    }
    
    \vspace{0.5ex}
    \raggedright
    \textsuperscript{*}\, Relative rankings derived empirically; specific accuracies vary by task.\\
    \textsuperscript{†}\, FL can incorporate DP externally, often reducing accuracy. \\
    \textsuperscript{\$}\, Performance approaches FedAvg under strict orthonormality.
\end{table*}

%---------------------------------------------------------------------------
\section{Empirics on Concordance}\label{sec: concordance evaluation}
%---------------------------------------------------------------------------

Theoretically, orthogonally concordant change-of-basis matrices are expected to preserve downstream performance in distance-based models. However, in practice, violations of assumptions, finite data, and non-distance-based architectures can introduce performance variability. Here, we empirically evaluate whether orthogonal concordance (ODC) offers a practical advantage over weak concordance (Imakura‑DC).

\subsection{Experimental Design}
The methods evaluated are:

\begin{itemize}
    \item \textbf{Imakura-random}: Change-of-basis matrices constructed via \textbf{Algorithm~\ref{alg: Imakura-DC}}, using a uniformly random matrix for \(\matr{R}\).
    \item \textbf{Imakura-identity}: Change-of-basis matrices constructed via \textbf{Algorithm~\ref{alg: Imakura-DC}}, with \(\matr{R} = \matr{I}\).
    \item \textbf{ODC-random}: Change-of-basis matrices constructed via \textbf{Algorithm~\ref{alg: ODC}}, using a uniformly random orthogonal matrix for \(\matr{O}\) (Haar distribution).
    \item \textbf{ODC-identity}: Change-of-basis matrices constructed via \textbf{Algorithm~\ref{alg: ODC}} with \(\matr{O} = \matr{I}\).
\end{itemize}

We follow the canonical DC protocol with \(c=100\) users, each providing \(n_i=100\) samples. The anchor matrix \(\matr{A} \in \mathbb{R}^{1000\times 784}\) is generated as a uniformly random matrix. 

Experiments are conducted on standard image classification datasets (MNIST~\cite{MNIST}, Fashion-MNIST~\cite{fashionmnist}). For each dataset, we train (i) a Support Vector Machine (SVM), representative of distance-based classifiers, and (ii) a single-hidden-layer MLP with 256 ReLU neurons. Images are flattened into $28\times28$ vectors and normalized to the interval $[0,1]$. The reported performance metrics are the mean testing accuracy along with the 95\% confidence interval obtained using SVM and MLP models across 100 independent runs. To assess statistical significance, we perform paired, one-sided \(t\)-tests at the 5\% level, testing the null hypothesis that the \emph{random} variant performs \emph{no worse} than the identity variant. Results are presented in Tables~\ref{tab:span_concordance_same} and ~\ref{tab:span_concordance_comparison}.

All other experimental settings (e.g., hardware, software, and implementation details) are identical to those described in \S~\ref{sec: assumptions-evaluation}, unless otherwise stated.

\subsection{Results and Discussion}

\begin{table*}[htbp]
  \centering
  \footnotesize
  \caption{
    Evaluation of the practical significance of orthogonal concordance (\textbf{ODC}) compared to weak concordance (\textbf{Imakura})
    under the \textbf{SameSpan-Orth} condition on MNIST and Fashion-MNIST
    with SVM and MLP classifiers.
    The table reports mean accuracies (\%) $\pm$ margin of error (95\% confidence interval, 100 runs each)
    for identity and random matrix choices, together with the accuracy difference between random and identity
    conditions (random minus identity) and the paired Cohen's $d$ for the same comparison.
    Negative $\Delta$ and $d$ indicate that the random matrix underperforms the identity.
    Statistical significance is indicated by $^{*}$ for $p < 0.05$, based on one-sided (left-tailed) paired $t$-tests
    testing the null hypothesis that $\mu_{\mathrm{rand}} \ge \mu_{\mathrm{iden}}$; $^{-}$ denotes non-significant differences.
  }
  \resizebox{\textwidth}{!}{
  \begin{tabular}{lllcccccccc}
    \toprule
    \textbf{Condition}
      & \textbf{Dataset}
      & \textbf{Classifier}
      & \textbf{Imakura-Identity}
      & \textbf{Imakura-Random}
      & \textbf{Imakura-$\Delta$ [\%]}
      & \textbf{Imakura-$d$}
      & \textbf{ODC-Identity}
      & \textbf{ODC-Random}
      & \textbf{ODC-$\Delta$ [\%]}
      & \textbf{ODC-$d$} \\
    \midrule
    \multirow{4}{*}{\textbf{SameSpan-Orth}}
      & MNIST         & SVM  & $96.20\pm0.00$ & $92.79\pm0.12$ & $-3.42^{*}$ & $-5.81$ & $96.00\pm0.00$ & $96.00\pm0.00$ & $+0.00^{-}$ & $+0.00$ \\
      & MNIST         & MLP  & $96.17\pm0.05$ & $95.34\pm0.10$ & $-0.82^{*}$ & $-1.50$ & $96.00\pm0.00$ & $96.41\pm0.05$ & $+0.41^{-}$ & $+1.58$ \\
      & Fashion-MNIST & SVM  & $86.21\pm0.01$ & $82.60\pm0.16$ & $-3.61^{*}$ & $-4.41$ & $84.30\pm0.00$ & $84.31\pm0.00$ & $+0.00^{-}$ & $+0.00$ \\
      & Fashion-MNIST & MLP  & $86.62\pm0.09$ & $85.28\pm0.16$ & $-1.35^{*}$ & $-1.50$ & $85.70\pm0.00$ & $86.47\pm0.12$ & $+0.77^{-}$ & $+1.28$ \\
    \bottomrule
  \end{tabular}
  }
  \label{tab:span_concordance_same}
\end{table*}

\begin{table*}[htbp]
  \centering
  \footnotesize
  \caption{
    Evaluation of concordance under different span conditions (\textbf{SameSpan}, \textbf{DiffSpan}, \textbf{DiffSpan-Orth}), where the theoretical assumptions are partially violated, on MNIST and Fashion-MNIST using SVM and MLP classifiers.
    The table reports mean accuracies (\%) $\pm$ margin of error (95\% confidence interval, 100 runs each)
    for identity and random matrix choices, together with the accuracy difference between random and identity
    conditions (random minus identity) and the paired Cohen's $d$ for the same comparison.
    Negative $\Delta$ and $d$ indicate that the random matrix underperforms the identity.
    Statistical significance is indicated by $^{*}$ for $p < 0.05$, based on one-sided (left-tailed) paired $t$-tests
    testing the null hypothesis that $\mu_{\mathrm{rand}} \ge \mu_{\mathrm{iden}}$; $^{-}$ denotes non-significant differences.
  }
  \resizebox{\textwidth}{!}{
  \begin{tabular}{lllcccccccc}
    \toprule
    \textbf{Condition}
      & \textbf{Dataset}
      & \textbf{Classifier}
      & \textbf{Imakura-Identity}
      & \textbf{Imakura-Random}
      & \textbf{Imakura-$\Delta$ [\%]}
      & \textbf{Imakura-$d$}
      & \textbf{ODC-Identity}
      & \textbf{ODC-Random}
      & \textbf{ODC-$\Delta$ [\%]}
      & \textbf{ODC-$d$} \\
    \midrule
    \multirow{4}{*}{\textbf{SameSpan}}
      & MNIST         & SVM  & $96.20\pm0.00$ & $92.99\pm0.11$ & $-3.21^{*}$ & $-5.70$ & $87.00\pm0.00$ & $87.00\pm0.00$ & $+0.00^{-}$ & $+0.00$ \\
      & MNIST         & MLP  & $96.21\pm0.05$ & $95.40\pm0.09$ & $-0.81^{*}$ & $-1.49$ & $89.30\pm0.00$ & $90.83\pm0.17$ & $+1.53^{-}$ & $+1.81$ \\
      & Fashion-MNIST & SVM  & $86.20\pm0.00$ & $82.77\pm0.14$ & $-3.43^{*}$ & $-4.80$ & $77.80\pm0.00$ & $77.80\pm0.00$ & $+0.00^{-}$ & $+0.00$ \\
      & Fashion-MNIST & MLP  & $86.51\pm0.11$ & $85.24\pm0.16$ & $-1.26^{*}$ & $-1.22$ & $79.10\pm0.00$ & $78.45\pm0.20$ & $-0.65^{*}$ & $-0.63$ \\
    \midrule
    \multirow{4}{*}{\textbf{DiffSpan-Orth}}
      & MNIST         & SVM  & $95.40\pm0.00$ & $91.90\pm0.13$ & $-3.50^{*}$ & $-5.38$ & $95.40\pm0.00$ & $95.40\pm0.00$ & $+0.00^{-}$ & $+0.00$ \\
      & MNIST         & MLP  & $95.55\pm0.07$ & $94.69\pm0.11$ & $-0.86^{*}$ & $-1.33$ & $95.90\pm0.00$ & $95.52\pm0.06$ & $-0.38^{*}$ & $-1.20$ \\
      & Fashion-MNIST & SVM  & $84.30\pm0.00$ & $81.36\pm0.16$ & $-2.94^{*}$ & $-3.63$ & $81.30\pm0.00$ & $81.30\pm0.00$ & $+0.00^{-}$ & $+0.00$ \\
      & Fashion-MNIST & MLP  & $84.37\pm0.10$ & $83.45\pm0.16$ & $-0.92^{*}$ & $-0.98$ & $84.70\pm0.00$ & $83.47\pm0.16$ & $-1.23^{*}$ & $-1.53$ \\
    \midrule
    \multirow{4}{*}{\textbf{DiffSpan}}
      & MNIST         & SVM  & $94.38\pm0.01$ & $91.24\pm0.13$ & $-3.13^{*}$ & $-4.65$ & $83.90\pm0.00$ & $83.90\pm0.00$ & $+0.00^{-}$ & $+0.00$ \\
      & MNIST         & MLP  & $94.65\pm0.06$ & $93.76\pm0.13$ & $-0.90^{*}$ & $-1.16$ & $87.50\pm0.00$ & $87.87\pm0.22$ & $+0.37^{-}$ & $+0.32$ \\
      & Fashion-MNIST & SVM  & $83.20\pm0.00$ & $80.62\pm0.15$ & $-2.58^{*}$ & $-3.33$ & $70.50\pm0.00$ & $70.50\pm0.00$ & $+0.00^{-}$ & $+0.00$ \\
      & Fashion-MNIST & MLP  & $83.80\pm0.10$ & $82.61\pm0.19$ & $-1.19^{*}$ & $-1.05$ & $77.80\pm0.00$ & $72.66\pm0.37$ & $-5.15^{*}$ & $-2.69$ \\
    \bottomrule
  \end{tabular}
  }
  \label{tab:span_concordance_comparison}
\end{table*}

Table~\ref{tab:span_concordance_same} reports the results for the \textbf{SameSpan-Orth} setting, in which \textbf{Assumption~\ref{assumptions: DC}} and \textbf{Assumption~\ref{assumptions: ODC}} are simultaneously satisfied, so that both weak and orthogonal concordance hold. In this ideal regime, Imakura's weak-concordance method exhibits a strong dependence on the choice of right factor $\matr{R}$. Replacing the identity ($\matr{R} = \matr{I}$) with a uniformly random invertible matrix consistently reduces accuracy by about $0.8$--$3.6$ percentage points across all four tasks, with large negative effect sizes ($d \le -1.5$) and statistically significant differences (all entries in the Imakura-$\Delta$ column are marked with ${}^{*}$). Thus, even when weak concordance is theoretically guaranteed, the particular target basis has a substantial practical impact on downstream performance.

By contrast, ODC is essentially invariant to the choice of orthogonal matrix~$O$ under the same conditions. For the SVM classifiers, \textbf{ODC-Identity} and \textbf{ODC-Random} yield numerically identical accuracies. For the MLPs, the differences remain within $0.8$ percentage points and are not flagged as significant by the one-sided $t$-tests (all entries carry the $^{-}$ mark). In three out of four cases, the random orthogonal matrix even attains a slightly higher mean accuracy than the identity, although we do not test for the superiority of the random choice. These observations are consistent with \textbf{Theorem~\ref{theorem: orthogonal concordance}}: under \textbf{Assumption~\ref{assumptions: ODC}}, all parties' representations are aligned up to a common orthogonal transform, implying exact invariance for distance-based models, such as SVMs, and near-invariance in practice, even for MLPs.

Table~\ref{tab:span_concordance_comparison} examines how these conclusions change when the secret-basis assumptions are partially violated. For Imakura's method, the qualitative picture does not change: across all three conditions (\textbf{SameSpan}, \textbf{DiffSpan}, \textbf{DiffSpan-Orth}), \textbf{Imakura-Random} underperforms \textbf{Imakura-Identity} by roughly $0.8$--$3.5$ percentage points, with consistently large negative Cohen's $d$ and statistically significant differences. This confirms that the instability of weak concordance with respect to the choice of $\matr{R}$ is intrinsic and persists even when its theoretical assumptions are not met.

For ODC, the behavior depends strongly on which parts of \textbf{Assumption}~\ref{assumptions: ODC} are relaxed. When the bases share an identical span but lose orthonormality (\textbf{SameSpan}), the absolute accuracy of \textbf{ODC-Identity} is substantially degraded compared to the \textbf{SameSpan-Orth} case (drops of about $6$--$7$ percentage points across both datasets and classifiers), highlighting orthonormality as the key structural requirement for ODC. In this regime, the additional effect of choosing a random orthogonal matrix $\matr{O}$ is relatively modest (within about $\pm 1.5$ percentage points) and task-dependent: for MNIST, the random choice slightly improves accuracy, whereas for Fashion-MNIST, it leads to small but statistically significant decreases.

When orthonormality is preserved but the spans differ (\textbf{DiffSpan-Orth}), ODC remains quite stable. For SVMs, \textbf{ODC-Identity} and \textbf{ODC-Random} are equivalent, and for MLPs, the degradation of \textbf{ODC-Random} relative to \textbf{ODC-Identity} is at most about $1.2$ percentage points. This indicates that ODC tolerates moderate subspace misalignment as long as orthonormality is enforced. Finally, when both identical-span and orthonormality assumptions are violated simultaneously (\textbf{DiffSpan}), ODC experiences the largest loss of absolute accuracy and becomes most sensitive to the choice of $\matr{O}$, with differences of up to about $5$ percentage points (Fashion-MNIST, MLP). In this regime, the alignment matrices obtained from the Orthogonal Procrustes step no longer correspond to a common global rotation, so a random orthogonal target can indeed be suboptimal.

Taken together, Tables~\ref{tab:span_concordance_same} and~\ref{tab:span_concordance_comparison} empirically confirm the theoretical claims of orthogonal concordance. When \textbf{Assumption}~\ref{assumptions: ODC} holds, ODC is practically invariant to the choice of the orthogonal target basis, in sharp contrast to the weak-concordance baseline. When the assumptions are relaxed, ODC continues to exhibit smaller and more structured sensitivity to the random matrix than Imakura's method, and its degradation is primarily driven by violations of orthonormality rather than by moderate span mismatches.

\section{Empirics on Anchor Construction}
\label{sec:anchor-construction}

The shared anchor dataset $\matr{A}\in\mathbb{R}^{a\times m}$ is the only object that is common across users, and its intermediate representations $\matr{A}_i=\matr{A}\matr{F}_i$ are the sole inputs used by the analyst to construct the \emph{change-of-basis} matrices $\matr{G}_i$ (\textbf{Algorithm}~\ref{alg: DCframework}). While \S\ref{sec:efficiency-eval} established that ODC's alignment runtime scales favorably with the anchor size $a$, a natural question is whether \emph{how we construct $\matr{A}$} (e.g., synthetic vs.\ public anchors, and the choice of $a$) also affects downstream \emph{utility} and the practical \emph{privacy--efficiency} trade-off.

\subsection{Experimental Design}
Throughout, we keep the secret bases $\matr{F}_i$ and downstream models fixed while varying (i) the anchor size $a$ and (ii) the distribution from which the rows of $\matr{A}$ are sampled. We compare the proposed ODC framework against two existing basis-alignment schemes, Imakura-DC and Kawakami-DC, and report the mean $\pm$95\% confidence interval over repeated runs.

For MNIST and Fashion-MNIST we reuse the setup of \S~\ref{sec: assumptions-evaluation}, draw each row of $\matr{A}$ i.i.d.\ from $\mathrm{Unif}[0,1)$, and vary the anchor size over $a\in\{196,392,784,1568\}$. Table~\ref{tab:anchor-image} reports MLP test accuracy for Imakura-DC, Kawakami-DC, and ODC under the secret-basis conditions \textbf{SameSpan-Orth} and \textbf{DiffSpan-Orth}.

We next investigate how the \emph{source} of the anchor interacts with $a$ on AMES. Following standard DC practice, we consider two constructions of $\matr{A}$: (i) a synthetic anchor whose rows are sampled i.i.d.\ from $\mathrm{Unif}[0,1)$ (``Uniform''), and (ii) a domain-matched anchor~\cite{DCnoniid} obtained by sampling unlabeled molecular fingerprints from PubChem~\cite{PubChem}. For each construction, we vary $a\in\{512,1024,2048,4096\}$ and evaluate ODC, Imakura-DC, and Kawakami-DC under \textbf{SameSpan-Orth} and \textbf{DiffSpan-Orth}. Table~\ref{tab:anchor-ames} reports ROC-AUC and PR-AUC.

All other experimental settings (e.g., hardware, software, and implementation details) are identical to those described in \S~\ref{sec: assumptions-evaluation}, unless otherwise stated.

\subsection{Results and Discussion}

\paragraph{Image classification (MNIST, Fashion-MNIST)}

\begin{table}[t]
  \centering
  \caption{Test accuracy (\%) of DC methods as a function of anchor size $a$ on MNIST and Fashion-MNIST
  using an MLP classifier. We report the mean accuracy $\pm $ 95\% confidence interval over 100 iterations.}
  \label{tab:anchor-image}
  \resizebox{\textwidth}{!}{
  \begin{tabular}{lllcccc}
    \toprule
    Dataset & Condition & Method & $a=196$ & $392$ & $784$ & $1568$ \\
    \midrule
    \multirow{6}{*}{MNIST}
      & \multirow{3}{*}{\textbf{SameSpan-Orth}}
        & Imakura-DC~\cite{DCframework1}   & 95.74$\pm$0.12 & 95.92$\pm$0.14 & 95.65$\pm$0.14 & 95.59$\pm$0.11 \\
      &  & Kawakami-DC~\cite{KawakamiDC} & 94.78$\pm$0.15 & 94.63$\pm$0.21 & 94.11$\pm$0.24 & 93.41$\pm$0.25 \\
      &  & ODC  & \textbf{96.03$\pm$0.13} & \textbf{96.04$\pm$0.14} & \textbf{95.90$\pm$0.12} & \textbf{95.96$\pm$0.12} \\
    \cmidrule(lr){2-7}
      & \multirow{3}{*}{\textbf{DiffSpan-Orth}}
        & Imakura-DC~\cite{DCframework1}   & 93.33$\pm$0.14 & 94.49$\pm$0.17 & 94.61$\pm$0.14 & 94.47$\pm$0.13 \\
      &  & Kawakami-DC~\cite{KawakamiDC} & 91.81$\pm$0.21 & 92.27$\pm$0.27 & 91.22$\pm$0.34 & 90.32$\pm$0.26 \\
      &  & ODC  & \textbf{94.02$\pm$0.15} & \textbf{94.83$\pm$0.15} & \textbf{94.92$\pm$0.14} & \textbf{95.15$\pm$0.12} \\
    \midrule
    \multirow{6}{*}{Fashion-MNIST}
      & \multirow{3}{*}{\textbf{SameSpan-Orth}}
        & Imakura-DC~\cite{DCframework1}   & \textbf{86.08$\pm$0.23} & 86.22$\pm$0.22 & 85.91$\pm$0.22 & 85.70$\pm$0.19 \\
      &  & Kawakami-DC~\cite{KawakamiDC} & 84.32$\pm$0.29 & 83.93$\pm$0.26 & 83.23$\pm$0.31 & 82.69$\pm$0.24 \\
      &  & ODC  & 86.05$\pm$0.23 & \textbf{86.23$\pm$0.20} & \textbf{86.35$\pm$0.23} & \textbf{86.07$\pm$0.18} \\
    \cmidrule(lr){2-7}
      & \multirow{3}{*}{\textbf{DiffSpan-Orth}}
        & Imakura-DC~\cite{DCframework1}   & 80.22$\pm$0.35 & \textbf{82.19$\pm$0.30} & \textbf{82.57$\pm$0.29} & 82.76$\pm$0.25 \\
      &  & Kawakami-DC~\cite{KawakamiDC} & 78.82$\pm$0.37 & 79.71$\pm$0.26 & 79.79$\pm$0.30 & 80.42$\pm$0.25 \\
      &  & ODC  & \textbf{80.36$\pm$0.32} & 81.81$\pm$0.27 & 82.24$\pm$0.26 & \textbf{82.88$\pm$0.22} \\
    \bottomrule
  \end{tabular}
  }
\end{table}

Under \textbf{SameSpan-Orth}, all three DC methods are remarkably stable in $a$. For ODC, the accuracy ranges from $95.90\%$ to $96.04\%$ on MNIST and from $86.05\%$ to $86.35\%$ on Fashion-MNIST as $a$ increases by a factor of eight. Imakura-DC closely tracks ODC (within about $0.3$ percentage points on average), while Kawakami-DC is consistently $1$–$3$ points worse. Thus, once $a$ moderately exceeds the latent dimension~$\ell$, further enlarging $\matr{A}$ has a negligible impact on image-level utility for any of the DC schemes.

Under \textbf{DiffSpan-Orth}, accuracies drop for all methods due to span mismatch between parties, but the dependence on $a$ remains weak. Increasing $a$ from $196$ to $1568$ improves ODC by about $1.1$ percentage points on MNIST ($94.02\%\to95.15\%$) and $2.5$ points on Fashion-MNIST ($80.36\%\to82.88\%$), with Imakura-DC exhibiting similar gains (e.g., $93.33\%\to94.47\%$ on MNIST, $80.22\%\to82.76\%$ on Fashion-MNIST). Kawakami-DC lags further behind—roughly $2$–$4$ points below ODC across all $a$ in this harder regime. Averaged over anchor sizes, ODC and Imakura-DC are essentially tied (differences $\leq 0.5$ points), and both substantially outperform Kawakami-DC. For image classification, anchor size is therefore a second-order effect compared to the choice of basis-alignment method and the validity of the span assumptions.

\paragraph{Molecular classification (AMES) and anchor source}

\begin{table}[t]
  \centering
  \caption{ODC and baseline DC performance on AMES as a function of anchor size $a$ and anchor source.
  We report mean ROC-AUC / PR-AUC (\%) $\pm $ 95\% confidence intervals over 100 iterations.
  ``Uniform'' uses synthetic anchors with rows sampled i.i.d.\ from $\mathrm{Unif}[0,1)$; ``PubChem''
  uses unlabeled compounds from PubChem as rows of $\matr{A}$.}
  \label{tab:anchor-ames}
  \begin{tabular}{lllcccc}
    \toprule
    Condition & Anchor & Method & $a=512$ & $1024$ & $2048$ & $4096$ \\
    \midrule
    \multicolumn{7}{c}{ROC-AUC (\%) $\uparrow$} \\
    \midrule
    \multirow{6}{*}{\textbf{SameSpan-Orth}}
      & \multirow{3}{*}{Uniform}
        & Imakura-DC~\cite{DCframework1}   & 88.23$\pm$0.13 & 87.90$\pm$0.18 & 86.93$\pm$0.33 & 84.66$\pm$0.52 \\
      &  & Kawakami-DC~\cite{KawakamiDC} & 87.38$\pm$0.17 & 85.89$\pm$0.40 & 83.07$\pm$0.42 & 81.63$\pm$0.24 \\
      &  & ODC   & \textbf{88.64$\pm$0.10} & \textbf{88.59$\pm$0.12} & \textbf{88.68$\pm$0.11} & \textbf{88.71$\pm$0.10} \\
    \cmidrule(lr){2-7}
      & \multirow{3}{*}{PubChem}
        & Imakura-DC~\cite{DCframework1}   & 88.10$\pm$0.10 & 88.09$\pm$0.14 & 87.89$\pm$0.14 & 87.28$\pm$0.21 \\
      &  & Kawakami-DC~\cite{KawakamiDC} & 86.91$\pm$0.19 & 86.87$\pm$0.19 & 85.88$\pm$0.33 & 84.06$\pm$0.42 \\
      &  & ODC   & \textbf{88.73$\pm$0.09} & \textbf{88.63$\pm$0.12} & \textbf{88.61$\pm$0.13} & \textbf{88.66$\pm$0.10} \\
    \midrule
    \multirow{6}{*}{\textbf{DiffSpan-Orth}}
      & \multirow{3}{*}{Uniform}
        & Imakura-DC~\cite{DCframework1}   & 59.11$\pm$0.28 & 60.81$\pm$0.30 & 62.22$\pm$0.34 & 63.29$\pm$0.37 \\
      &  & Kawakami-DC~\cite{KawakamiDC} & 59.11$\pm$0.29 & 61.39$\pm$0.34 & 64.10$\pm$0.37 & 65.98$\pm$0.42 \\
      &  & ODC   & \textbf{61.84$\pm$0.30} & \textbf{64.34$\pm$0.31} & \textbf{66.78$\pm$0.28} & \textbf{68.65$\pm$0.30} \\
    \cmidrule(lr){2-7}
      & \multirow{3}{*}{PubChem}
        & Imakura-DC~\cite{DCframework1}   & 77.57$\pm$0.19 & 80.60$\pm$0.17 & 81.94$\pm$0.16 & 82.31$\pm$0.15 \\
      &  & Kawakami-DC~\cite{KawakamiDC} & 74.36$\pm$0.21 & 77.30$\pm$0.22 & 78.28$\pm$0.22 & 78.65$\pm$0.21 \\
      &  & ODC   & \textbf{84.51$\pm$0.13} & \textbf{85.68$\pm$0.11} & \textbf{86.37$\pm$0.11} & \textbf{86.60$\pm$0.11} \\
    \midrule
    \multicolumn{7}{c}{PR-AUC (\%) $\uparrow$} \\
    \midrule
    \multirow{6}{*}{\textbf{SameSpan-Orth}}
      & \multirow{3}{*}{Uniform}
        & Imakura-DC~\cite{DCframework1}   & 89.67$\pm$0.13 & 89.36$\pm$0.17 & 88.56$\pm$0.34 & 86.31$\pm$0.54 \\
      &  & Kawakami-DC~\cite{KawakamiDC} & 88.97$\pm$0.17 & 87.55$\pm$0.42 & 84.69$\pm$0.44 & 83.02$\pm$0.25 \\
      &  & ODC   & \textbf{89.96$\pm$0.10} & \textbf{89.91$\pm$0.12} & \textbf{90.00$\pm$0.11} & \textbf{90.00$\pm$0.10} \\
    \cmidrule(lr){2-7}
      & \multirow{3}{*}{PubChem}
        & Imakura-DC~\cite{DCframework1}   & 89.61$\pm$0.10 & 89.55$\pm$0.13 & 89.41$\pm$0.13 & 88.87$\pm$0.21 \\
      &  & Kawakami-DC~\cite{KawakamiDC} & 88.52$\pm$0.18 & 88.48$\pm$0.19 & 87.47$\pm$0.36 & 85.49$\pm$0.47 \\
      &  & ODC   & \textbf{90.08$\pm$0.10} & \textbf{89.85$\pm$0.11} & \textbf{89.89$\pm$0.12} & \textbf{89.95$\pm$0.09} \\
    \midrule
    \multirow{6}{*}{\textbf{DiffSpan-Orth}}
      & \multirow{3}{*}{Uniform}
        & Imakura-DC~\cite{DCframework1}   & 63.05$\pm$0.29 & 64.87$\pm$0.30 & 66.11$\pm$0.31 & 67.11$\pm$0.34 \\
      &  & Kawakami-DC~\cite{KawakamiDC} & 63.16$\pm$0.29 & 65.36$\pm$0.31 & 67.78$\pm$0.34 & 69.74$\pm$0.40 \\
      &  & ODC   & \textbf{65.51$\pm$0.31} & \textbf{67.87$\pm$0.31} & \textbf{70.33$\pm$0.28} & \textbf{72.19$\pm$0.29} \\
    \cmidrule(lr){2-7}
      & \multirow{3}{*}{PubChem}
        & Imakura-DC~\cite{DCframework1}   & 80.56$\pm$0.19 & 83.16$\pm$0.17 & 84.23$\pm$0.16 & 84.55$\pm$0.14 \\
      &  & Kawakami-DC~\cite{KawakamiDC} & 77.71$\pm$0.21 & 80.27$\pm$0.20 & 80.98$\pm$0.19 & 81.24$\pm$0.19 \\
      &  & ODC   & \textbf{86.51$\pm$0.13} & \textbf{87.53$\pm$0.12} & \textbf{88.15$\pm$0.11} & \textbf{88.32$\pm$0.11} \\
    \bottomrule
  \end{tabular}
\end{table}

Under \textbf{SameSpan-Orth}, all three methods lie close to the centralized oracle, and both the size and source of $\matr{A}$ have only a mild influence. ODC remains the strongest method, with ROC-AUC around $88.6\%$ and PR-AUC around $90.0\%$ across all anchor sizes and both anchor sources, while Imakura-DC trails by roughly
$1$–$2$ points and Kawakami-DC by $3$–$4$ points on average. The variation of ODC’s scores with $a$ is extremely small (at most $0.1$ points in ROC-AUC), confirming that when the span assumptions are satisfied, anchor construction is largely irrelevant for utility.

Under \textbf{DiffSpan-Orth}, anchor design becomes critical. With a synthetic Uniform anchor, all methods degrade substantially relative to the centralized model: for ODC, ROC-AUC ranges from $61.84\%$ to $68.65\%$ and PR-AUC from $65.51\%$ to $72.19\%$ as $a$ grows from $512$ to $4096$, while Imakura-DC and Kawakami-DC are roughly
$2$–$5$ points below ODC at each anchor size. Switching from Uniform to a PubChem anchor dramatically stabilizes all three methods, but especially ODC. For instance, at $a=4096$ ODC attains $86.60\%$ ROC-AUC and $88.32\%$ PR-AUC, compared with $82.31\%$ / $84.55\%$ for Imakura-DC and $78.65\%$ / $81.24\%$ for Kawakami-DC. Across anchor sizes, PubChem anchors improve ODC by roughly $18$–$22$ points over its Uniform counterpart in both ROC-AUC and PR-AUC, while yielding gains of $15$–$20$ points for the baselines. In this challenging span-mismatched regime, the proposed ODC framework consistently achieves the highest scores, outperforming Imakura-DC by approximately $4$–$7$ points and Kawakami-DC by about $8$–$10$ points.

\paragraph{Statistical analysis}
We now quantify how strongly anchor construction affects performance using formal hypothesis tests. For the \emph{anchor size} experiments, we test, for each task, span condition, method, and metric, the null hypothesis
\begin{equation}
H_{0,\text{size}}:\quad
\mathbb{E}[\text{metric} \mid a = a_1]
= \cdots =
\mathbb{E}[\text{metric} \mid a = a_K],
\end{equation}
i.e., that the mean performance is identical across all anchor sizes $a\in\{a_1,\dots,a_K\}$. For the AMES anchor \emph{source} experiments (Uniform vs.\ PubChem), we fit two-way ANOVA models with factors anchor source and anchor size and test, for each span condition, method, and metric, the null
\begin{equation}
H_{0,\text{source}}:\quad
\mathbb{E}[\text{metric} \mid \text{Uniform}]
=
\mathbb{E}[\text{metric} \mid \text{PubChem}],
\end{equation}
conditional on $a$. In all cases we report the $F$-statistic, $p$-value, and Cohen's $f$
($f\approx 0.10$ small, $f\approx 0.25$ medium, $f\approx 0.40$ large).

Tables~\ref{tab:anova-image} and ~\ref{tab:anova-ames} summarize the one-way ANOVA results for anchor size. On MNIST and Fashion-MNIST under \textbf{SameSpan-Orth}, ODC shows no statistically significant dependence on $a$ ($p>0.15$, $f<0.12$ on both datasets), while Imakura-DC exhibits small effects ($p<0.005$, $f\approx 0.18$--$0.19$) and Kawakami-DC moderate ones
($p<10^{-16}$, $f\approx 0.45$--$0.47$). Under \textbf{DiffSpan-Orth}, the null $H_{0,\text{size}}$ is rejected for all three methods, with large effect sizes (e.g., for ODC, $f\approx 0.57$ on MNIST and $0.67$ on Fashion-MNIST), although the corresponding changes in mean accuracy remain modest (on the order of $1$--$3$ percentage points as $a$ increases; Table~\ref{tab:anchor-image}). Thus, in the well-specified image regime, ODC is essentially invariant to $a$, whereas the contemporary methods show mild-to-moderate sensitivity; in the span-mismatched regime, all DC methods benefit from larger anchors, with ODC typically achieving the best accuracies.

On AMES, the one-way ANOVA reveals a similar pattern (Table~\ref{tab:anova-ames}). Under \textbf{SameSpan-Orth} with synthetic Uniform anchors, ODC again appears insensitive to $a$ ($p=0.40$ and $0.66$ for ROC-AUC and PR-AUC, $f\leq 0.09$), while Imakura-DC and Kawakami-DC show strong statistical dependence on $a$ with large effect sizes (e.g., for Uniform ROC-AUC, $f\approx 0.83$ for Imakura-DC and $1.37$ for Kawakami-DC). The same trend holds for PubChem anchors: ODC has at most a small effect of $a$ (e.g., PR-AUC $p=0.017$, $f=0.16$), whereas Imakura-DC and Kawakami-DC exhibit medium-to-large effects ($f\approx 0.39$--$0.76$). Under \textbf{DiffSpan-Orth}, $H_{0,\text{size}}$ is overwhelmingly rejected for all three methods and both anchor sources, with very large Cohen's $f$ values ($f\gtrsim 0.95$ and often $>1.5$), reflecting the substantial improvements in ROC-AUC and PR-AUC as $a$ grows from $512$ to $4096$ (Table~\ref{tab:anchor-ames}). Overall, anchor size is practically negligible for ODC in the well-specified setting, but it can matter substantially for Imakura-DC and Kawakami-DC and for \emph{all} methods under span mismatch.

To isolate the effect of anchor \emph{source} on AMES, Table~\ref{tab:anova-anchor-source} reports two-way ANOVA results for the main effect of anchor source (Uniform vs.\ PubChem), controlling for $a$. Under \textbf{SameSpan-Orth}, ODC does not distinguish between Uniform and PubChem anchors (ROC-AUC: $p=0.87$, $f=0.01$; PR-AUC: $p=0.53$, $f=0.02$), whereas Imakura-DC and Kawakami-DC show statistically significant but small-to-moderate effects ($f\approx 0.33$--$0.46$), corresponding to absolute differences of roughly $1$--$1.5$ percentage points between the two anchor distributions. Under \textbf{DiffSpan-Orth}, the anchor source has an extremely large impact for all three methods: for ROC-AUC and PR-AUC the null $H_{0,\text{source}}$ is rejected with $F$-statistics on the order of $10^4$ and Cohen's $f$ between $4.8$ and $8.9$, consistent with the $\approx 13$--$20$ point gains observed when replacing a synthetic Uniform anchor with a domain-matched PubChem anchor. Among the three DC schemes, ODC consistently achieves the largest absolute improvements with PubChem (around $18$--$20$ points in ROC-AUC and PR-AUC), and is also the most robust to anchor size when the span assumptions are satisfied.

\begin{table}[t]
  \centering
  \caption{One-way ANOVA on the effect of anchor size $a$ on MNIST and Fashion-MNIST performance
  with Uniform anchors. For each span condition and method, we report the
  $F$-statistic, $p$-value, and Cohen's $f$ for $H_{0,\text{size}}$.}
  \label{tab:anova-image}
  \begin{tabular}{lllccc}
    \toprule
    \multirow{2}{*}{Dataset} & \multirow{2}{*}{Condition} & \multirow{2}{*}{Method} & \multicolumn{3}{c}{Effect of anchor size $a$} \\
    \cmidrule(lr){4-6}
    & & & {$F$} & {$p$} & {Cohen's $f$} \\
    \midrule
    \multirow{6}{*}{MNIST}
      & \multirow{3}{*}{\textbf{SameSpan-Orth}}
        & Imakura-DC~\cite{DCframework1}   & 5.04   & 0.002    & 0.19  \\
      &  & Kawakami-DC~\cite{KawakamiDC} & 32.37  & $6.79 \times 10^{-19}$ & 0.47  \\
      &  & ODC  & 1.05   & 0.371    & 0.09  \\
    \cmidrule(lr){2-6}
      & \multirow{3}{*}{\textbf{DiffSpan-Orth}}
        & Imakura-DC~\cite{DCframework1}   & 59.41  & $3.38 \times 10^{-32}$ & 0.64  \\
      &  & Kawakami-DC~\cite{KawakamiDC} & 39.70  & $1.10 \times 10^{-22}$ & 0.52  \\
      &  & ODC  & 47.20  & $2.05 \times 10^{-26}$ & 0.57  \\
    \midrule
    \multirow{6}{*}{Fashion-MNIST}
      & \multirow{3}{*}{\textbf{SameSpan-Orth}}
        & Imakura-DC~\cite{DCframework1}   & 4.44   & 0.004    & 0.18  \\
      &  & Kawakami-DC~\cite{KawakamiDC} & 28.10  & $1.41 \times 10^{-16}$ & 0.45  \\
      &  & ODC  & 1.74   & 0.158    & 0.11  \\
    \cmidrule(lr){2-6}
      & \multirow{3}{*}{\textbf{DiffSpan-Orth}}
        & Imakura-DC~\cite{DCframework1}   & 59.30  & $5.56 \times 10^{-32}$ & 0.65  \\
      &  & Kawakami-DC~\cite{KawakamiDC} & 19.76  & $5.33 \times 10^{-12}$ & 0.38  \\
      &  & ODC  & 63.44  & $7.54 \times 10^{-34}$ & 0.67  \\
    \bottomrule
  \end{tabular}
\end{table}

\begin{table}[t]
  \centering
  \caption{One-way ANOVA on the effect of anchor size $a$ on AMES performance with Uniform and Pubchem anchors. For each span condition,
  metric, and method we report the $F$-statistic, $p$-value, and Cohen's $f$
  for $H_{0,\text{size}}$.}
  \label{tab:anova-ames}
  \begin{tabular}{lllccc}
    \toprule
    \multirow{2}{*}{Condition} & \multirow{2}{*}{Anchor} & \multirow{2}{*}{Method} & \multicolumn{3}{c}{Effect of anchor size $a$} \\
    \cmidrule(lr){4-6}
    & & & {$F$} & {$p$} & {Cohen's $f$} \\
    \midrule
    \multicolumn{6}{c}{ROC-AUC (\%)} \\
    \midrule
    \multirow{6}{*}{\textbf{SameSpan-Orth}}
      & \multirow{3}{*}{Uniform}
        & Imakura-DC~\cite{DCframework1}   & 91.55  & $5.08\times 10^{-45}$ & 0.83  \\
      &  & Kawakami-DC~\cite{KawakamiDC} & 248.96 & $9.34\times 10^{-91}$ & 1.37  \\
      &  & ODC         & 0.98   & 0.404                 & 0.09  \\
    \cmidrule(lr){2-6}
      & \multirow{3}{*}{PubChem}
        & Imakura-DC~\cite{DCframework1}   & 25.54  & $4.00\times 10^{-15}$ & 0.44  \\
      &  & Kawakami-DC~\cite{KawakamiDC} & 77.09  & $2.73\times 10^{-39}$ & 0.76  \\
      &  & ODC         & 0.93   & 0.425                 & 0.08  \\
    \midrule
    \multirow{6}{*}{\textbf{DiffSpan-Orth}}
      & \multirow{3}{*}{Uniform}
        & Imakura-DC~\cite{DCframework1}   & 119.68 & $3.54\times 10^{-55}$ & 0.95   \\
      &  & Kawakami-DC~\cite{KawakamiDC} & 278.68 & $3.30\times 10^{-97}$ & 1.45   \\
      &  & ODC         & 385.15 & $5.19\times 10^{-117}$ & 1.71  \\
    \cmidrule(lr){2-6}
      & \multirow{3}{*}{PubChem}
        & Imakura-DC~\cite{DCframework1}   & 626.14 & $6.98\times 10^{-150}$ & 2.18  \\
      &  & Kawakami-DC~\cite{KawakamiDC} & 318.73 & $3.34\times 10^{-105}$ & 1.55  \\
      &  & ODC         & 249.27 & $7.94\times 10^{-91}$  & 1.37  \\
    \midrule
    \multicolumn{6}{c}{PR-AUC (\%)} \\
    \midrule
    \multirow{6}{*}{\textbf{SameSpan-Orth}}
      & \multirow{3}{*}{Uniform}
        & Imakura-DC~\cite{DCframework1}   & 77.97  & $1.19\times 10^{-39}$ & 0.77  \\
      &  & Kawakami-DC~\cite{KawakamiDC} & 243.30 & $1.80\times 10^{-89}$ & 1.36  \\
      &  & ODC         & 0.54   & 0.656                 & 0.06  \\
    \cmidrule(lr){2-6}
      & \multirow{3}{*}{PubChem}
        & Imakura-DC~\cite{DCframework1}   & 20.33  & $2.83\times 10^{-12}$ & 0.39  \\
      &  & Kawakami-DC~\cite{KawakamiDC} & 74.20  & $4.24\times 10^{-38}$ & 0.75  \\
      &  & ODC         & 3.42   & 0.017                 & 0.16  \\
    \midrule
    \multirow{6}{*}{\textbf{DiffSpan-Orth}}
      & \multirow{3}{*}{Uniform}
        & Imakura-DC~\cite{DCframework1}   & 121.81 & $6.69\times 10^{-56}$ & 0.96   \\
      &  & Kawakami-DC~\cite{KawakamiDC} & 273.49 & $4.08\times 10^{-96}$ & 1.44   \\
      &  & ODC         & 364.78 & $1.47\times 10^{-113}$ & 1.66  \\
    \cmidrule(lr){2-6}
      & \multirow{3}{*}{PubChem}
        & Imakura-DC~\cite{DCframework1}   & 465.59 & $1.96\times 10^{-129}$ & 1.88  \\
      &  & Kawakami-DC~\cite{KawakamiDC} & 255.57 & $3.11\times 10^{-92}$  & 1.39  \\
      &  & ODC         & 191.34 & $1.12\times 10^{-76}$  & 1.20  \\
    \bottomrule
  \end{tabular}
\end{table}

\begin{table}[t]
  \centering
  \small
  \setlength{\tabcolsep}{5pt}
  \caption{Two-way ANOVA on the effect of anchor source (Uniform vs.\ PubChem) on AMES performance,
  controlling for anchor size $a$. For each span condition, method, and metric, we report the main-effect
  $F$-statistic, $p$-value, and Cohen's $f$ for the null hypothesis $H_{0,\text{source}}$ that Uniform and
  PubChem anchors yield identical mean performance.}
  \label{tab:anova-anchor-source}
  \begin{tabular}{
    l
    l
    l
    S[table-format=5.2]   % F
    c                     % p
    S[table-format=1.2]   % f
  }
    \toprule
    \multirow{2}{*}{Condition} &
    \multirow{2}{*}{Metric} &
    \multirow{2}{*}{Method} &
    \multicolumn{3}{c}{Effect of anchor source} \\
    \cmidrule(lr){4-6}
    & & & {$F$} & $p$ & {Cohen's $f$} \\
    \midrule
    \multirow{6}{*}{\textbf{SameSpan-Orth}}
      & \multirow{3}{*}{PR-AUC $\uparrow$}
      & Imakura-DC~\cite{DCframework1}  & 88.18    & $6.20\times 10^{-20}$ & 0.33 \\
      & & Kawakami-DC~\cite{KawakamiDC} & 143.53   & $1.63\times 10^{-30}$ & 0.43 \\
      & & ODC         & 0.40     & 0.526                  & 0.02 \\
      \cmidrule(lr){2-6}
      & \multirow{3}{*}{ROC-AUC $\uparrow$}
      & Imakura-DC~\cite{DCframework1}  & 97.21    & $1.04\times 10^{-21}$ & 0.35 \\
      & & Kawakami-DC~\cite{KawakamiDC} & 163.15   & $4.18\times 10^{-34}$ & 0.45 \\
      & & ODC         & 0.03     & 0.868                  & 0.01 \\
    \midrule
    \multirow{6}{*}{\textbf{DiffSpan-Orth}}
      & \multirow{3}{*}{PR-AUC $\uparrow$}
      & Imakura-DC~\cite{DCframework1}  & 39598.38 & $<10^{-300}$          & 7.07 \\
      & & Kawakami-DC~\cite{KawakamiDC} & 18249.18 & $<10^{-300}$          & 4.80 \\
      & & ODC         & 52065.06 & $<10^{-300}$          & 8.11 \\
      \cmidrule(lr){2-6}
      & \multirow{3}{*}{ROC-AUC $\uparrow$}
      & Imakura-DC~\cite{DCframework1}  & 42639.30 & $<10^{-300}$          & 7.34 \\
      & & Kawakami-DC~\cite{KawakamiDC} & 18904.45 & $<10^{-300}$          & 4.89 \\
      & & ODC         & 63204.41 & $<10^{-300}$          & 8.93 \\
    \bottomrule
  \end{tabular}
\end{table}

\paragraph{Privacy considerations}
Across all experiments, anchor construction does not alter the formal privacy guarantees of DC. The anchor dataset $\matr{A}$ is either synthetic or drawn from public unlabeled data, and only its projections $\matr{A}\matr{F}_i$ and the user-side representations $\tilde{\matr{X}}_i = \matr{X}_i\matr{F}_i$ participate in the collaboration protocol. As a result, the threat surfaces analyzed in \S~\ref{subsection: DCprivacy} (e.g., collusion and reconstruction attacks) remain governed by the secret bases $\matr{F}_i$ and the latent dimension~$\ell$, rather than by the particular sampling distribution or size of $\matr{A}$. The anchor ablations in this section, therefore, show that anchor design is a powerful knob for improving utility in difficult span-mismatched settings—especially for ODC—without introducing additional privacy risks.

\section{Conclusion}
\label{sec:conclusion}

In this paper, we revisited the Data Collaboration (DC) paradigm and identified a key gap between existing theory
and practice. Although weak concordance guarantees that parties can be aligned to a common subspace, downstream
performance can still vary substantially with the analyst’s \emph{choice of target basis}. To resolve this instability,
we proposed \textbf{Orthonormal Data Collaboration (ODC)}, which enforces orthonormality on both secret and
target bases. Under this constraint, basis alignment reduces exactly to the classical Orthogonal Procrustes
Problem~\cite{procrustessolution}, yielding a closed-form solution and improving the alignment complexity from
\begin{equation}
  O\bigl(\min\{a(c\ell)^2,\;a^2c\ell\}\bigr)
  \;\rightarrow\;
  O(ac\ell^2),
\end{equation}
where $a$ is the anchor size, $\ell$ the latent dimension, and $c$ the number of parties.

Our empirical results validate these computational gains. On the controlled efficiency benchmark (Fig.~\ref{fig:time_scaling}), ODC consistently outperformed contemporary DC alignments with speed-ups ranging from roughly $6\times$ to over $100\times$; for example, at $a=20{,}000$ the median runtime dropped from $\approx 50\,\mathrm{s}$ (baselines) to $0.47\,\mathrm{s}$ (ODC), i.e., more than two orders of magnitude. Moreover, the scaling experiments demonstrate that ODC remains practical even as the number of users increases (Fig.~\ref{fig:vary_c_time}), rendering the alignment phase typically negligible in realistic deployments.

Beyond speed, ODC’s central benefit is \emph{stability}. We proved that ODC’s orthogonal change-of-basis matrices satisfy \emph{orthogonal concordance}, aligning all parties’ representations up to a common orthogonal transform and thereby preserving distances and inner products. This theoretical invariance is reflected in the concordance experiments (Tables~\ref{tab:span_concordance_same}--\ref{tab:span_concordance_comparison}): replacing the identity target with a random \emph{invertible} target in Imakura-DC causes consistent and statistically significant accuracy drops (up to $3$--$4$ percentage points), whereas ODC is essentially invariant to the choice of orthogonal target---exactly so for SVMs and nearly so for MLPs under the orthonormal setting.

Across application benchmarks, ODC matches or improves upon existing DC methods when orthonormality holds (Tables~\ref{tab: assumptions-all}--\ref{tab: CelebA accuracy}). Under the realistic \textbf{DiffSpan-Orth} condition (heterogeneous spans but orthonormal bases), ODC remains competitive and often exceeds the DC baselines on biomedical tasks, while the ablations also make the practical requirement clear: violating orthonormality leads to marked degradation for ODC (Tables~\ref{tab: assumptions-all}--\ref{tab: CelebA accuracy}). This directly motivates the use of standard orthonormal basis construction (e.g., PCA/SVD/QR), which is already common in DC pipelines.

Finally, ODC compares favorably to representative PPML baselines while retaining DC’s one-shot communication pattern. On CelebA, DC-style projections provide strong visual obfuscation and reduce FaceNet-based identifiability toward chance (Table~\ref{tab: celeba-identifiability}), while maintaining accuracy competitive with high-$\varepsilon$ DP noise and substantially better than stricter DP settings (Table~\ref{tab: CelebA accuracy}). On eICU regression, ODC under \textbf{DiffSpan-Orth} achieves RMSE close to FedAvg while avoiding iterative communication (Table~\ref{tab: eiCU rmse}). Section~\ref{sec:anchor-construction} further shows that ODC’s behavior with respect to anchor design is well-behaved: when the span assumptions hold, it is essentially insensitive to anchor size and source---image accuracies move by at most about $0.3$ percentage points across an $8\times$ sweep in $a$ (Table~\ref{tab:anchor-image}), and AMES ROC-/PR-AUC under \textbf{SameSpan-Orth} are nearly flat in both $a$ and between synthetic and PubChem anchors (Table~\ref{tab:anchor-ames})---whereas under the more realistic \textbf{DiffSpan-Orth} regime, anchor construction becomes a powerful lever, with domain-matched PubChem anchors improving ODC by roughly $18$--$22$ ROC-/PR-AUC points over synthetic anchors on AMES (Table~\ref{tab:anchor-ames}) without changing DC’s semi-honest privacy model.

Promising future directions include strengthening privacy beyond the semi-honest model (e.g., collusion-robust variants and principled DP integration on released representations), extending orthonormal alignment ideas to nonlinear mappings and partially overlapping feature spaces, and \emph{deepening} deployment guidelines (e.g., automatic strategies for selecting $\ell$ and anchor schemes under task-specific constraints on utility, communication, and attack surfaces).

\subsection{Practical deployment and systems considerations}
\label{sec:deployment}

While ODC is motivated by a mathematical gap in basis alignment, its design targets \emph{deployable} cross-silo workflows where parties cannot centralize raw records and where iterative communication (e.g., FL-style rounds) is expensive or operationally infeasible. A semi-realistic end-to-end deployment pipeline typically consists of three stages: (i) a one-off anchor bootstrapping step in which the consortium agrees on a shared anchor $\matr{A}$ and each site constructs its secret basis $\matr{F}_i$; (ii) a single DC round in which users compute and upload $(\tilde{\matr{X}}_i,\matr{A}_i,\matr{L}_i)$, and the analyst performs ODC alignment and model training; and (iii) a deployment step where each site receives $(\matr{G}_i,h)$ and applies $h(\matr{Y}_i \matr{F}_i \matr{G}_i)$ to local test data. Table~\ref{tab:deployment_checklist} distills the main engineering and governance decisions that a practitioner must finalize. Below, we highlight practical considerations that complement our analyses of complexity, memory, and communication (e.g., \S~\ref{subsec:communication}).

\paragraph{Cold-start anchor generation (one-off)}
A deployment must decide how to instantiate the shared anchor matrix $\matr{A}\in\mathbb{R}^{a\times m}$ at “cold start,” corresponding to the initial stage above. Common options are: \emph{(i) synthetic anchors} (e.g., i.i.d.\ random or dummy feature vectors), \emph{(ii) public/unlabeled anchors} from an external source (domain-dependent), or \emph{(iii) seed-based anchors} where all parties deterministically generate the same $\matr{A}$ from a shared PRNG seed (avoiding repeated cross-silo distribution). Independently of the source, the engineering requirement is to satisfy the rank and conditioning assumptions used by DC/ODC (e.g., $\mathrm{rank}(\matr{A})=m$ and typically $a>m$), while keeping $\matr{A}$ non-sensitive (not derived from private records) to avoid governance complications. These choices and their documentation are captured in the “Anchor plan” row of Table~\ref{tab:deployment_checklist}.

\paragraph{Resource constraints: compute, memory, and network}
ODC’s alignment workload is dominated by per-party $\ell\times\ell$ Procrustes/SVD computations and $\ell\times a$ by $a\times \ell$ products, which are parallelizable and friendly to modern BLAS kernels. Peak memory at the analyst is driven by storing the transformed anchors $\{\matr{A}_i\}_{i=1}^c$ (order $\Theta(ac\ell)$), as indicated in the “Compute/memory” row of Table~\ref{tab:deployment_checklist}. For networking, DC/ODC remains a one-shot protocol: parties upload $(\tilde{\matr{X}}_i,\matr{A}_i,\matr{L}_i)$ once and receive $(\matr{G}_i,h)$ once. The practical traffic budget can be estimated directly from the communication model in \S~\ref{subsec:communication}; the existing healthcare example (100 hospitals with a ResNet-50) illustrates how this translates to concrete GB-scale volumes and highlights when DC/ODC is advantageous under bandwidth and RTT constraints. The “Network budget” and “Numerics” rows in Table~\ref{tab:deployment_checklist} summarize the associated design choices (precision, quantization, batching, compression).

\paragraph{Cross-sector deployment sketches}
The same pipeline applies across sectors: only the local site semantics (data type, governance regime) change. Typical instantiations include:
\begin{itemize}
  \item \textbf{Cross-institution consortia (general):} organizations keep data on-prem and share only projected representations for joint modeling; ODC’s closed-form alignment supports predictable runtime and avoids multi-round coordination overhead.
  \item \textbf{Medical / multi-hospital:} typical targets include imaging, ICU risk, or outcomes prediction across hospitals with strict data governance. One-shot transfer and explicit communication accounting (\S~\ref{subsec:communication}) support feasibility checks under hospital network constraints.
  \item \textbf{Power / utilities:} operators can collaborate on forecasting or anomaly detection without exchanging raw operational traces; one-shot protocols reduce persistent connectivity requirements and simplify segmented network operations.
  \item \textbf{Finance:} banks and insurers can collaborate on fraud or risk models while maintaining data locality; deployment often requires strict auditability and controls on what representations and metadata leave each institution.
\end{itemize}

\paragraph{Compliance and governance considerations (non-exhaustive)}
ODC inherits the same privacy model as the underlying DC protocol (semi-honest analyst/users), so a deployment should explicitly document (i) the assumed threat model (including whether collusion is in scope), (ii) whether additional protections are required (e.g., encryption in transit/at rest, access control, audit logging, retention limits), and (iii) whether a formal privacy mechanism (e.g., DP layered on released representations) is necessary for the application’s regulatory posture. Table~\ref{tab:deployment_checklist} can serve as a concise checklist to ensure that these governance and engineering decisions are explicitly documented before a system goes into production.

\begin{table}[t]
\centering
\caption{Deployment checklist for ODC/DC (engineering and governance).}
\label{tab:deployment_checklist}
\begin{tabular}{p{0.28\linewidth} p{0.67\linewidth}}
\hline
\textbf{Item} & \textbf{What to decide / verify in practice} \\
\hline
Threat model &
Confirm semi-honest assumptions; decide whether collusion is in scope and whether additional protections are needed
(e.g., DP on released representations, encryption in transit/at rest) are required. \\
Anchor plan &
Synthetic/public/seed-based anchor; document cold-start distribution (topology, replication factor, and
whether seeding avoids cross-silo transfer). \\
Basis selection &
Enforce orthonormal $\matr{F}_i$ (PCA/SVD/QR); choose latent dimension $\ell$ to balance utility against
communication volume and collusion surface. \\
Numerics &
Choose precision/quantization (e.g., FP32/FP16/int8) for uplink; validate alignment stability for $\ell\times\ell$
matrices and heterogeneous sites. \\
Network budget &
Estimate uplink/downlink traffic using the DC accounting in \S~\ref{subsec:communication}; plan batching,
compression, and tolerance to intermittent or low-bandwidth links. \\
Compute/memory &
Analyst memory is dominated by storing $\matr{A}_i$ for all $i \in [c]$ (order $\Theta(ac\ell)$); consider
streaming or parallelization across sites; document expected wall-clock runtime for alignment and training. \\
\hline
\end{tabular}
\end{table}

\section*{CRediT authorship contribution statement}
\label{CRediT}

\textbf{Keiyu Nosaka:} Conceptualization, Methodology, Software, Validation, Investigation, Data curation, Writing – original draft, Visualization, Funding acquisition, Project administration. 

\textbf{Yamato Suetake:} Software, Writing – original draft.

\textbf{Yuichi Takano:} Writing – review \& editing, Supervision.

\textbf{Akiko Yoshise:} Writing – review \& editing, Supervision, Funding acquisition.

\section*{Declaration of Competing Interest}
\label{Interest}

The authors declare that they have no known competing financial interests or personal relationships that could have influenced the work reported in this paper.

\section*{Declaration of Generative AI and AI-Assisted Technologies in the Writing Process}
\label{GenerativeAI}

During the preparation of this work, the authors employed ChatGPT and Grammarly to enhance the readability and language of the manuscript. Following the use of these services, the authors thoroughly reviewed and edited the content, assuming full responsibility for the final published article.

\section*{Data Availability}
All datasets used in our experiments are publicly available (see the corresponding citations). The code used to conduct the experiments is available at the following URL:

\url{https://drive.google.com/drive/folders/1PwXVarZ7mCoDZVzOdYDuZ7jSpILBsXBj?usp=sharing}

\section*{Funding Sources}

This work was partially supported by the Japan Society for the Promotion of Science (JSPS) KAKENHI under Grant Numbers JP22K18866, JP23K26327, and JP25KJ0701, as well as the Japan Science and Technology Agency (JST) under Grant Number JPMJSP2124.

% \bibliographystyle{elsarticle-num} 
% \bibliography{Bibliography}

\begin{thebibliography}{10}
\expandafter\ifx\csname url\endcsname\relax
  \def\url#1{\texttt{#1}}\fi
\expandafter\ifx\csname urlprefix\endcsname\relax\def\urlprefix{URL }\fi
\expandafter\ifx\csname href\endcsname\relax
  \def\href#1#2{#2} \def\path#1{#1}\fi

\bibitem{DataBreaching}
P.~Rosati, P.~Deeney, M.~Cummins, L.~van~der Werff, T.~Lynn, Social media and stock price reaction to data breach announcements: Evidence from us listed companies, Research in International Business and Finance 47 (2019) 458--469.

\bibitem{FL1}
B.~McMahan, E.~Moore, D.~Ramage, S.~Hampson, B.~A. y~Arcas, Communication-efficient learning of deep networks from decentralized data, in: Artificial Intelligence and Statistics, PMLR, 2017, pp. 1273--1282.

\bibitem{DP}
C.~Dwork, Differential privacy: A survey of results, in: International Conference on Theory and Applications of Models of Computation, Springer, 2008, pp. 1--19.

\bibitem{NbaFL}
K.~Wei, J.~Li, M.~Ding, C.~Ma, H.~H. Yang, F.~Farokhi, S.~Jin, T.~Q. Quek, H.~V. Poor, Federated learning with differential privacy: Algorithms and performance analysis, IEEE transactions on information forensics and security 15 (2020) 3454--3469.

\bibitem{PPML}
R.~Xu, N.~Baracaldo, J.~Joshi, Privacy-preserving machine learning: Methods, challenges and directions, arXiv preprint arXiv:2108.04417 (2021).

\bibitem{DCframework1}
A.~Imakura, T.~Sakurai, Data collaboration analysis framework using centralization of individual intermediate representations for distributed data sets, ASCE-ASME Journal of Risk and Uncertainty in Engineering Systems, Part A: Civil Engineering 6~(2) (2020) 04020018.

\bibitem{DCframework2}
A.~Imakura, X.~Ye, T.~Sakurai, Collaborative data analysis: Non-model sharing-type machine learning for distributed data, in: Knowledge Management and Acquisition for Intelligent Systems: 17th Pacific Rim Knowledge Acquisition Workshop, PKAW 2020, Yokohama, Japan, January 7--8, 2021, Proceedings 17, Springer, 2021, pp. 14--29.

\bibitem{DCprivacy}
A.~Imakura, A.~Bogdanova, T.~Yamazoe, K.~Omote, T.~Sakurai, Accuracy and privacy evaluations of collaborative data analysis, Proceedings of the AAAI Conference on Artificial Intelligence (2021).

\bibitem{NRI-DC}
A.~Imakura, T.~Sakurai, Y.~Okada, T.~Fujii, T.~Sakamoto, H.~Abe, Non-readily identifiable data collaboration analysis for multiple datasets including personal information, Information Fusion 98 (2023) 101826.

\bibitem{DC-DP}
H.~Yamashiro, K.~Omote, A.~Imakura, T.~Sakurai, Toward the application of differential privacy to data collaboration, IEEE Access PP (2024) 1--1.
\newblock \href {https://doi.org/10.1109/ACCESS.2024.3396146} {\path{doi:10.1109/ACCESS.2024.3396146}}.

\bibitem{FedDCL}
A.~Imakura, T.~Sakurai, Feddcl: a federated data collaboration learning as a hybrid-type privacy-preserving framework based on federated learning and data collaboration, arXiv preprint arXiv:2409.18356 (2024).

\bibitem{KawakamiDC}
Y.~Kawakami, Y.~Takano, A.~Imakura, New solutions based on the generalized eigenvalue problem for the data collaboration analysis, arXiv preprint arXiv:2404.14164 (2024).

\bibitem{publication1}
K.~Nosaka, A.~Yoshise, Creating collaborative data representations using matrix manifold optimal computation and automated hyperparameter tuning, in: 2023 IEEE 3rd International Conference on Electronic Communications, Internet of Things and Big Data (ICEIB), IEEE, 2023, pp. 180--185.

\bibitem{procrustessolution}
P.~H. Sch{\"o}nemann, A generalized solution of the orthogonal procrustes problem, Psychometrika 31~(1) (1966) 1--10.

\bibitem{pseudoinverse}
R.~Penrose, A generalized inverse for matrices, Proceedings of the Cambridge Philosophical Society 51 (1955) 406--413.

\bibitem{DCfeatures}
A.~Mizoguchi, A.~Imakura, T.~Sakurai, Application of data collaboration analysis to distributed data with misaligned features, Informatics in Medicine Unlocked 32 (2022) 101013.

\bibitem{DCnoniid}
A.~Mizoguchi, A.~Bogdanova, A.~Imakura, T.~Sakurai, Data collaboration analysis applied to compound datasets and the introduction of projection data to non-iid settings (2023).

\bibitem{DCapp1}
T.~Nakayama, Y.~Kawamata, A.~Toyoda, A.~Imakura, R.~Kagawa, M.~Sanuki, R.~Tsunoda, K.~Yamagata, T.~Sakurai, Y.~Okada, \href{https://arxiv.org/abs/2501.06511}{Data collaboration for causal inference from limited medical testing and medication data} (2025).
\newblock \href {http://arxiv.org/abs/2501.06511} {\path{arXiv:2501.06511}}.
\newline\urlprefix\url{https://arxiv.org/abs/2501.06511}

\bibitem{DCapp2}
Y.~Kawamata, R.~Motai, Y.~Okada, A.~Imakura, T.~Sakurai, \href{https://www.sciencedirect.com/science/article/pii/S0957417423035261}{Collaborative causal inference on distributed data}, Expert Systems with Applications 244 (2024) 123024.
\newblock \href {https://doi.org/https://doi.org/10.1016/j.eswa.2023.123024} {\path{doi:https://doi.org/10.1016/j.eswa.2023.123024}}.
\newline\urlprefix\url{https://www.sciencedirect.com/science/article/pii/S0957417423035261}

\bibitem{DCapp3}
A.~Bogdanova, A.~Imakura, T.~Sakurai, Dc-shap method for consistent explainability in privacy-preserving distributed machine learning, Human-Centric Intelligent Systems 3~(3) (2023) 197--210.

\bibitem{DCapp4}
A.~Imakura, R.~Tsunoda, R.~Kagawa, K.~Yamagata, T.~Sakurai, Dc-cox: Data collaboration cox proportional hazards model for privacy-preserving survival analysis on multiple parties, Journal of Biomedical Informatics 137 (2023) 104264.

\bibitem{DCapp5}
A.~Imakura, H.~Inaba, Y.~Okada, T.~Sakurai, \href{https://www.sciencedirect.com/science/article/pii/S0957417421003328}{Interpretable collaborative data analysis on distributed data}, Expert Systems with Applications 177 (2021) 114891.
\newblock \href {https://doi.org/https://doi.org/10.1016/j.eswa.2021.114891} {\path{doi:https://doi.org/10.1016/j.eswa.2021.114891}}.
\newline\urlprefix\url{https://www.sciencedirect.com/science/article/pii/S0957417421003328}

\bibitem{DCrecommender}
T.~Yanagi, S.~Ikeda, N.~Sukegawa, Y.~Takano, Privacy-preserving recommender system using the data collaboration analysis for distributed datasets, arXiv preprint arXiv:2406.01603 (2024).

\bibitem{celeba}
Z.~Liu, P.~Luo, X.~Wang, X.~Tang, Deep learning face attributes in the wild, in: Proceedings of International Conference on Computer Vision (ICCV), 2015.

\bibitem{DRprivacy}
H.~Nguyen, D.~Zhuang, P.-Y. Wu, M.~Chang, Autogan-based dimension reduction for privacy preservation, Neurocomputing 384 (2020) 94--103.

\bibitem{Haxby2011CommonModel}
J.~V. Haxby, J.~S. Guntupalli, A.~C. Connolly, Y.~O. Halchenko, B.~R. Conroy, M.~I. Gobbini, M.~Hanke, P.~J. Ramadge, A common, high-dimensional model of the representational space in human ventral temporal cortex, Neuron 72~(2) (2011) 404--416.
\newblock \href {https://doi.org/10.1016/j.neuron.2011.08.026} {\path{doi:10.1016/j.neuron.2011.08.026}}.

\bibitem{Lorbert2012KernelHyperalignment}
A.~Lorbert, P.~J. Ramadge, Kernel hyperalignment, in: Advances in Neural Information Processing Systems 25, 2012, pp. 1799--1807.

\bibitem{GOPP}
S.~Ling, Near-optimal bounds for generalized orthogonal procrustes problem via generalized power method, Applied and Computational Harmonic Analysis 66 (2023) 62--100.

\bibitem{Nie2018AWP}
F.~Nie, L.~Tian, X.~Li, Multiview clustering via adaptively weighted procrustes, in: Proceedings of the 24th ACM SIGKDD International Conference on Knowledge Discovery and Data Mining, 2018, pp. 2022--2030.
\newblock \href {https://doi.org/10.1145/3219819.3220049} {\path{doi:10.1145/3219819.3220049}}.

\bibitem{Dong2022GrassmannProcrustes}
X.~Dong, D.~Wu, F.~Nie, R.~Wang, X.~Li, Multi-view clustering with adaptive procrustes on grassmann manifold, Information Sciences 609 (2022) 855--875.
\newblock \href {https://doi.org/10.1016/j.ins.2022.07.089} {\path{doi:10.1016/j.ins.2022.07.089}}.

\bibitem{Wang2008ManifoldProcrustes}
C.~Wang, S.~Mahadevan, Manifold alignment using procrustes analysis, in: Proceedings of the 25th International Conference on Machine Learning, 2008, pp. 1120--1127.

\bibitem{Grave2019WassersteinProcrustes}
E.~Grave, A.~Joulin, Q.~Berthet, Unsupervised alignment of embeddings with wasserstein procrustes, in: Proceedings of the 22nd International Conference on Artificial Intelligence and Statistics, Vol.~89 of Proceedings of Machine Learning Research, 2019, pp. 1880--1890.

\bibitem{Peng2021ProcrustesKGE}
X.~Peng, G.~Chen, C.~Lin, M.~Stevenson, Highly efficient knowledge graph embedding learning with orthogonal procrustes analysis, in: Proceedings of the 2021 Conference of the North American Chapter of the Association for Computational Linguistics: Human Language Technologies, 2021, pp. 2364--2375.

\bibitem{GEMM}
R.~Iakymchuk, D.~Defour, C.~Collange, S.~Graillat, Reproducible and accurate matrix multiplication, in: International Symposium on Scientific Computing, Computer Arithmetic, and Validated Numerics, Springer, 2015, pp. 126--137.

\bibitem{QR}
P.-G. Martinsson, G.~Quintana~OrtI, N.~Heavner, R.~Van De~Geijn, Householder qr factorization with randomization for column pivoting (hqrrp), SIAM Journal on Scientific Computing 39~(2) (2017) C96--C115.

\bibitem{SVD}
L.~Wang, G.~Libert, P.~Manneback, Kalman filter algorithm based on singular value decomposition, in: [1992] Proceedings of the 31st IEEE Conference on Decision and Control, IEEE, 1992, pp. 1224--1229.

\bibitem{trinv}
R.~Mahfoudhi, A fast triangular matrix inversion, in: Proceedings of the World Congress on Engineering, Vol.~1, 2012.

\bibitem{GeometricPerturbation}
K.~Chen, L.~Liu, Geometric data perturbation for privacy preserving outsourced data mining, Knowledge and information systems 29~(3) (2011) 657--695.

\bibitem{TDC}
K.~Huang, T.~Fu, W.~Gao, Y.~Zhao, Y.~Roohani, J.~Leskovec, C.~W. Coley, C.~Xiao, J.~Sun, M.~Zitnik, Therapeutics data commons: Machine learning datasets and tasks for drug discovery and development, arXiv preprint arXiv:2102.09548 (2021).

\bibitem{adult}
B.~Becker, R.~Kohavi, {Adult}, UCI Machine Learning Repository, {DOI}: https://doi.org/10.24432/C5XW20 (1996).

\bibitem{eicu}
T.~J. Pollard, A.~E. Johnson, J.~D. Raffa, L.~A. Celi, R.~G. Mark, O.~Badawi, The eicu collaborative research database, a freely available multi-center database for critical care research, Scientific data 5~(1) (2018) 1--13.

\bibitem{analyticalgaussianmechanism}
B.~Balle, Y.-X. Wang, Improving the gaussian mechanism for differential privacy: Analytical calibration and optimal denoising, in: International Conference on Machine Learning, PMLR, 2018, pp. 394--403.

\bibitem{AMES}
C.~Xu, F.~Cheng, L.~Chen, Z.~Du, W.~Li, G.~Liu, P.~W. Lee, Y.~Tang, In silico prediction of chemical ames mutagenicity, Journal of chemical information and modeling 52~(11) (2012) 2840--2847.

\bibitem{Tox21}
A.~Mayr, G.~Klambauer, T.~Unterthiner, S.~Hochreiter, Deeptox: toxicity prediction using deep learning, Frontiers in Environmental Science 3 (2016) 80.

\bibitem{HIV}
Z.~Wu, B.~Ramsundar, E.~N. Feinberg, J.~Gomes, C.~Geniesse, A.~S. Pappu, K.~Leswing, V.~Pande, Moleculenet: a benchmark for molecular machine learning, Chemical science 9~(2) (2018) 513--530.

\bibitem{CYP2D6}
H.~Veith, N.~Southall, R.~Huang, T.~James, D.~Fayne, N.~Artemenko, M.~Shen, J.~Inglese, C.~P. Austin, D.~G. Lloyd, et~al., Comprehensive characterization of cytochrome p450 isozyme selectivity across chemical libraries, Nature biotechnology 27~(11) (2009) 1050--1055.

\bibitem{facenet}
F.~Schroff, D.~Kalenichenko, J.~Philbin, Facenet: A unified embedding for face recognition and clustering, in: Proceedings of the {IEEE} Conference on Computer Vision and Pattern Recognition ({CVPR}), 2015, pp. 815--823.
\newblock \href {https://doi.org/10.1109/CVPR.2015.7298682} {\path{doi:10.1109/CVPR.2015.7298682}}.

\bibitem{resnet}
C.~Szegedy, S.~Ioffe, V.~Vanhoucke, A.~A. Alemi, Inception-v4, inception-resnet and the impact of residual connections on learning, in: Proceedings of the Thirty-First {AAAI} Conference on Artificial Intelligence ({AAAI}), 2017, pp. 4278--4284.

\bibitem{vggface}
Q.~Cao, L.~Shen, W.~Xie, O.~M. Parkhi, A.~Zisserman, Vggface2: A dataset for recognising faces across pose and age, in: 2018 13th {IEEE} International Conference on Automatic Face \& Gesture Recognition ({FG}), 2018, pp. 67--74.
\newblock \href {https://doi.org/10.1109/FG.2018.00020} {\path{doi:10.1109/FG.2018.00020}}.

\bibitem{MNIST}
L.~Deng, The mnist database of handwritten digit images for machine learning research [best of the web], IEEE signal processing magazine 29~(6) (2012) 141--142.

\bibitem{fashionmnist}
H.~Xiao, K.~Rasul, R.~Vollgraf, Fashion-mnist: a novel image dataset for benchmarking machine learning algorithms, arXiv preprint arXiv:1708.07747 (2017).

\bibitem{PubChem}
Y.~Wang, J.~Xiao, T.~O. Suzek, J.~Zhang, J.~Wang, S.~H. Bryant, Pubchem: a public information system for analyzing bioactivities of small molecules, Nucleic acids research 37~(suppl\_2) (2009) W623--W633.

\end{thebibliography}

\appendix

\end{document}